%% file: main_informed.tex
\documentclass[nohyperref]{article}

\usepackage{microtype}
\usepackage{graphicx}
\usepackage{subfigure}
\usepackage{booktabs} 

\usepackage{hyperref}

\usepackage[accepted]{icml2022}

\usepackage{amsmath}
\usepackage{amssymb}
\usepackage{mathtools}
\usepackage{amsthm}

\usepackage[utf8]{inputenc} 
\usepackage[T1]{fontenc}    
\usepackage{hyperref}       
\usepackage{url}            
\usepackage{booktabs}       
\usepackage{amsfonts}       
\usepackage{nicefrac}       
\usepackage{microtype}      
\usepackage{xcolor}         
\usepackage{amsmath}
\usepackage{amssymb}
\usepackage{amsthm}
\usepackage{algorithm}
\usepackage{algorithmic}
\usepackage{xspace}
\usepackage{graphicx}
\usepackage{subfigure}
\usepackage{wrapfig}
\usepackage{graphicx}
\usepackage{subfigure}
\usepackage{url}
\usepackage{color}
\usepackage{multirow}
\usepackage{gensymb}
\usepackage{longtable}
\usepackage{tikz}
\usepackage{comment}
\usepackage{dsfont}
\usepackage{framed}
\usepackage{enumitem}
\usepackage{xspace}
\usepackage{bm}

\usepackage[capitalize,noabbrev]{cleveref}

\theoremstyle{plain}
\newtheorem{theorem}{Theorem}[section]

\newtheorem{lemma}[theorem]{Lemma}
\newtheorem{corollary}[theorem]{Corollary}
\theoremstyle{definition}
\newtheorem{definition}{Definition}
\newtheorem{assumption}{Assumption}
\theoremstyle{remark}
\newtheorem{remark}{Remark}

\usepackage[textsize=tiny]{todonotes}

\icmltitlerunning{Informed Learning by Wide Neural Networks: Convergence, Generalization and Sampling Complexity}

\begin{document}

\twocolumn[
\icmltitle{Informed Learning by Wide Neural Networks: \\Convergence, Generalization and Sampling Complexity}



\icmlsetsymbol{equal}{*}

\begin{icmlauthorlist}
\icmlauthor{Jianyi Yang}{yyy}
\icmlauthor{Shaolei Ren}{yyy}
\end{icmlauthorlist}

\icmlaffiliation{yyy}{Department of Electrical and Computer Engineering, University of California, Riverside, CA 92521, United States}

\icmlcorrespondingauthor{Shaolei Ren}{sren@ece.ucr.edu}

\icmlkeywords{Machine Learning, ICML}

\vskip 0.3in
]



\printAffiliationsAndNotice{} 

\begin{abstract}
By integrating domain knowledge with labeled samples,
{informed} machine learning has been emerging to improve
the learning performance for a wide range of applications. Nonetheless,
 rigorous
understanding of the role
of injected domain knowledge has been under-explored.
In this paper, we consider an informed deep neural network
(DNN) with over-parameterization and domain knowledge
integrated into its training objective function,
and study how and why domain knowledge
benefits the performance.
Concretely,
we
quantitatively demonstrate the two benefits of domain knowledge in informed learning
--- regularizing the label-based supervision and supplementing the labeled
samples --- and reveal the trade-off between label and knowledge imperfectness in the bound of the population risk. Based on the theoretical analysis, we propose a generalized informed training objective to better exploit the benefits of knowledge and balance the label and knowledge imperfectness, which is validated by the population risk bound.
Our analysis on sampling complexity sheds lights
on how to choose the hyper-parameters for informed learning, and further justifies the advantages of knowledge informed learning.
\end{abstract}
\input{Introduction}

\input{RelatedWork}

\input{Formulation}

\input{Convergence}

\input{Connection}
\input{Conclusion}


\bibliographystyle{icml2022}

\input{reference.bbl}
\newpage
\appendix
\onecolumn

\section*{Appendix}

\input{TrainingAlgorithm}
\input{ProofConvergence}
\input{ProofGeneralization}
\input{Proof2nd}

\input{Proof3rd}
\input{Application}
\input{Experiements}

\end{document}

%% file: Introduction.tex
\section{Introduction}

 The remarkable success of deep  neural networks (DNNs), or more generally machine learning,
 largely relies on the proliferation of data samples with ground-truth labels for supervised learning.
Nonetheless, labeled data of high quality
can often be very limited and/or extremely expensive
to collect in real application domains,
including medical sciences, security-related fields, and specialized
engineering areas \cite{Informed_ML_von19}.

In parallel with the data-driven learning paradigm,
 domain \emph{knowledge} (which we simply
 refer
 to as  knowledge) has been utilized
 to assist with decision making and system designs, with
  a long history of success.
 As its name would suggest,
 domain knowledge is naturally domain-specific and can come
 from various sources in multiple forms, such as subjective experiences (e.g., medical prognosis), external sources, and scientific laws.
  For example, partial differential equations are used
  to govern many flow dynamics in physics, and
 the Shannon channel capacity is the fundamental principle
 to guide the design of modern communications systems \cite{Goldsmith_WirelessCommunications_2005,InformedML_PhysicsModeling_Survey_arXiv_2020_DBLP:journals/corr/abs-2003-04919}.

Importantly, domain knowledge has already been,
sometimes implicitly, integrated
into every stage of the machine learning pipeline,
including training data augmentation, hypothesis set
selection, model training and hypothesis finalization (more
details in Appendix~\ref{sec:preliminary}).
For example,
differential equations and logic rules from physical sciences and/or common knowledge provide additional constraints or new {functional} regularization terms for model training \cite{Interaction_learning_physics_battaglia2016interaction,Improving_DL_Knowledge_Constriant_borghesi2020improving,Injecting_knowledge_NN_silvestri2020injecting,Incorporating_domain_knowledge_muralidhar2018incorporating,Semantic_Loss_smbolic_xu2018semantic}.

Despite the numerous successful examples \cite{Informed_ML_von19,InformedML_HumanKnowledge_Survey_Cell_Umich_2020_DENG2020101656},
there still lacks a rigorous understanding of the role of
 domain knowledge in informed  learning.
In this paper, we focus on informed DNNs  --- DNNs with domain knowledge
explicitly integrated into the training risk/loss function.
Concretely, we consider an over-parameterized DNN with a sufficiently
large network width \cite{overparameterization_generalization_neyshabur2018role},
and study how domain knowledge
affects the DNN from three complementary aspects: convergence, generalization, and sampling complexity.

\textbf{Convergence (\textbf{Theorem~\ref{the:convergence}})}:
We show the convergence of  training an informed risk function
under milder technical assumptions than the prior works (Section~\ref{sec:convergencemainresult}).
More specifically, we show that for inputs within a smooth set (Definition~\ref{def:smoothset}), the network outputs converge to the optimal solution jointly determined by all the samples in the set.

\textbf{Generalization (\textbf{Theorems~\ref{thm:generalizationbound} and~\ref{crl:generalizationbound}})}:
We show in Theorem~\ref{thm:generalizationbound}  that the population risk relies on the  knowledge imperfectness (Definition~\ref{def:knwoledgeimperfect}) as well as knowledge-regularized label imperfectness (Definition~\ref{def:regularizedimperfectness}). Specifically, knowledge has two benefits: regularization for noisy labels and supplementing labels.
 We propose a generalized informed risk function which disentangles the two effects by introducing another hyper-weight $\beta$, followed by the population risk bounds in Theorem~\ref{crl:generalizationbound} and Corollary~\ref{crl:generalizationbound2}.

\textbf{Sampling Complexity (\textbf{Corollay~\ref{thm:generalization_choiceweight}})}:
 By establishing a quantitative equivalence
   between domain knowledge and labeled samples,
   we show that domain knowledge (with a reasonable quality) can effectively reduce
   the number of labeled samples while achieving the same generalization performance,
   compared to the no-knowledge case.

%% file: RelatedWork.tex
\section{Related Work}

\textbf{Informed Machine Learning.}
The broad paradigm of informed machine learning \cite{Informed_ML_von19}
includes several existing learning frameworks, such as
learning using privileged information (LUPI) \cite{privileged_vapnik2009new} where
side knowledge is available for labeled samples \cite{privileged_vapnik2009new,Information_bottleneck_priviledged_Motiian_2016_CVPR,Rank_priviledged_info_Sharmanska_2013_ICCV}.  
Likewise, knowledge distillation \cite{KD_analysis_rahbar2020unreasonable,Knowledge_Distillation_hinton2015distilling,distillation_survey_gou2020knowledge,efficacy_distillation_cho2019efficacy} transfers prior knowledge from teacher networks to a student network.
Some recent studies have also focused on understanding knowledge distillation \cite{ensemble_distillation_allen2020towards}.
In \cite{Towards_understanding_distillation_phuong2019towards}, a generalization bound
is derived for knowledge distillation based on linear classifiers and deep linear classifiers, providing insights towards the mechanism of knowledge distillation.
The subsequent analysis
\cite{distillation_riskbound_ji2020knowledge,KD_analysis_rahbar2020unreasonable}
extends to neural networks, showing that the student network may
generalize better by exploiting soft labels from the teacher model. 
 Teacher imperfectness is investigated in \cite{KD_inference_dao2021knowledge}, which bounds the learning error 
and proposes enhanced methods to address imperfect teachers.

Physics-informed neural networks  (PINNs) have
been recently proposed to solve
 partial differential equations (PDEs)
 \cite{physical_model_deep_networks_guen2020augmenting,Physics-informed_learning,PIDL_raissi2017physics,sciML_baker2019workshop,InformedML_HumanKnowledge_Survey_Cell_Umich_2020_DENG2020101656,InformedML_PhysicsModeling_Survey_arXiv_2020_DBLP:journals/corr/abs-2003-04919}.
 Besides empirical studies,
 \cite{PINN_convergence_shin2020convergence} bounds the expected PINN loss, showing that the minimizer of the regularized 
 loss converges to the PDE solution. 

More broadly,  informed
machine learning also includes weakly-supervised learning \cite{weak-supervised_zhou2018brief,weak_supervision_robinson2020strength} and few-shot learning \cite{Generalizing_few_shot_wang2020}, where knowledge provides weak supervision.
Domain-specific constraints \cite{Incorporating_domain_knowledge_muralidhar2018incorporating} and semantic information \cite{Semantic_Loss_smbolic_xu2018semantic,Semantic_regularization_diligenti2017semantic}
can also be viewed as knowledge injected into training.
Our work complements these empirical studies and
provides a rigorous understanding of knowledge
in a unified framework.

\textbf{Over-parameterized neural networks.}
Several recent studies 
\cite{DNN_Overparameterized_OpportunitiesPitfalls_ICML_Workshop_2021,DNN_Overparameterized_ShallowNN_NIPS_2021_song2021subquadratic,DNN_Overparameterized_Convergence_ReLu_JiaLiu_OSU_ICLR_2022_anonymous2022a,DNN_OverparameterizedDecentralizedMultiAgent_JiaLiu_OSU_ICLR_2022_anonymous2022decentralized,NTK_jacot2018neural,Wide_NN_linear_lee2019wide,Neural_tangent_yang2019scaling,Convergence_zhu_allen2019convergence,NTK_arora2019exact,convergence_generalization_Arora2019fine,Generalization_cao2019generalization,generalization_allen2018learning,overparameterization_generalization_neyshabur2018role}
show
that over-parameterized neural networks have good convergence and generalization performance.
In addition to assuming data separability in a strong sense, another crucial assumption
often made in the existing studies
 is that the network widths increase polynomially with the total number of training samples.
In informed DNNs, however, we can have many (unlabeled) training samples fed
into the knowledge
risk, which hence may not satisfy these assumptions.
Thus, we analyze knowledge-informed over-parameterized neural networks
under relaxed assumptions (Section~\ref{sec:main}).

\textbf{Regularization.}
In the broad context of regularization, \cite{regularization_overparameterized_nn_wei2019regularization} shows that over-parameterized neural networks with $l_2$-regularization can achieve a larger margin and thus better generalization, \cite{implicit_regularization_blanc2020implicit} proves that SGD with label noise is equivalent to  an implicit regularization term,
while  \cite{dropout_regularization_wei2020implicit} shows that the drop-out operation for neural networks has both explicit and implicit regularization effects.
These regularizers are usually imposed on the network weights,
whereas the knowledge-based regularizer in informed machine learning
also incorporates inputs and directly
regularizes the network output.

%% file: Formulation.tex
\section{Informed Neural Network}

\textbf{Notations:} We use the expression $[L]$ to denote the set $\{1,2,\cdots, L\}$ for a positive integer $L$. Denote the indicator function as $\mathds{1}(x)=1$ if $x>0$, and $\mathds{1}(x)=0$ otherwise. $\mathbb{E}$ is the expectation operator and $\mathbb{P}$ is a probability measure. $\mathbb{R}^d$ is $d$-dimensional real number space. $\mathcal{N}(x,\sigma^2)$ is the Gaussian distribution with mean  $x$ and variance  $\sigma^2$.  Denote $|\mathcal{A}|$ as the size of a set $\mathcal{A}$. For a vector $x$, $\|x\|$ is $l_2$-norm and $[x]_j$ is the $j$th entry. For a matrix $\bm{X}$, $\|\bm{X}\|_2$
represents the spectral norm,
 and $\|\bm{X}\|$ is the Frobenius norm. $\mathcal{B}(x,\tau)=\left\lbrace y\mid \|x-y\|\leq \tau\right\rbrace $ is the neighborhood domain.

\subsection{Preliminaries of Neural Networks}
Consider a supervised learning task to learn a relationship mapping the input $x\in\mathcal{X}\subseteq
\mathbb{R}^b$ to its output $y\in \mathcal{Y}\subseteq
\mathbb{R}^d$.
The pair of input and output $\left( x,y\right) $  follows a joint distribution $\mathbb{P}_{XY}$.
More concretely, we consider
a fully-connected DNN with an input layer, $L\geq 1$ hidden layers, and an output layer.
Each hidden layer has $m$ neurons, followed by ReLu activation denoted as $\sigma(\cdot)$. Denote $\bm{W}_0\in\mathbb{R}^{b\times m}$ as the weights for the input layer, $\bm{W}_l\in\mathbb{R}^{m\times m}$ as the weights for the $l$-th  layer for $l\in[L]$, and $\bm{V}\in\mathbb{R}^{d\times m}$ as the weights for the output layer.
We denote the output of the $l$-th layer as $h_l=\sigma\left( \bm{W}_lh_{l-1}\right)$,
for  $l\in[L]$, where $h_{0}$ is the input $x$. The output of the neural network can be expressed as $h_{\bm{W}}=\bm{V}h_{L}$, where $\bm{W}=\left\lbrace \bm{W}_0,\bm{W}_1,\cdots, \bm{W}_{L}\right\rbrace $. Thus, the DNN can be expressed as
\begin{equation}
h_{\bm{W}}\left(x \right) =\bm{V}\sigma\left( \bm{W}_{L}\sigma(\bm{W}_{L-1} \cdots \sigma\left( \bm{W}_1\sigma(\bm{W}_0x))\right) \right).
\end{equation}
Given a DNN $h_{\bm{W}}$, the risk for a labeled sample $(x,y)$ is denoted as $r\left(h_{\bm{W}}\left(x \right),y \right) $.
The goal of the learning task is to learn a DNN that minimizes the population risk:
\begin{equation}
R\left(h \right) =\mathbb{E}\left[r\left(h\left(x \right),y \right)  \right].
\end{equation}

\subsection{Integration of Knowledge}\label{sec:paradigm}
We consider a commonly-used informed learning method, i.e., integrating knowledge into the neural network during the training stage \cite{Informed_ML_von19}.
During training, a labeled dataset $S_z=\left\lbrace \left(x_1, z_1 \right),\cdots, \left(x_{n_z}, z_{n_z} \right) \right\rbrace$  with $n_z$ samples drawn from $\mathbb{P}_{XZ}$ is provided. We assume $x_i, i\in[n]$ are drawn from the distribution $\mathbb{P}_X$, but the training label $z_i\in\mathcal{Y}$ may not be the same as the true label $y_i$ for the input $x_i$,
 because the training label may be of low quality (e.g., corrupted, noisy, and/or quantized)\cite{classification_imperfect_labels_cannings2020classification,weak-supervised_zhou2018brief}. Denote $h_{\bm{W},i}=h_{\bm{W}}(x_i)$ as the output of the neural network with respect to the input $x_i$. Based on the labeled dataset, the empirical label-based risk can be written as
$
\hat{R}_{S_z}\left(\bm{W}\right) =\frac{1}{n_z}\sum_{S_z} r\left(h_{\bm{W},i},z_i \right).
$

The  domain knowledge 
includes a knowledge-based model $g(x)$ regarding the input $x$
and a knowledge-based risk function $r_{\mathrm{K}}\left( h_{\bm{W}}\left(x\right),g(x)\right)$ that relates the DNN's output
$h_{\bm{W}}\left(x\right)$
to  $g(x)$.  More concrete examples of risk functions for  domain knowledge 
can be found in Appendix~\ref{sec:formulation_example}.

For the ease of analysis, we assume that both the risk function $r$ and the knowledge-based risk function $r_{\mathrm{K}}$ are Lipschitz continuous, upper bounded, and strongly convex with respect to the network output, and the eigenvalues of their Hessian matrix regarding the network output lie in $[\rho,1]$ for  $\rho\in(0,1]$.
Note that the incorporated domain knowledge may not necessarily be perfect since it
can  be obtained based on subjective experiences (e.g., medical prognosis)
\cite{Incorporating_domain_knowledge_muralidhar2018incorporating,DNN_Medical_ICLR_2020_Mihaela_Bica2020Estimating},
pre-existing machine learning models \cite{Knowledge_Distillation_hinton2015distilling}  or theoretical models which itself can deviate from the real physical world
\cite{Physics-informed_learning}.

For training, in addition to the labeled dataset $S_z$,
a dataset $S_g$ with $n_g$ unlabeled samples is generated for knowledge-based supervision.
Note that $S_g$ can also include inputs in $S_z$, and $n_g$ can be sufficiently large since unlabeled samples are typically easier to obtain than labeled ones.   The
 training risk of the informed neural network, which we simply refer to as  \emph{informed} risk, is
\begin{equation}\label{eqn:informlossnn}
\begin{split}
\hat{R}_{\mathrm{I}}\!\left(\bm{W}\right)& \!\!=\!\!\frac{1-\lambda}{n_z} \!\sum_{ S_z}   r\!\left(h_{\bm{W},i},\!z_i \right)\!\!+\!\!\frac{\lambda}{n_g}\!\!\sum_{S_g} r_{\mathrm{K}}\!\left( h_{\bm{W},i},\!g_i\right),
\end{split}
\end{equation}
where $\lambda\in[0,1]$ is a hyper-weight, $h_{\bm{W},i}=h_{\bm{W}}(x_i)$, and $g_i=g(x_i)$.  Note that Eqn.~\eqref{eqn:informlossnn} can  also
be re-written as
\begin{equation}\label{eqn:informlossnn_training}
	\begin{split}
		\hat{R}_{\mathrm{I}}\!\!\left(\bm{W}\right)\!=\!\!\!\!\!\sum_{ S_z\bigcup S_g} \!\!\!\! \left[ \mu_i r\left(h_{\bm{W},i},z_i \right)\!+\! \lambda_ir_{\mathrm{K}}\left( h_{\bm{W},i},g_i\right)\right]
	\end{split}
\end{equation}
 with hyper-parameters chosen as $\mu_i=\frac{1-\lambda}{n_z}\mathds{1}(x_i\in S_z)$ and $\lambda_i=\frac{\lambda}{n_g}\mathds{1}(\!x_i\in S_g\!)$. Eqn.~\eqref{eqn:informlossnn_training} is used for convergence analysis.
 
To train the informed DNN,
we consider a
gradient descent approach in Algorithm~\ref{alg:neuraltrain} shown in Appendix~\ref{appendix:training_algorithm}.
This training approach has also been commonly considered in the literature \cite{Convergence_zhu_allen2019convergence,Convergence_Gu_zou2019improved,GD_Neural_du2019gradient}
for theoretical analysis of standard DNNs without domain knowledge.
For the sake of analysis, we also define a hypothesis space $\mathcal{H}=\left\lbrace h_{\bm{W}}\mid \bm{W}\in \mathcal{B}\left(\bm{W}^{(0)},\tau \right) \right\rbrace $ where $\bm{W}^{(0)}$ is the initialized weight and $\tau$ is the maximum distance between the weights in gradient descent and the initialized weights. We denote $h^{(0)}_l(x), l\in[L]$ as the output of the $l$-th layer for an input $x$ at initialization.
 
\begin{remark}
The considered informed learning 
is relevant to several other  frameworks. For example, it
can model weakly-supervised learning \cite{weak-supervised_zhou2018brief,Generalizing_few_shot_wang2020} with a few 
(possibly imperfectly) labeled samples as well as other weak supervision signals (i.e., knowledge). Besides, by viewing $\{z_i\}$ as hard labels and the knowledge-based model $g(x)$ as soft labels provided by a teacher model, the informed learning captures knowledge distillation \cite{Knowledge_Distillation_hinton2015distilling,Towards_understanding_distillation_phuong2019towards,KD_analysis_rahbar2020unreasonable}. Thus, our work can complement the existing analysis
for the
aforementioned learning frameworks from a different and more unified
perspective. Additionally, PAC-Bayesian learning optimizes the PAC-Bayesian bound which is a trade-off between the empirical error and a regularization term based on a prior distribution given by knowledge \cite{PAC_Bayesian_guedj2019primer,Meta_learning_PAC_Bayes_amit2018meta,PAC_Bayesian_Bayesian_Inference_germain2016pac}. But, different from  PAC-Bayesian learning which considers random hypothesis, we analyze
an over-parameterized neural network with a predetermined architecture.
\end{remark}

%% file: Convergence.tex
\section{Effects of Domain Knowledge}\label{sec:main}

\subsection{Convergence}\label{sec:convergencemainresult}
Since the domain knowledge is integrated into a neural network during training, it is important to analyze the convergence to understand how the label and knowledge supervision jointly determine the network output.
While convergence  based on gradient descent for over-parameterized neural networks has been studied extensively \cite{DNN_Overparameterized_OpportunitiesPitfalls_ICML_Workshop_2021,Convergence_zhu_allen2019convergence,Convergence_Gu_zou2019improved,convergence_generalization_Arora2019fine,GD_Neural_du2019gradient}, the current analysis is \emph{not} suitable to study the convergence of informed over-parameterized neural networks.  The reasons are summarized as follows.\\
$\bullet$ \textbf{Inapplicable for multiple supervisions.} Typically,
assuming one unique label for each distinct training sample
and a large enough network width, the prior studies
 show that the neural network can fit to the labels, i.e., the network output for
  each training input converges to the corresponding label \cite{understanding_deep_learning_rethinking_generalizaiton_zhang2021understanding,convergence_generalization_Arora2019fine,Convergence_Gu_zou2019improved,moderate_overparameterization_convergence_shallow_oymak2020toward}.   But, in our case, one training input can have multiple supervisions from both label and knowledge with possibly different forms of risks.
  Thus, the network output for an input may not be necessarily determined by a unique label.
  The convergence of knowledge distillation supervised by both hard and soft labels is studied by \cite{KD_analysis_rahbar2020unreasonable}, but only the quadratic risk and shallow networks are considered.\\
$\bullet$ \textbf{Strong data separability assumption.}
	Some prior studies require a lower-bounded distance of any two samples \cite{Convergence_zhu_allen2019convergence,Convergence_Gu_zou2019improved,GD_Neural_du2019gradient},
	but this may not be satisfied for an informed DNN because
	the input samples for label-based and knowledge-based risks
	can be very close or even the same.
	Other studies assume data separability by a neural tangent model \cite{Overparameterization_needed_chen2019much,ji2019polylogarithmic, cao2020generalization,nitanda2019gradient}, but  data separability by a neural tangent model
is not well defined for training with multiple supervisions 
in informed DNNs.

 To address these challenges, we provide convergence analysis for informed over-parameterized neural networks based on a new data separability assumption of smooth sets.  The construction of smooth sets approximates the space $\mathcal{X}$ with discrete pieces, each containing samples that jointly satisfy the smooth properties. The smooth sets are formally defined below, followed by the data separability assumption.

 \begin{definition}[Smooth sets]\label{def:smoothset}
Given $\phi>0$, construct a $\phi-$net \cite{epislon_net_clarkson2006building} $\mathcal{X}_{\phi}=\{x_k', k\in[N], x_k'\in \mathcal{X}\}$ with 
$N\sim O(1/\phi^b)$  such that  $\forall x'_i, x'_j\in \mathcal{X}_{\phi}$ and $x'_i\neq  x'_j$, $\|x'_i-x'_j\|\geq\phi$ holds, and $\forall x_i\in S_z\bigcup S_g$, there exists at least one $x_k'\in \mathcal{X}_{\phi}$ satisfying $\|x_i-x_k'\|\leq \phi$.
Each input $x_k'\in\mathcal{X}_{\phi}$, referred to as a representative input, determines a smooth set $\mathcal{C}_{\phi,k}=\{ x\in \mathcal{X}\mid  \|x-x'_k\|\leq \phi, \|x-x'_j\|\geq \phi/2, \forall j\neq k, x'_k,x'_j\in \mathcal{X}_{\phi} \} $. The index set of training samples within the $k$th smooth set is
 $\mathcal{I}_{\phi,k}=\left\lbrace  i\mid x_i\in S_z\bigcup S_g, x_i\in \mathcal{C}_{\phi,k}\right\rbrace , k\in[N].$
 \end{definition}

 \begin{assumption}[Data separability by smooth sets]\label{asp:smoothset}
For each smooth set $k$ with representative sample $x_k'$, there exists a non-empty subset of neuron indices $\mathcal{G}_{k,\alpha}\in[m]$ with size $|\mathcal{G}_{k,\alpha}|=\alpha m, \alpha\in(0,1]$ such that at initialization, $\forall i\in \mathcal{I}_{\phi,k}$, $\forall j\in \mathcal{G}_{k,\alpha}$, $\mathds{1}\left(\left[h^{(0)}_{L}(x_i)\right]_j\geq 0\right)=\mathds{1}\left(\left[h^{(0)}_{L}(x_k')\right]_j\geq 0\right)$, and $\forall j\notin \mathcal{G}_{k,\alpha}$, the pre-activation of the $L$-th layer $\left|\left[\bm{W}_L^{(0)}h^{(0)}_{L-1}(x_i)\right]_j\right|\geq \frac{3\sqrt{2\pi}\phi^{b+1}}{16\sqrt{m}}$.
 \end{assumption}

 Instead of requiring a lower-bounded distance of any two training samples, the data separability assumption requires that, at initialization, for samples in one smooth set, the
 outputs of the last hidden layer  either have the same signs as those of the representative sample, or their absolute values are larger than a very small threshold.
 Thus, this data separability assumption is set-wise and addresses the cases where two training inputs are very close or the same, and hence is milder than the one in existing studies (e.g., \cite{Convergence_zhu_allen2019convergence}).
  The parameter $\alpha$ indicates slackness: with larger $\alpha$, more neurons have the same signs.
  Actually, data separation by smooth set with $\phi>0$ in Assumption~\ref{asp:smoothset} always exists:
  when $\phi$ is  small enough such that only one inputs or several same inputs are included in a smooth set,  Assumption~\ref{asp:smoothset} is satisfied with $\alpha=1$.
  Even in this worst case, our assumption is still 
   milder than the data separability assumption considered
   in \cite{Convergence_zhu_allen2019convergence,Convergence_Gu_zou2019improved} that excludes the existence of two training samples with the same inputs but different supervisions.

With the data-separability assumption by smooth sets, we are ready to show the labels and knowledge  jointly determine the network output for training inputs.
We introduce the notation \emph{effective label}, as formally defined below.

 \begin{definition}[Effective label]\label{def:efflabel}
 	For the $k$-th smooth set,
 	define the effective label  as $y_{\mathrm{eff},k}=\arg\min_{h}\sum_{i\in \mathcal{I}_{\phi,k}}\left\lbrace \mu_i r(h,z_i)+\lambda_i r_{K}(h,g_i) \right\rbrace $ with $\mu_i, \lambda_i$ defined in Eqn. \eqref{eqn:informlossnn} and $h$ in the space of network output, and the effective optimal risk as $r_{\mathrm{eff},k}=\sum_{i\in\mathcal{I}_{\phi,k}}\left\lbrace \mu_i r(y_{\mathrm{eff},k},z_i)+\lambda_i r_{K}(y_{\mathrm{eff},k},g_i) \right\rbrace$.
 \end{definition}

Next, we show the convergence analysis. Note that the proof based on the data separability by smooth sets (Assumption \ref{asp:smoothset}) invalidates the proofs in previous studies, and we need new
lemmas that lead to novel
convergence to effective labels in Definition \ref{def:efflabel}. In particular, in Lemma \ref{lma:forwardinputperturbation}, to approximate
the outputs in the smooth set $k$ by the output of the
representative input $x_k'$, we need to bound the difference
of the outputs with respect to $x_k'$ and an input in the smooth set $k$. Also, based on Assumption \ref{asp:smoothset}, we derive in Lemma~\ref{lma:gradientboud}
the gradient lower bound which relies on
the number of smooth sets $N$ instead of the sample size
$n_z + n_g$ in the previous analysis. This makes the network width $m$ in our analysis directly rely on the smooth set size $\phi$.  Moreover, in Lemma~\ref{lma:risksmoothness}, we prove based on the definition of smooth sets that the first-order approximation error of the total informed risk depends on the difference between the risk and effective risk in Definition \ref{def:efflabel}. This is important to prove the convergence to the effective labels. The details of the convergence analysis are deferred to Appendix~\ref{sec:proofconvergence}.

\begin{theorem}\label{the:convergence}
	Assume that the network width satisfies $m\geq\Omega\left(\phi^{-11b-4}L^{15}d\rho^{-4}\bar{\lambda}^{-4}\alpha^{-4}\log^3(m) \right)$, and the step size is set as $\eta=O(\frac{d}{L^2m})$. With Assumptions~ \ref{asp:smoothset} satisfied, for any $\epsilon>0$ and $\phi\leq \widetilde{O}\left( \epsilon L^{-9/2}\log^{-3}(m)\right) $, we have with probability at least $1-O(\phi)$, by gradient descent after
$
T= O\left(\frac{L^2}{\phi^{1+2b}\rho\bar{\lambda}\alpha}\log(\epsilon^{-1}\log(\phi^{-1})) \right)
$
steps,
the informed risk in Eqn.~\eqref{eqn:informlossnn_training} is bounded as:
$
\hat{R}_{\mathrm{I}}(\bm{W}^{(T)})-\hat{R}_{\mathrm{eff}}\leq O(\epsilon),
$
where $\hat{R}_{\mathrm{eff}}=\sum_{k=1}^Nr_{\mathrm{eff},k}$, $\bar{\lambda}=\Omega(\min(1-\lambda, \lambda)\mathds{1}(\lambda\in(0,1))+\mathds{1}(\lambda \in\{0,1\}))$. Also,
the DNN outputs satisfy:
\[
\sum_{S_z\bigcup S_g}(\mu_i+\lambda_i) \left\| h_{\bm{W}^{(T)}}\left(x_i \right)-y_{\mathrm{eff},k(x_i)} \right\|^2\leq O(\epsilon),
\]
where $k(x_i)$ is the index of the smooth set that includes $x_i$, $\mu_i=\frac{1-\lambda}{n_z}\mathds{1}(x_i\in S_z)$ and $\lambda_i=\frac{\lambda}{n_g}\mathds{1}(\!x_i\in S_g\!)$.
\end{theorem}

\begin{remark}
	The convergence analysis in Theorem~\ref{the:convergence} addresses the limitations
 mentioned at the beginning
of this section. First, instead of fitting a unique label for each input, the informed neural network with multiple supervisions converges to effective labels. Second, the data separability assumption is enough for convergence analysis of  informed neural networks. 
 Another observation is that with smaller $\phi$ and smaller $\alpha$, Assumption~\ref{asp:smoothset} becomes milder, but  a larger network width and more training steps are needed to guarantee  convergence.

 Additionally, different from previous convergence analysis where the width $m$ increases
directly with the sample size, the network width
$m$ in our analysis depends on the smooth set size $\phi$ and is non-decreasing
with sample size (i.e., $m$ may not always increase with the
sample size). To see this, given a construction of smooth sets by size $\phi$ that meets Assumption \ref{asp:smoothset}, if we continue to add (either labeled or knowledge-supervised) training samples that lie in
the existing smooth sets and satisfy Assumption \ref{asp:smoothset}, the width $m$
remains the same, and smaller $\phi$ (larger $m$) is needed to guarantee the convergence only when the added samples violate
Assumption \ref{asp:smoothset} under the current $\phi$.
 The large network width
	needed for analysis is due
	to the limitation of over-parameterization techniques, while in practice
	a much smaller network width is enough.
	Albeit beyond the scope of our study,
	addressing
	the gap between theory and practice is clearly important
	and still active
	research in the community \cite{DNN_Overparameterized_OpportunitiesPitfalls_ICML_Workshop_2021}.
\end{remark}
\begin{remark}\label{rmk:knowledgeeffect}
We can get more insights about the effects of labels and knowledge from the conclusion that the network outputs converge to the corresponding \emph{effective} labels in Definition \ref{def:efflabel}.
On the one hand, if knowledge is applied to the samples within the same smooth sets as labeled samples, knowledge-based supervision and label-based supervision jointly determine the network output together: knowledge serves as \emph{a regularization for labels} in this case. On the other hand, if a smooth set only contains knowledge-supervised samples,  the network output is determined solely by knowledge: knowledge  \emph{supplements  labeled samples} (albeit possibly imperfectly) to provide additional supervision.
\end{remark}

\subsection{Generalization}\label{sec:generalization}
We now formally analyze how the domain knowledge affects the generalization performance.
From our convergence analysis, there are two different effects of  knowledge (Remark~\ref{rmk:knowledgeeffect}). 
We characterize  the two effects by formally defining knowledge imperfectness and knowledge-regularized label imperfectness. Before this, we list some notations for further analysis. Given a $\phi-$net $\mathcal{X}_{\phi}$ (Definition \ref{def:smoothset}), $\mathcal{U}_{\phi}(S_z)=\left\lbrace k\in[N]\mid \exists x\in S_z, x \in \mathcal{C}_{\phi,k}\right\rbrace$ is the index collection of smooth sets that contain  at least one labeled sample, and $\mathcal{X}_{\phi}(S_z)=\bigcup_{k\in\mathcal{U}_{\phi}}(S_z)\mathcal{C}_{\phi,k}$ is the region covered by the smooth sets in $\mathcal{U}_{\phi}(S_z)$.  $S_g'=S_g\bigcap\mathcal{X}_{\phi}(S_z) $ is the knowledge supervised dataset with samples share the common smooth sets with labeled samples in $S_z$ while
 the samples in $S''_{g}=S_g\setminus S_g'$ lie in smooth sets without labeled samples. Denote $n'_{g}=|S'_{g}|$ and $n''_{g}=|S''_{g}|$.

\begin{definition}[Knowledge imperfectness]\label{def:knwoledgeimperfect}
	Let $h_{\mathrm{K}}^*=\min_{h}\frac{1}{n_g''}\sum_{S_g''}\left[ r_\mathrm{K}(h(x_i),g(x_i))\right] $ be the optimal hypothesis for the knowledge-based risk on the dataset $S_g''$. The imperfectness of domain knowledge $K$ applied to the dataset $S_g''$ is defined as
	$
\widehat{Q}_{\mathrm{K},S_g''}= \frac{1}{n_g''}\sum_{x_i\in S_g''} r(h_{\mathrm{K}}^*(x_i),y_i)
$
where  $y_i$ is the true label of $x_i$.
	Correspondingly, let $\bar{h}_{\mathrm{K}}^*=\min_{h}\mathbb{E}\left[ r_\mathrm{K}(h(x),g(x))\right] $ be the optimal hypothesis for the expected knowledge-based risk, and the expected imperfectness of domain knowledge $K$ is defined as
	$
	Q_{\mathrm{K}}= \mathbb{E} \left[ r(\bar{h}_{\mathrm{K}}^*(x),y)\right].
	$
\end{definition}

The (empirical or expected) knowledge imperfectness is defined
as the risk under the  hypothesis optimally learned by knowledge-based supervision. Thus, it measures the extent to which the domain knowledge is inconsistent with the true labels, measured in terms of the risk over the hypothesis set $\mathcal{H}$. Besides knowledge-based supervision, the network outputs for some smooth sets that
contain both samples for knowledge risks and labeled samples are jointly determined by label-based and knowledge-based supervisions. Thus, we define knowledge-regularized label imperfectness below.

\begin{definition}[Knowledge-regularized label imperfectness]\label{def:regularizedimperfectness}
Let $h_{\mathrm{R},\beta}^*=\arg\min_{h}\frac{1-\beta}{n_z}\sum_{S_z}r(h(x_i),z_i)+ \frac{\beta}{n_g'}\sum_{ S'_{g}}r_{\mathrm{K}}(h(x_i),g(x_i)) $ be the optimal hypothesis for the knowledge-regularized risk and $\beta\in[0,1]$. The knowledge-regularized label imperfectness
is
$
	\widehat{Q}_{\mathrm{R},S_z,S'_{g}}(\beta)= \frac{1}{n_z}\sum_{ S_z}r(h_{\mathrm{R},\beta}^*(x_i),y_i),
$
	where  $y_i$ is the true label regarding $x_i$.
	Correspondingly, with $\bar{h}_{\mathrm{R},\beta}^*=\arg\min_{h}\mathbb{E}[\frac{1-\beta}{n_z}\sum_{S_z}r(h(x_i),z_i)+ \frac{\beta}{n_g'}\sum_{ S'_{g}}r_{\mathrm{K}}(h(x_i),g(x_i))]  $  being the optimal hypothesis for the regularized risk, the expected knowledge regularized label imperfectness is
	$
	Q_{\mathrm{R}}(\beta)= \mathbb{E}\left[ r(\bar{h}_{\mathrm{R},\beta}^*(x),y)\right].
	$
\end{definition}

Like knowledge imperfectness, knowledge-regularized label imperfectness indicates the risk of the  hypothesis optimally learned by joint supervision
from labels and knowledge. We see that when $\beta=0$, $\widehat{Q}_{\mathrm{R}}(0)$ (or $Q_{\mathrm{R}}(0)$) is the imperfectness of pure label-based supervision. Thus, the
gain due to knowledge is $\Delta \widehat{Q}_{\mathrm{R},\beta}=\widehat{Q}_{\mathrm{R}}(0)-\widehat{Q}_{\mathrm{R}}(\beta)$ (or $\Delta Q_{\mathrm{R},\beta}=Q_{\mathrm{R}}(0)-Q_{\mathrm{R}}(\beta)$ for the expected version). We show in the following theorem how the two types of imperfectness affect the population risk trained on the informed risk in Eqn.~\eqref{eqn:informlossnn}. The details are deferred to Appendix \ref{sec:proofgeneralization}.

\begin{theorem}\label{thm:generalizationbound}
	With $\bm{W}^{(T)}$ trained on
	Eqn.~\eqref{eqn:informlossnn}, $\phi\leq \widetilde{O}\left( \epsilon^2 L^{-9/2}\log^{-3}(m)\right), \phi\leq (\sqrt{\epsilon}/n_z)^{1/b} $, and other assumptions the same as Theorem \ref{the:convergence},
with probability at least $1-O(\phi)-\delta, \delta\in(0,1)$, the population risk satisfies
	\[
\begin{split}
&R\left(h_{\bm{W}^{(T)}}\!\right) \!\!\leq\!\! O(\sqrt{\epsilon})+(1-\lambda) \widehat{Q}_{\mathrm{R},S_z,S'_{g}}(\beta_{\lambda}) \\
&+\lambda \widehat{Q}_{\mathrm{K},S''_{g}}+O\!\left( \Phi\!+\!\sqrt{\log(1/\delta)}\right)\! \left(\frac{1-\lambda}{\sqrt{n_z}}\!+\!\frac{\lambda}{\sqrt{n_g}}\right),
\end{split}
	\]
	where $\beta_{\lambda}=\frac{\lambda n'_{g}}{(1-\lambda) n_g+\lambda n'_{g}}$,  $\widehat{Q}_{\mathrm{R},S_z,S'_{g}}(\beta_{\lambda})$ is the knowledge-regularized label imperfectness in Definition \ref{def:regularizedimperfectness} and $\widehat{Q}_{\mathrm{K},S''_{g}}$ is the knowledge imperfectness  in Definition \ref{def:knwoledgeimperfect} applied
to $S''_{g}$, and $\Phi=O\left( 4^L L^{3/2} m^{1/2}\phi^{-b-1/2}d \rho^{-1/2}\bar{\lambda}^{-1/2}\alpha^{-1/2} \right)$.
\end{theorem}

\begin{remark}\label{remark:knowledge_effect}
Theorem~\ref{thm:generalizationbound} shows that  by training on the informed risk \eqref{eqn:informlossnn}, knowledge affects the generation performance in the following two ways.  \\	
	$\bullet$ \textbf{Knowledge for regularization.}
  	When knowledge is applied to sample inputs inside the same smooth sets as labeled samples, it serves as an explicit regularization for label-based supervision, possibly reducing the label imperfectness from $\widehat{Q}_{\mathrm{R},S_z,S'_{g}}(0)$ to $\widehat{Q}_{\mathrm{R},S_z,S'_{g}}(\beta_{\lambda})$.\\	
	$\bullet$ \textbf{Knowledge for supplementing labels.}
	The generalization error is in the order of $O\left( \frac{1-\lambda}{\sqrt{n_z}}\!+\!\frac{\lambda}{\sqrt{n_g}}\right) $. When no knowledge is used (
	$\lambda=0$), the order is as large as $O\left( \frac{1}{\sqrt{n_z}}\right) $. If knowledge is applied ($\lambda>0$), then the generalization error decreases with the increasing of knowledge-supervised sample size $n_g$. Thus, when knowledge is applied to smooth sets without labeled samples, it serves as a (possibly imperfect) supplement for labels, while introducing knowledge imperfectness $\widehat{Q}_{\mathrm{K},S''_{g}}$. \\
	The hyper-parameter $\lambda$ can be used to balance the introduced imperfectness and generalization error from label and knowledge supervision. However, by the risk bound, it is hard to use one hyper-parameter $\lambda$ to control the two effects of knowledge, which will be further discussed in the next section.
\end{remark}

\section{A Generalized Training Objective}\label{sec:improved_objective}
In the informed risk in Eqn.~\eqref{eqn:informlossnn},
only one hyper-weight $\lambda$ is present, controlling the  two
different effects of knowledge (Remark~\ref{remark:knowledge_effect}).
To better reap the benefits of knowledge,
we consider a generalized informed risk in Eqn.\eqref{eqn:informlossnn3} by introducing another hyper-weight $\beta$, which introduces more flexibility to govern the roles of domain knowledge.
\begin{equation}\label{eqn:informlossnn3}
\begin{split}
&\hat{R}_{\mathrm{I},G}\!\left(\bm{W}\right)\! =\!  \frac{( 1-\lambda) (1-\beta)}{n_z}\!\sum_{S_z}  r\left(h_{\bm{W},i},z_i \right)\!+\\
&\!\frac{( 1-\lambda)\beta}{n'_{g}}\sum_{S'_{g}}r_{\mathrm{K}}\!\left( h_{\bm{W},i},g_i\right)\!+\!\frac{\lambda}{n''_{g}}\sum_{S''_{g}} r_{\mathrm{K}}\!\left( h_{\bm{W},i},g_i\right),
\end{split}
\end{equation}
where $\beta, \lambda\in[0,1], h_{\bm{W},i}=h_{\bm{W}}(x_i), g_i=g(x_i)$.

In Eqn.~\eqref{eqn:informlossnn3}, the two hyper-parameters
$\lambda$ and $\beta$ can jointly control the knowledge effects (and the introduced imperfectness)
when knowledge is applied. The hyperparameter $\beta$ is used to controls the knowledge regularization strength. By Remark~\ref{remark:knowledge_effect}, knowledge-supervised
samples in $S_g'$ serve as an explicit regularization for label-based supervision while introducing
knowledge-regularized label imperfectness $Q_{\mathrm{R}}(\beta)$. Thus,
 when $\beta$ is larger, more effects from $S_g'$ are incorporated and the regularization effect from knowledge is stronger. 
Also, we use $\lambda$ to adjust the effect of supplementing labels and
the introduction of $Q_{\mathrm{K}}$. By Remark~\ref{remark:knowledge_effect}, $S_g''$ serves as an supplement for labels
while introducing the knowledge imperfectness $Q_{\mathrm{K}}$. Thus,
 with larger $\lambda$, more effects from $S_g''$ are incorporated, which means we incorporate more 
effects of data supplement from knowledge and also knowledge imperfectness $Q_{\mathrm{K}}$
but less effect of knowledge regularization and knowledge-regularized label imperfectness $Q_{\mathrm{R}}$. The benefit of the training objective in Eqn.~\eqref{eqn:informlossnn3} will be explained formally in Theorem~\ref{crl:generalizationbound} and Corollary~\ref{crl:generalizationbound2}.

Compared with the objective in Eqn.~\eqref{eqn:informlossnn} with only one hyper-parameter $\lambda$, Eqn.~\eqref{eqn:informlossnn3} introduces another hyper-parameter $\beta$ to independently adjust the degree
of the knowledge regularization,
making Eqn.~\eqref{eqn:informlossnn3} more general and flexible.
To train on Eqn.~\eqref{eqn:informlossnn3}, we need to separate dataset for knowledge supervision into two datasets $S_g'$ and $S_g''$ based on whether an input is close to a labeled input and assign different hyper-weights to them. 
The knowledge-based dataset separation is determined by $\phi$ in Definition \ref{def:smoothset}. Specifically, when the network width goes to infinity ($\phi$ goes to zero), $S_g'$ shares the same inputs as $S_z$, but $S_g'$ and $S_z$ are supervised by knowledge
and labels, respectively. We have $S_g''=S_{g}\setminus S_g'=S_{g}\setminus S_z$ which supplements the labels as shown in
Remark \ref{remark:knowledge_effect}. Note that when the knowledge is perfect and knowledge-supervised samples are sufficient, we do not need labeled samples, i.e., $S_z=\emptyset$ and we set $\lambda=1, \beta=1$. Then, we have $S_g''=S_g$ and Eqn.~\eqref{eqn:informlossnn3} becomes a purely knowledge-based risk. When no knowledge is applied, we set $\lambda=0, \beta=0$, and Eqn.~\eqref{eqn:informlossnn3} becomes a purely lable-based risk. In general cases when labels and knowledge are both used, hyper-parameters $\lambda$ and $\beta$ are used to control the effects of knowledge.

\subsection{Population Risk}
Note that Eqn.~\eqref{eqn:informlossnn3} can also be written as the form of Eqn. \eqref{eqn:informlossnn_training} with hyper-parameters chosen as $\mu_i= \frac{( 1-\lambda) (1-\beta)}{n_z}\mathds{1}(x_i\in S_z)$ and  $\lambda_i= \frac{( 1-\lambda)\beta}{n'_{g}}\mathds{1}(x_i\in S_g')+\frac{\lambda}{n''_{g}}\mathds{1}(x_i\in S_g'')$, so Theorem~\ref{the:convergence} for convergence still holds.
Next, we bound the population risk based on the generalized informed risk. The details are given in Appendix \ref{sec:proofgeneralization2}.

\begin{theorem}\label{crl:generalizationbound}
	Assume that $\bm{W}^{(T)}$ trained on
	Eqn.~\eqref{eqn:informlossnn3} and other assumptions are the same with those of Theorem \ref{the:convergence}, setting $\phi: \phi\leq \widetilde{O}\left( \epsilon^2 L^{-9/2}\log^{-3}(m)\right) $ and $\phi\leq (\sqrt{\epsilon}/n_z)^{1/b}  $,
with probability at least $1-O(\phi)-\delta, \delta\in(0,1)$, the population risk satisfies
	\[
	\begin{split}
	R(h_{\bm{W}^{(T)}})&\leq O(\sqrt{\epsilon})+(1-\lambda) \widehat{Q}_{\mathrm{R},S_z,S'_{g}}(\beta) +\lambda \widehat{Q}_{\mathrm{K},S''_{g}}+
	\\&O\left( \Phi+\sqrt{\log(1/\delta)}\right) \left(\frac{1-\lambda}{\sqrt{n_z}}+\frac{\lambda}{\sqrt{n''_{g}}} \right),
	\end{split}
	\]
	where $\beta$ and $\lambda$ are trade-off hyper-parameters in Eqn.~\eqref{eqn:informlossnn3}
 \end{theorem}

Additionally, to obtain more insights for sampling complexity, we
further bound the population risk in terms of expected imperfectness, at the expense of some tightness. The proof details are deferred to Appendix \ref{sec:proofgeneralization3}.
\begin{corollary}\label{crl:generalizationbound2}
	With the same assumptions as in Theorem~\ref{crl:generalizationbound}, with probability at least $1-O(\phi)-\delta, \delta\in(0,1)$, the population risk satisfies
	\[
	\begin{split}
	R(h_{\bm{W}^{(T)}})&\leq O(\sqrt{\epsilon})+(1-\lambda) Q_{\mathrm{R}}(\beta) +\lambda Q_{\mathrm{K}}+\\
	&O\left( \Phi+\log^{1/4}(1/\delta)\right)\sqrt{\frac{1-\lambda}{\sqrt{n_z}}+\frac{\lambda}{\sqrt{n''_{g}}}},
	\end{split}
	\]
	where  $Q_{\mathrm{R}}(\beta)$ is the expected knowledge-regularized label imperfectness in Definition \ref{def:regularizedimperfectness}, $Q_{\mathrm{K}}$ is the expected knowledge imperfectness in Definition~\ref{def:knwoledgeimperfect}.
\end{corollary}
\begin{remark}\label{rmk:improved_knowledge_effect}
Theorem \ref{crl:generalizationbound} and  Corollary \ref{crl:generalizationbound2} show that by training on the generalized informed risk in Eqn.~\eqref{eqn:informlossnn3}, label and knowledge supervision jointly affect the population risk while introducing a  combination of knowledge-regularized label imperfectness $Q_{\mathrm{R}}(\beta) $ and knowledge imperfectness $Q_{\mathrm{K}}$. The effect of knowledge regularization is controlled by $\beta$ and the trade-off between the two imperfectness terms and the trade-off between the two generalization errors $\frac{1-\lambda}{\sqrt{n_z}}$ and $\frac{\lambda}{\sqrt{n''_{g}}}$ are both controlled by $\lambda$. Thus, this gives us more flexibility to adjust how much domain knowledge is incorporated when it plays different roles in informed learning as discussed in Remark \ref{remark:knowledge_effect}. Also, as  shown by the population risk bounds, we can tune the two hyper-parameters separately --- we can first tune $\beta$ to minimize $Q_{\mathrm{R}}(\beta)$, and then tune $\lambda$  to balance $Q_{\mathrm{R}}(\beta)$ and $Q_{\mathrm{K}}$, and also balance the generalization errors due to sizes of datasets.
\end{remark}

\subsection{Sampling Complexity}\label{sec:samplingcomplexity}

We discuss the choices of hyper-parameters $\beta$ and $\lambda$ in different cases to guarantee a small population risk, and give the sampling complexity in each case, whose details are deferred to Appendix \ref{sec:proofsamplingcomplexity}.

\begin{corollary}[Sampling Complexity]\label{thm:generalization_choiceweight}
	With the same set of assumptions as in Corollary~\ref{crl:generalizationbound2} and setting $\beta^*=\arg\min_{\beta\in[0,1]} Q_{\mathrm{R}}(\beta)$, with probability at least $1-O(\phi)-\delta, \delta\in(0,1)$,  to guarantee a population risk no larger than $\sqrt{\epsilon}$, we have the following cases:
	\begin{enumerate}
		\item[(a)] If $Q_\mathrm{K}\leq \sqrt{\epsilon}$, set $\lambda=1$, the sampling complexity for
		labels is $n_z=0$ and the sampling complexity for knowledge-supervision is $n_{g}\sim O(1/(\epsilon^2-\epsilon^3))$.
		\item[(b)]  If $Q_\mathrm{K}> \sqrt{\epsilon}$ and $\frac{\sqrt{\epsilon}}{Q_{\mathrm{K}}}+\frac{\sqrt{\epsilon}}{Q_{\mathrm{R}}(\beta^*)}\geq 1$, set $\lambda=  \frac{\sqrt{\epsilon}}{Q_{\mathrm{K}}}$, the sampling complexity for
		labels is $n_z\sim O\left(\left( 1/\epsilon-1/\left( \sqrt{\epsilon}Q_{\mathrm{K}}\right)  \right)^2   \right) $ and the sampling complexity for knowledge-supervision is $n_{g}\sim O(1/\left( (\epsilon -\epsilon^2)Q^2_{\mathrm{K}}\right) )$.
		\item[(c)]  If $\frac{\sqrt{\epsilon}}{Q_{\mathrm{K}}}+\frac{\sqrt{\epsilon}}{Q_{\mathrm{R}}(\beta^*)}< 1$, a population risk as low as $\sqrt{\epsilon}$ cannot be achieved no matter what $\lambda$ is and how many samples are used.
	\end{enumerate}
\end{corollary}

\begin{remark}
In practice, unlabeled samples are typically cheaper to obtain than
\emph{labeled} samples.
	If $Q_{\mathrm{K}}\leq \sqrt{\epsilon}$, the domain knowledge is good enough for supervision,
	and thus we can perform purely knowledge-based training without any labeled
	samples and guarantee a population risk no larger than $\sqrt{\epsilon}$ with $n''_{g}\sim O(1/\epsilon^2)$, and hence $n_g\sim O(1/(\epsilon^2-\epsilon^3))$.
	When the knowledge imperfectness $Q_\mathrm{K}> \sqrt{\epsilon}$, we discuss  the following two cases. First, if $\frac{\sqrt{\epsilon}}{Q_{\mathrm{K}}}+\frac{\sqrt{\epsilon}}{Q_{\mathrm{R}}(\beta^*)}\geq 1$, we can choose $\lambda$ from $\left[1-\frac{\sqrt{\epsilon}}{Q_{\mathrm{R}}(\beta^*)},\frac{\sqrt{\epsilon}}{Q_{\mathrm{K}}}\right] $ to control the risk from knowledge and label imperfectness as low as $\sqrt{\epsilon}$.  We thus choose the largest $\lambda=\frac{\sqrt{\epsilon}}{Q_{\mathrm{K}}}$ to reduce the label sampling complexity. In this case, knowledge is not good enough, but label imperfectness is not too large. Thus, we can guarantee a population risk no larger than $\sqrt{\epsilon}$ with labeled samples $n_z\sim O\left(\left( 1/\epsilon-1/\left( \sqrt{\epsilon}Q_{\mathrm{K}}\right)  \right)^2   \right) $ and knowledge supervised samples $n_g\sim O(1/\left( (\epsilon -\epsilon^2)Q^2_{\mathrm{K}}\right) )$.  Finally, if $\frac{\sqrt{\epsilon}}{Q_{\mathrm{K}}}+\frac{\sqrt{\epsilon}}{Q_{\mathrm{R},\beta^*}}< 1$, we cannot guarantee a population risk less than $\sqrt{\epsilon}$ no matter what $\lambda$ is and how many samples are used since the neither knowledge nor labels are of high enough quality. 
	
	In summary, the extreme cases are: Case (a) where the knowledge supervision alone is nearly perfect, and Case (c) where the knowledge and labels are both of low quality. Usually, we are in Case~(b) where knowledge is imperfect but labels (after knowledge regularization) are good enough. In contrast, DNNs without using domain knowledge requires the label imperfectness $Q_{\mathrm{R},0}$ not to exceed $\sqrt{\epsilon}$; otherwise, the population risk
cannot be guaranteed to be no greater than $\sqrt{\epsilon}$. The informed DNNs relaxes this requirement by requiring $\frac{\sqrt{\epsilon}}{Q_{\mathrm{K}}}+\frac{\sqrt{\epsilon}}{Q_{\mathrm{R}}(\beta^*)}\geq 1$. In addition, the incorporation of domain knowledge reduces the labeled sampling complexity from $n_z\sim O(\frac{1}{\epsilon^2})$ in the traditional no-knowledge setting to $n_z\sim O\left(\left( 1/\epsilon-1/\left( \sqrt{\epsilon}Q_{\mathrm{K}}\right)  \right)^2   \right) $. In other words, the incorporation of knowledge is equivalent to $O(\frac{2}{\epsilon^{3/2} Q_{\mathrm{K}}}-\frac{1}{\epsilon Q_{\mathrm{K}}^2})$ labeled samples, establishing a quantitative comparison between  knowledge supervision and labeled samples.
\end{remark}

%% file: Connection.tex
\section{Further Discussions}

\textbf{Summary of analysis.}
The convergence analysis in Theorem \ref{the:convergence} introduces the concept of smooth sets and explains how the neural network output behaves by training on an informed risk.
 The generalization analysis in Theorem \ref{thm:generalizationbound} explicitly shows the two different effects the domain knowledge has on the population risk (i.e.,
 regularizing labels and supplementing labels). Based on this observation,  we propose a generalized informed risk in Eqn. \ref{eqn:informlossnn3} to get more flexibility to control the two effects of knowledge, which is validated by Theorem~\ref{crl:generalizationbound} and its Corollary~\ref{crl:generalizationbound2}. Finally, the sampling complexity in Corollary \ref{thm:generalization_choiceweight} shows the effects of joint knowledge and label supervision in a quantitative way.

\begin{figure*}[!t]	
	\centering
	\subfigure[ ]{
		\label{fig:acc1_fun}
		\includegraphics[width={0.32\textwidth}]{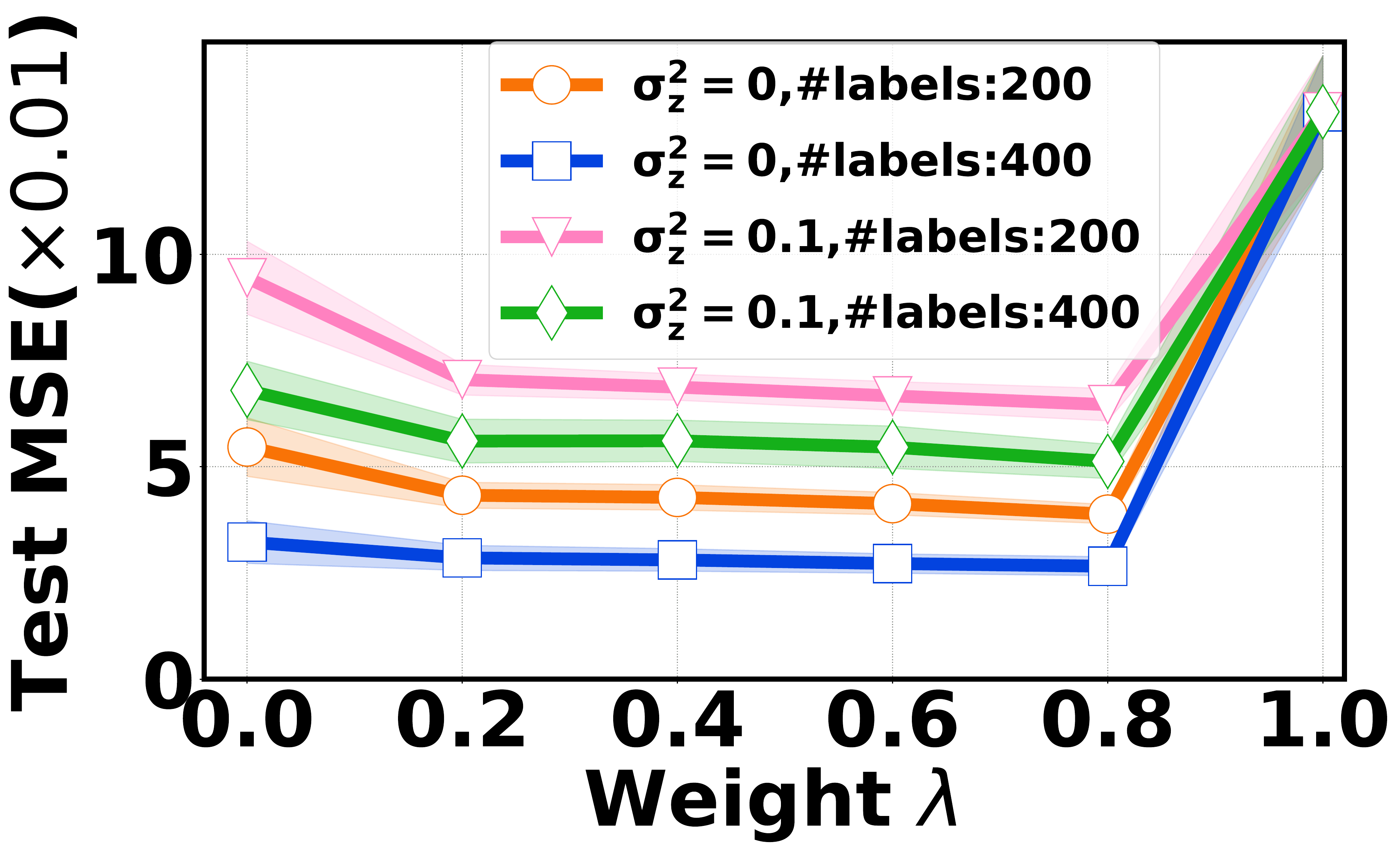}
	}
	\subfigure[]{
		\label{fig:acc2_fun}
		\includegraphics[width=0.32\textwidth]{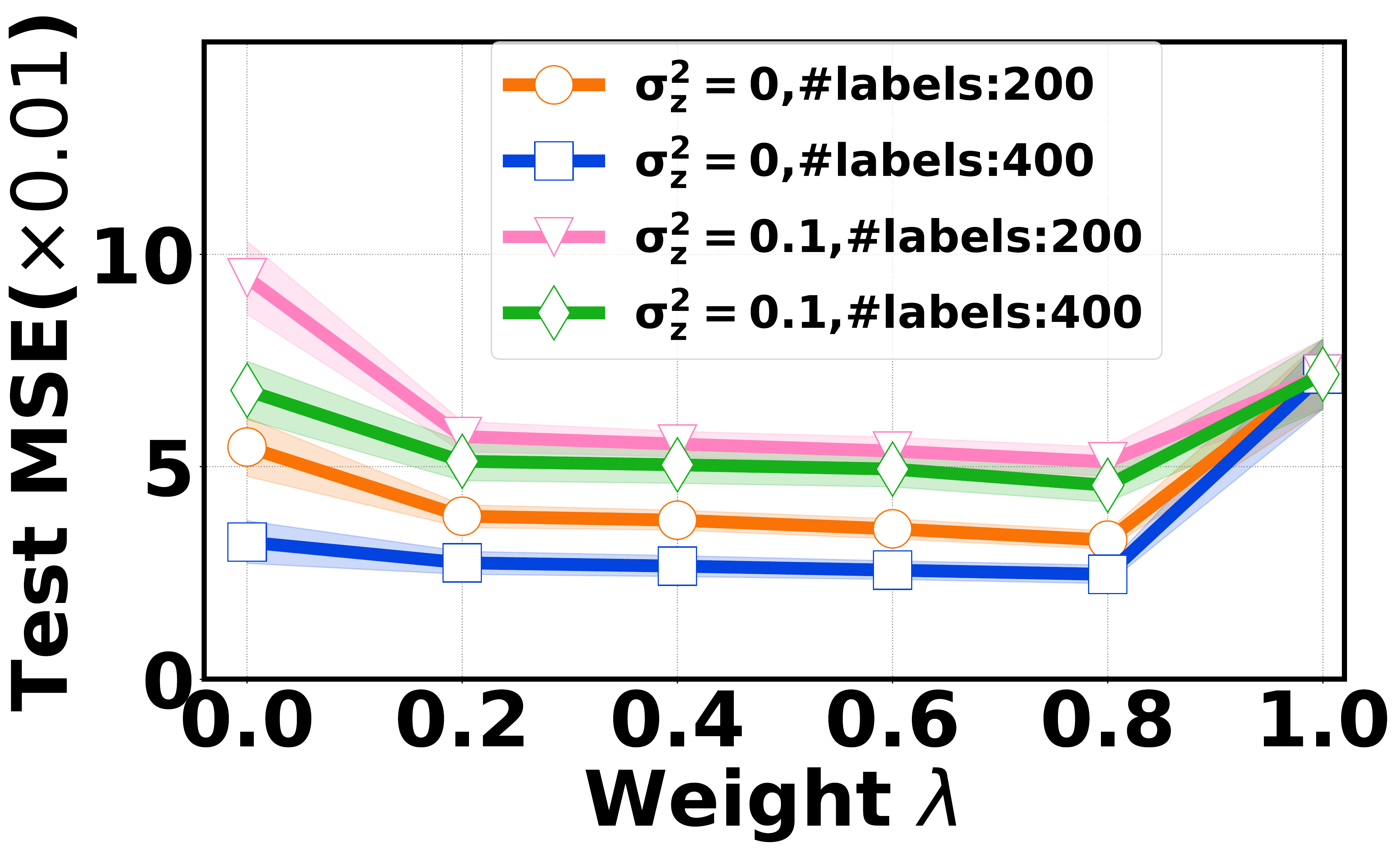}
	}%
	\subfigure[]{
		\label{fig:acc3_fun}
		\includegraphics[width=0.32\textwidth]{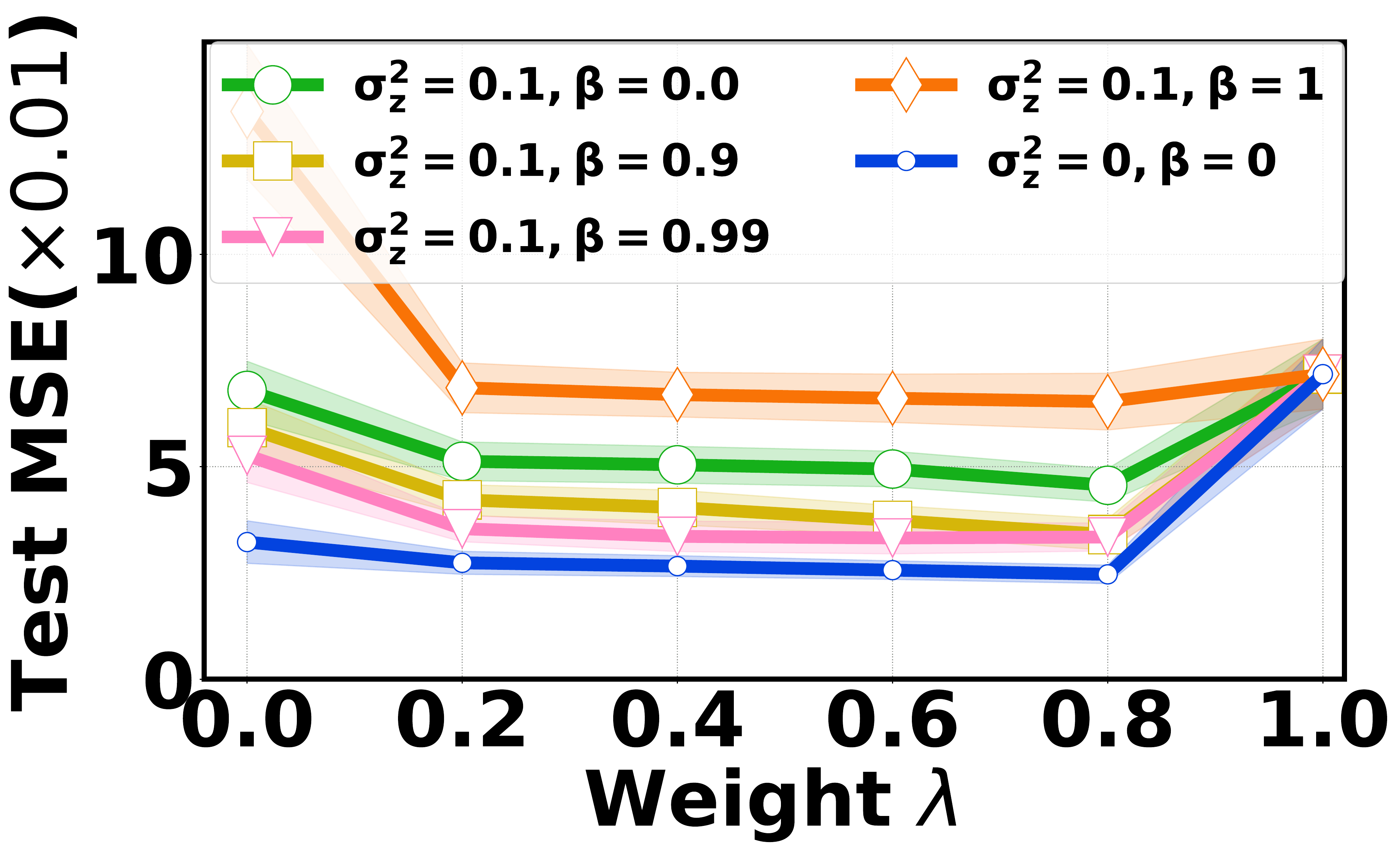}
	}%
	\centering
	\vspace{-0.4cm}	
	\caption{Test MSE under different hyper-parameters.  $\sigma^2_z=0$ means using perfect labels; $\sigma^2_z=0.1$ means using imperfect labels with noise variance 0.1; knowledge imperfectness is determined by $ub$ and $lb$ in the problem setting. (a) Training on the standard informed objective Eqn.~\eqref{eqn:informlossnn} using knowledge with high imperfectness; (b) Training on the standard informed objective Eqn.~\eqref{eqn:informlossnn} using knowledge with low imperfectness; (c) Training on the generalized informed objective Eqn.~\eqref{eqn:informlossnn3} using knowledge with low imperfectness and 400 labels.}
	\vspace{-0.5cm}	
	\label{fig:simulationacc_function}
\end{figure*}

\textbf{Understanding knowledge distillation from
the perspective of informed learning.}
Knowledge distillation is extremely useful
in practice (e.g., for model compression \cite{Knowledge_Distillation_hinton2015distilling}).
Here, we show how our analysis 
complement the existing understanding of
knowledge distillation  \cite{Knowledge_Distillation_hinton2015distilling,Towards_understanding_distillation_phuong2019towards,KD_analysis_rahbar2020unreasonable,KD_inference_dao2021knowledge,distillation_riskbound_ji2020knowledge} from the perspective of hard label and teacher's knowledge imperfectness. In our formulation, hard labels are $\{z_i\}$ in the labeled dataset, whose imperfectness (non-softness) is measured by $Q_{\mathrm{R}}(0)$. In Theorems \ref{thm:generalizationbound}, \ref{crl:generalizationbound}, and Corollary \ref{crl:generalizationbound2}, by viewing the teacher model $g(x)$ as domain knowledge, we show the teacher benefits the student training by providing a regularization gain $\Delta Q_{\mathrm{R},\beta}$, and reducing the sampling complexity of hard labels by Corollary \ref{thm:generalization_choiceweight}. 
The knowledge-regularized label imperfectness $Q_{\mathrm{R},\beta}$ can be less than pure lable imperfectness $Q_{\mathrm{R}}(0)$ because the soft label can smooth the network output within each smooth set. But, given the teacher (knowledge) imperfectness $Q_{\mathrm{K}}$, there exists a trade-off between hard label and teacher supervision.

Importantly, our results are in line with 
the observations and also complement the analysis in \cite{distillation_riskbound_ji2020knowledge}. Specifically,  \cite{distillation_riskbound_ji2020knowledge} 
uses NTK to show that the soft labels provided by a teacher model (knowledge) are easier
to learn than hard labels while hard labels can correct imperfect teachers pointwise,
exhibiting a trade-off between hard labels and the imperfect teacher. We define the hard label and teacher (knowledge) imperfectness, and show that for a neural network with finite width, hard labels and teacher's knowledge compensate for each other within each smooth set.  In consistency with our results, \cite{KD_analysis_rahbar2020unreasonable} based on NTK also presents a trade-off between labels and the imperfect teacher. The teacher  model imperfectness is also observed by \cite{KD_inference_dao2021knowledge} which measures the teacher imperfectness by the squared norm of the difference of the soft label and the true Bayesian class probability.  
Note, however, that our analysis \emph{cannot} adequately explain the benefit of knowledge distillation for the perspective of feature learning due to the inherent limitations of over-parameterization techniques, which are further discussed in \cite{ensemble_distillation_allen2020towards}.

\section{Numerical Results}\label{sec:simulation}

\subsection{Problem Setup}
We consider an informed DNN with domain knowledge
in the form of constraints to learn a Bohachevsky function. The learning task is to learn a relationship $y(x)$. The learner is provided with a dataset with labeled samples $S_z=\left\lbrace (x_i,z_i),i\in[n_z]\right\rbrace$, having possibly noisy labels
$
z_i=y(x_i)+n_i, n_i\sim \mathcal{N}(0, \sigma_z^2),
$
and an unlabeled dataset $S_g=\left\lbrace (x_i),i\in[n_g]\right\rbrace$. Additionally, the learner is informed with the constraint knowledge, which includes an upper bound $g_{\mathrm{ub}}(x)$ and an lower bound $g_{\mathrm{lb}}(x)$ on the true label corresponding to input $x$, i.e. $g_{\mathrm{lb}}(x) \leq y(x)\leq g_{\mathrm{ub}}(x)$. A neural network $h_{\bm{W}}(x)$ is used for learning
and  the metric of interest is the mean square error (MSE) of the network output $h_{\bm{W}}(x)$ with respect to the true label $y(x)$ on a test dataset $S_t$, which is expressed as
$
\hat{R}_{S_t}(h_{\bm{W}})=\frac{1}{2|S_t|}\sum_{(x_i,y_i)\in S_t}\left(h_{\bm{W}}(x_i)-y_i \right)^2.
$
Assume that the relationship to be learned is governed by a multi-dimensional Bohachevsky function
$
y(x)=x\bm{A}\bm{A}^\top x^\top-c\cos\left( a^\top x\right)+c,
$
where $\bm{A}$ is a $b\times b$ matrix, $a$ is a $b$-dimensional vector and $c$ is a constant. 
The constraint knowledge includes an upper
bound model
$
g_{\mathrm{ub}}(x)=x\bm{A}\bm{A}^\top x^\top+ub
$ with $ub\geq 2c$, and an lower bound model
$
g_{\mathrm{lb}}(x)=x\bm{A}\bm{A}^\top x^\top+lb.
$ with $lb\leq 0$.
While it is not strongly convex and hence deviates from the assumptions in our theoretical
analysis, we use ReLU  as
the knowledge-based risk function, i.e., the knowledge-based risk is written as
$
r_{\mathrm{K}}(h_{\bm{W}}(x))=\mathrm{relu}\left(h_{\bm{W}}(x)-g_{\mathrm{ub}}(x) \right)+\mathrm{relu}\left(g_{\mathrm{lb}}(x)-h_{\bm{W}}(x) \right).
$ If $ub-lb$ is larger, the uncertainty of the label given the knowledge is larger --- the knowledge imperfectness is higher. We choose $(lb,ub)$ as $(0,0.6)$ and $(0,0.8)$ respectively to show the performances under low and high knowledge imperfectness.  
More details of the setup are in Appendix~\ref{appendix:simulation_setting}.

\subsection{Results}
The curves of test MSE  with different knowledge and label settings are shown in Fig.~\ref{fig:simulationacc_function}. In all the three figures, test MSE in $\lambda=0$ approximately measures  knowledge-regularized label imperfectness in Definition. \ref{def:regularizedimperfectness}, while test MSE in $\lambda=1$ approximately measures  knowledge imperfectness in Definition~\ref{def:knwoledgeimperfect}.
We first use the training objective Eqn.\eqref{eqn:informlossnn} in Fig.~\ref{fig:acc1_fun} and Fig.~\ref{fig:acc2_fun} to show the effect of adjusting $\lambda$, which controls the knowledge effects (see Remark~\ref{remark:knowledge_effect}).
From both Fig.~\ref{fig:acc1_fun} and Fig.~\ref{fig:acc2_fun}, we see that the test MSE is smaller when there are more labeled samples and when label noise variance is lower.
Importantly, domain knowledge helps reduce the MSE compared with pure label-supervised learning, especially for the cases with fewer labels and high label noise variance. Also, by comparing Fig.~\ref{fig:acc1_fun} and Fig.~\ref{fig:acc2_fun}, we can find that the test MSE is lower when the knowledge imperfectness is lower.  Additionally, Fig.~\ref{fig:acc3_fun} gives the test MSEs training on the generalized objective \eqref{eqn:informlossnn3} under different $\beta$ when the labeled dataset size is 400, showing that the test risk can be reduced by adjusting $\beta$ which controls the knowledge regularization effect (see Remark~\ref{rmk:improved_knowledge_effect}). We can find that by properly adjusting $\beta$, the test MSEs under label noise are very close to that without label noise (the blue line). When $\beta=1$, the test MSE is the highest since no labeled data is used to provide supervision.

More results, including another application of
learning to manage
wireless spectrum, are available in Appendix~\ref{appendix:wireless}.

%% file: Conclusion.tex
\section{Conclusion}

In this paper, we consider an informed DNN with domain knowledge integrated
with its training risk function.
 We
quantitatively demonstrate that domain knowledge can improve
the generalization performance and reduce the sampling complexity,
while also impacting the point to which the network output converges.
 Our analysis also reveals
 that
 knowledge affects the generalization performance in two ways:
 regularizing the label supervision, and supplementing the labeled
 samples.
Finally,  
we discuss how an informed DNN relates to other learning
frameworks. 

\section*{Acknowledgment}
This work was supported in part by the U.S. NSF under CNS-1910208.

%% file: TrainingAlgorithm.tex
\section{Training Algorithm}\label{appendix:training_algorithm}

To train the knowledge-informed DNN,
we consider a
gradient descent approach in Algorithm~\ref{alg:neuraltrain}.
This training approach has also been commonly considered in the literature \cite{Convergence_zhu_allen2019convergence,Convergence_Gu_zou2019improved,GD_Neural_du2019gradient}
for theoretical analysis of standard DNNs without domain knowledge.

\begin{algorithm}[!h]
	\caption{Informed Neural Network Training by Gradient Descent}\label{alg:neuraltrain}
	\begin{algorithmic}
		\STATE \textbf{Initialization}: Initialize each entry of weights $\bm{W}_0^{(0)}$, $\bm{W}_l^{(0)}, l\in[L]$ independently by $\mathcal{N}\left(0, \frac{2}{m} \right)$ and each entry of $\bm{V}^{(0)}$ independently by  $\mathcal{N}\left(0, \frac{1}{d} \right).$
		\FOR {$t=0,\cdots,T-1$}
		\STATE  Update the weights as $\bm{W}^{(t+1)}=\bm{W}^{(t)}+\eta \bigtriangledown_{\bm{W}}\hat{R}_{\mathrm{I}}\left(\bm{W}^{(t)} \right)$.
		\ENDFOR
		\STATE \textbf{Output}: $\bm{W}^{(T)}$.
	\end{algorithmic}
\end{algorithm} 

%% file: ProofConvergence.tex
\section{Notations, Key Lemmas and Proofs of Main Results in Section \ref{sec:main} and Section \ref{sec:improved_objective}}\label{sec:proofmain} 

\subsection{Further Notations}
Before the proofs, we list some additional  notations as below.
Denote $n'=|S_z|+|S_g|=n_z+n_g$. We assign the samples in the dataset $S_z$ with indices from $1$ to $n_z$ and the samples in the dataset $S_g$ with indices from $n_z+1$ to $n'$. The informed risk of an informed DNN in Eqns.~\eqref{eqn:informlossnn},\eqref{eqn:informlossnn3} can be re-written as
\begin{equation}\label{eqn:informlossnn2}
\begin{split}
\hat{R}_{\mathrm{I}}\left(\bm{W}\right) =\sum_{i=1}^{n'}\left[ \mu_ir\left(h_{\bm{W}}\left(x_i \right),z_i \right)+\lambda_i r_{\mathrm{K}}\left( h_{\bm{W}}\left(x_i \right),g(x_i)\right) \right],
\end{split}
\end{equation}
where $\sum_{i=1}^{n'}\left( \mu_i+\lambda_i\right) =1$. Thus, in Eqn.~\eqref{eqn:informlossnn}, 
we have $\mu_i=\frac{1-\lambda}{n_z}\mathds{1}(x_i\in S_z)$ and $\lambda_i=\frac{\lambda}{n_g}\mathds{1}(\!x_i\in S_g\!)$; in Eqn.~\eqref{eqn:informlossnn3}, we have $\mu_i=\frac{(1-\lambda)(1-\beta)}{n_z}\mathds{1}(x_i\in S_z)$ and $\lambda_i=\frac{(1-\lambda)\beta}{n_g'}\mathds{1}(\!x_i\in S_g'\!)+\frac{\lambda}{n_g''}\mathds{1}(\!x_i\in S_g''\!)$. We prove convergence for the above three risks.

For any input $x_i,i\in[n']$, we denote the DNN output with respect to weight $\bm{W}$ as $h_{\bm{W},i}=h_{\bm{W}}(x_i)$. To express the output of the ReLu activation of the $l$-th  layer for an input sample $x_i$, for $l\in[L]$ and $i\in[n]$, we denote a diagonal matrix $\bm{D}_{l,i}$ with its $j$-th (for $j\in[m]$) diagonal entry as $\mathds{1}\left( \left[ \bm{W}_lh_{l-1}\right]_j \geq 0\right) $.
Thus, given the input $x_i$, the DNN outpoput can be expressed as
\begin{equation}
h_{\bm{W},i}=\bm{VD}_{L,i}\bm{W}_{L}\bm{D}_{L-2,i}\cdots \bm{D}_{0,i}\bm{W}_0x_i.
\end{equation}
Also, we denote the informed risk for hypothesis $h\in\mathcal{H}$ and input $x_i$ as
 \begin{equation}
 r_{\mathrm{I},i}=\mu_ir\left(h(x_i),z_i \right)+\lambda_i r_{\mathrm{K}}\left( h(x_i),g(x_i)\right)
\end{equation}
The gradient of informed risk with respect to the hypothesis output is
\begin{equation}\label{eqn:informedriskgradient}
u_i(h(x_i))=\bigtriangledown_h r_{\mathrm{I},i}(h(x_i))= \mu_i\bigtriangledown_h r\left(h(x_i),z_i \right)+\lambda_i \bigtriangledown_h r_{\mathrm{K}}\left( h(x_i),g(x_i)\right).
\end{equation}

After constructing the smooth sets,
denote for the $k$th smooth set, the sum of indices as $M_k=\sum_{\mathcal{I}_{\phi,k}}\left( \mu_{i}+\lambda_{i}\right) $.
Denote the sum risk of the $k$th smooth set for hypothesis $h\in\mathcal{H}$ as
\begin{equation}
\bar{r}_{\mathrm{I},k}(h(x_i))=\sum_{i\in\mathcal{I}_{\phi,k}} r_{\mathrm{I},i}(h(x_i)).
\end{equation}
Thus, the effective label given in Definition~\ref{def:efflabel} is written as
$
y_{\mathrm{eff},k}=\arg\min_{h}\bar{r}_{\mathrm{I},k}(h)
$ with $h$ in the space of network output,
  and the optimal effective risk is written as
$
  r_{\mathrm{eff},k}=\bar{r}_{\mathrm{I},k}(y_{\mathrm{eff},k}).
$

We then give some key technical lemmas which are the foundations for our further analysis.
The proofs for these lemmas are shown in Appendix~\ref{sec:proof2nd}.

\subsection{Forward Perturbation Regarding Inputs}
The forward perturbation for weights in the weight update range is proved in \cite{Convergence_zhu_allen2019convergence}, However, to characterize the smooth sets, it is important to prove forward perturbation for inputs in a smooth set,  which is given as follows.

\begin{lemma}\label{lma:forwardinputperturbation}
	For any $i\in\mathcal{I}_{\phi,k}, k\in[N]$, let $h_{l,k}=h_l(x_k')$, $h_{l,i}=h_l(x_i)$, and $f_{l,k}=\bm{W}_lh_{l-1}(x_k')$,  $f_{l,i}=\bm{W}_lh_{l-1}(x_i)$ and denote $\bm{D}_{l,i,k}'\in\mathbb{R}^{m\times m}$ as the diagonal matrix with $[\bm{D}_{l,i,k}']_{j,j}=\mathds{1}([f^{(0)}_{l,i}]_j\geq 0)-\mathds{1}([f^{(0)}_{l,k}]_j\geq 0)$. Assuming $\phi\leq O(L^{-9/2}\log^{-3}(m)\log^{-3/4}(1/\phi))$, we have
with probability at least $1-\phi$ over the randomness of $\bm{W}^{(0)}$,\\
(a) At initialization, $\|\bm{D}_{l,i,k}'\|_0\leq O(m\phi^{2/3}L\log^{1/2}(1/\phi))$\\
(b) For $\bm{W}\in\mathcal{B}(\bm{W}^{(0)},\tau)$ with $\tau\leq O(\phi^{3/2})$ we have $\|h_{l,i}-h_{l,k}\|\leq O(L^{5/2}\phi\sqrt{\log(m)\log(1/\phi)})$ and $\|f_{l,i}-f_{l,k}\|\leq O(L^{5/2}\phi\sqrt{\log(m)\log(1/\phi)})$.
\end{lemma}
The proof of Lemma \ref{lma:forwardinputperturbation} is given in Section \ref{sec:proofforwardinputperturbation}. The forward perturbation regarding inputs indicates the smoothness property of neural networks with respect to inputs. For compactness, we absorb the logarithmically increasing terms into $\tilde{O}$ and denote $\widetilde{O}(L^{5/2}\phi\log^{1/2}(m))=O(L^{5/2}\phi\sqrt{\log(m)\log(1/\phi)})$, $\widetilde{O}(L^{-9/2}\log^{-3}(m)) =O(L^{-9/2}\log^{-3}(m)\log^{-3/4}(1/\phi))$ in the following analysis.

\subsection{Properties of Strong Convexity}\label{sec:strongconvexity}
Since our analysis is based on strongly convex risk functions, we give some key properties of strongly convex functions.
\begin{lemma}[Properties of Strong Convexity]\label{lma:knowledgeriskfuncgradient}
	If a strongly convex function $r(h)$ has a minimum value of $r(h^*)=r_{\min}$  and the eigenvalues of its Hessian matrix lie in $[\rho, 1]$, then we have$\left\| \bigtriangledown r(h)\right\|^2\leq  2\left( r\left(h \right)-r_{\min}\right) $,
	$
	\left\| \bigtriangledown r(h)\right\|^2\geq 2\rho\left( r\left(h \right)-r_{\min}\right) $
	and $	\|h^*-h\|\leq \frac{2}{\rho}\left\| \bigtriangledown r\left(h \right)\right\|$.
\end{lemma}

\begin{lemma}\label{lma:cellriskgradientbound}
	If the risk functions $r$ and $r_{\mathrm{K}}$ are strongly convex with their eigenvalues of Hessian matrices in $[\rho, 1]$, then we have for hypothesis $h\in\mathcal{H}$, if $\|h(x_i)-h(x_k')\|\leq \widetilde{O}(L^{5/2}\phi\log^{1/2}(m))$ for $i\in\mathcal{I}_{\phi,k}, k\in[N]$, the sum risk gradient for a smooth set with $\phi\leq \widetilde{O}(L^{-9/2}\log^{-3}(m))$ with respect to $h$ satisfies,
	\[	
	\begin{split}
	&\|\sum_{i\in\mathcal{I}_{\phi,k}}u_i(h(x_i))\|^2\geq 2M_k\rho\left(\bar{r}_{\mathrm{I},k}-r_{\mathrm{eff},k}\right)-M^2_k\widetilde{O}(L^{5/2}\phi\log^{1/2}(m))\\ &\|\sum_{i\in\mathcal{I}_{\phi,k}}u_i(h(x_i))\|^2\leq 2M_k\left(\bar{r}_{\mathrm{I},k}-r_{\mathrm{eff},k}\right)+M^2_k\widetilde{O}(L^{5/2}\phi\log^{1/2}(m)) \\
	& \sum_{i\in\mathcal{I}_{\phi,k}}(\mu_i+\lambda_i)\left\| h\left(x_i \right)-y_{\mathrm{eff},k} \right\|^2
	\leq \frac{1}{\rho^2}O\left( \bar{r}_{\mathrm{I},k}-r_{\mathrm{eff},k}+M_k\widetilde{O}(L^{5/2}\phi\log^{1/2}(m))\right),
	\end{split}
	\]
	where $u_i$ is defined in Eqn.~\eqref{eqn:informedriskgradient} and $M_k=\sum_{\mathcal{I}_{\phi,k}}\left( \mu_{i}+\lambda_{i}\right) $.
\end{lemma}
Lemma \ref{lma:knowledgeriskfuncgradient} and \ref{lma:cellriskgradientbound} are proved in Section \ref{sec:proofcellriskgradientbound}.

\subsection{Proof of Theorem \ref{the:convergence}. }\label{sec:proofconvergence}
In this section, we prove the convergence for informed risks in Eqn.~\eqref{eqn:informlossnn}, Eqn.~\eqref{eqn:informlossnn3}. First, the gradient lower bound, semi-smoothness of the risk function, and initialized risk bound are proved.

\begin{lemma}[Gradient Lower Bound]\label{lma:gradientboud}.
 For any $\bm{W}: \|\bm{W}-\bm{W}^{(0)}\|\leq \tau$ $ \tau=O(N^{-9/2}\phi^{3/2}\rho^{3/2}\bar{\lambda}^{3/2}\alpha^{3/2}L^{-15/2}\log^{-3/2}(m)) $ and $\phi\leq \widetilde{O}(L^{-9/2}\log^{-3}(m))$, with Assumption \ref{asp:smoothset} satisfied, we have with probability at least $1-O(\phi)$ over the randomness of $\bm{W}^{(0)}$, the gradient of label-based data risk satisfies
	\[
\left\|\bigtriangledown_{\bm{W}}\hat{R}_{\mathrm{I}}\left(\bm{W} \right) \right\|_F^2\geq  \Omega\left(\frac{\alpha m\phi\rho\bar{\lambda}}{dN^2}\right) \left( \hat{R}_{\mathrm{I}}\left(\bm{W} \right) -\hat{R}_{\mathrm{eff}}-\widetilde{O}(L^{5/2}\phi\log^{1/2}(m))\right).
\]
where $\hat{R}_{\mathrm{eff}}=\sum_{k=1}^Nr_{\mathrm{eff},k}$, and $\bar{\lambda}$ is a parameter with lower bound $\Omega(\min(1-\lambda, \lambda)\mathds{1}(\lambda\in(0,1))+\mathds{1}(\lambda \in\{0,1\}))$.
\end{lemma}
The proof of Lemma \ref{lma:gradientboud} can be found in Section \ref{sec:proofgradientboud}.

\begin{lemma}\label{lma:risksmoothness}
 For any $\bm{W}$, $\bm{W}'\in \mathcal{B}\left( \bm{W}^{(0)},\tau\right) $, $\tau\in\left[\Omega(d^{3/2}m^{-3/2}L^{-5/2}\log^{-3/2}(m)),  O(L^{-9/2}[\log^{-3}(m)])\right] $ and $\phi\leq \widetilde{O}(L^{-9/2}\log^{-3}(m))$, with probability at least $1-O(\phi)$ over the randomness of $\bm{W}^{(0)}$, we have
\[
\begin{split}
\hat{R}_{\mathrm{I}}\left(\bm{W}' \right)\leq& \hat{R}_{\mathrm{I}}\left(\bm{W} \right)+\left\langle \bigtriangledown_{\bm{W}}\hat{R}_{\mathrm{I}}\left( \bm{W}\right) ,\bm{W}'-\bm{W}\right\rangle+O(L^2m/d)\left\|\widehat{\bm{W}} \right\|^2\\
&\quad +\left( \sqrt{\left( \hat{R}_{\mathrm{I}}\left(\bm{W}\right)-\hat{R}_{\mathrm{eff}}-\widetilde{O}(L^{5/2}\phi\log^{1/2}(m))\right)}\right) O\left( N^{1/2}\tau^{1/3}L^{5/2}\sqrt{m\log(m)}d^{-1/2}\right)\left\| \widehat{\bm{W}}\right\|
\end{split}
\]
\end{lemma}
The proof of Lemma \ref{lma:risksmoothness} can be found in Section \ref{sec:proofrisksmoothness}.

\begin{lemma}\label{initialized risk}
	If $m\geq \Omega\left(L\log(NL\phi^{-1}) \right) $ and $\phi\leq \widetilde{O}(L^{-9/2}\log^{-3}(m))$, with probability at least $1-O(\phi)$ over the randomness of $\bm{W}^{(0)}$, at initialization, we have for any $x_i, i\in[n']$,
	\[
	\left\| h_{\bm{W}^{(0)},i}\right\| \leq O\left(\log^{1/2}(1/\phi) \right),
	\]
	\[\text{ and }
	\hat{R}_{\mathrm{I}}\left(\bm{W}^{(0)} \right)-\hat{R}_{\mathrm{eff}}\leq O\left(\log^{1/2}(1/\phi) \right).
	\]
\end{lemma}
The proof of Lemma \ref{initialized risk} can be found in Section \ref{sec:proofinitialized risk}.

\textbf{Proof of Theorem \ref{the:convergence}. }
\begin{proof}
\textbf{Convergence of the informed risk.}
We first assume $\tau=\frac{\Gamma}{\sqrt{m}}$ with $\Gamma=Nd^{1/2}\phi^{-1/2}\rho^{-1/2}\bar{\lambda}^{-1/2}\alpha^{-1/2}$.
Hence, with the choice of $m$, we have $\tau=O(N^{-9/2}\phi^{3/2}\rho^{3/2}\bar{\lambda}^{3/2}\alpha^{3/2}L^{-15/2}\log^{-3/2}(m))$.
We get the recursion inequality based on gradient descent.
By the weight update rule of gradient descent, we have $\bm{W}^{(t)} -\bm{W}^{(t-1)} =-\eta \bigtriangledown\hat{R}_{\mathrm{I}}\left(\bm{W}^{(t-1)}\right)$. Let $\Psi=\widetilde{O}(L^{5/2}\phi\log^{1/2}(m))$. By Lemma~\ref{lma:risksmoothness}, we have
\begin{equation}\label{eqn:iteration}
\begin{split}
&\hat{R}_{\mathrm{I}}\left(\bm{W}^{(t+1)} \right)-\hat{R}_{\mathrm{eff}}-\Psi\\
\leq& \hat{R}_{\mathrm{I}}\left(\bm{W}^{(t)}\right)-\hat{R}_{\mathrm{eff}}-\Psi-\left( \eta-O(\eta^2L^2m/d) \right) \left\| \bigtriangledown\hat{R}_{\mathrm{I}}\left( \bm{W}^{(t)}\right)\right\|^2 \\&
+\eta\sqrt{2N\left( \hat{R}_{\mathrm{I}}\left(\bm{W}^{(t)}\right)-\hat{R}_{\mathrm{eff}}-\Psi\right)}O\left( \tau^{1/3}L^{5/2}\sqrt{m\log(m)}d^{-1/2}\right)\left\|\bigtriangledown\hat{R}_{\mathrm{I}}\left( \bm{W}^{(t)}\right) \right\|\\
\leq& \hat{R}_{\mathrm{I}}\left(\bm{W}^{(t)}\right)-\hat{R}_{\mathrm{eff}}-\Psi- \Omega\left(\eta\right) \left\|\bigtriangledown\hat{R}_{\mathrm{I}}\left( \bm{W}^{(t)}\right) \right\|^2\\&+\eta \Omega\left( N^{3/2}\tau^{1/3}L^{5/2}\log^{1/2}(m)\phi^{-1/2}\rho^{-1/2}\bar{\lambda}^{-1/2}\alpha^{-1/2}\right)\left\|\bigtriangledown\hat{R}_{\mathrm{I}}\left( \bm{W}^{(t)}\right) \right\|^2\\
\leq& \hat{R}_{\mathrm{I}}\left(\bm{W}^{(t)}\right)-\hat{R}_{\mathrm{eff}}-\Psi- \Omega\left(\eta\right) \left\|\bigtriangledown\hat{R}_{\mathrm{I}}\left( \bm{W}^{(t)}\right) \right\|^2,
\end{split}
\end{equation}
where the second inequality holds by the choice of $\eta=O(\frac{d}{L^2m})$ such that $O(\eta L^2m/d)=O(1)$ and the gradient lower bound in Lemma \ref{lma:gradientboud}, and the last inequality holds by the choice of $m\geq\Omega\left(N^{11}L^{15}d\phi^{-4}\rho^{-4}\bar{\lambda}^{-4}\alpha^{-4}\log^3(m) \right) $ such that $\Omega\left( N^{3/2}\tau^{1/3}L^{5/2}\log^{1/2}(m)\phi^{-1/2}\rho^{-1/2}\bar{\lambda}^{-1/2}\alpha^{-1/2}\right)\leq O(1)$.

Further, by Lemma \ref{lma:gradientboud}, we have
\[
\begin{split}
\hat{R}_{\mathrm{I}}\left(\bm{W}^{(t+1)} \right)-\hat{R}_{\mathrm{eff}}-\Psi
\leq \left( 1- \Omega\left(\frac{\eta \alpha m\phi\rho\bar{\lambda}}{dN^2}\right)\right) \left( \hat{R}_{\mathrm{I}}\left(\bm{W}^{(t)} \right)-\hat{R}_{\mathrm{eff}}-\Psi\right).
\end{split}
\]

Based on the iteration of the recursion inequality, with probability at least $1-O(\phi)$, we have
\[
\begin{split}
&\hat{R}_{\mathrm{I}}\left(\bm{W}^{(t)} \right)-\hat{R}_{\mathrm{eff}}-\Psi\\
\leq	&\left( 1-\Omega\left(\frac{\eta \alpha m\phi\rho\bar{\lambda}}{dN^2}\right)\right)^{t}  \left( \hat{R}_{\mathrm{I}}\left(\bm{W}^{(0)} \right)-\hat{R}_{\mathrm{eff}}-\Psi\right)\\
\leq  &\left( 1-\Omega\left(\frac{\eta \alpha m\phi\rho\bar{\lambda}}{dN^2}\right) \right)^{t}O(\log^{1/2}(1/\phi)),
\end{split}
\]
where the last inequality comes from Lemma \ref{initialized risk}.
Then, by taking logarithm, we get
\[
\begin{split}
\ln\left(\left( \hat{R}_{\mathrm{I}}\left(\bm{W}^{(t)}  \right)-\hat{R}_{\mathrm{eff}}-\Psi\right)\right) &\leq t\ln\left( 1-\Omega\left(\frac{\eta \alpha m\phi\rho\bar{\lambda}}{dN^2}\right) \right)+\frac{1}{2}\ln\log(1/\phi)\\
&\leq -t\Omega\left(\frac{\eta \alpha m\phi\rho\bar{\lambda}}{dN^2}\right)+\frac{1}{2}\ln\log(1/\phi).
\end{split}
\]
Since $\eta=O(\frac{d}{L^2m})$, after $T= O\left(\frac{L^2N^2}{\phi\rho\bar{\lambda}\alpha}\ln(\epsilon^{-1}\log(\phi^{-1})) \right)$ iterations, for any $\epsilon>0$, we have
\[
 \hat{R}_{\mathrm{I}}\left(\bm{W}^{(T)} \right)-\hat{R}_{\mathrm{eff}} \leq O(L^{5/2}\phi\log^{1/2}(1/\phi)\log^{1/2}m)+ \epsilon.
\]
By setting $\phi$ as $\phi\log^{1/2}(1/\phi))\leq \epsilon L^{-5/2}\log^{-1/2}m$ (which satisfies the assumption of $\phi$ in Theorem \ref{the:convergence}), we can bound $\hat{R}_{\mathrm{I}}\left(\bm{W}^{(t)} \right)-\hat{R}_{\mathrm{eff}}$ by a small positive quantity $\epsilon$.

\textbf{Verify the weight update range.}
Now, we verify that the assumption $\left\| \bm{W}^{(t)}-\bm{W}^{(0)}\right\|\leq \frac{\Gamma}{\sqrt{m}}$ holds. Denote $\bar{R}_{\mathrm{I}}\left(\bm{W}^{(t)} \right)=\hat{R}_{\mathrm{I}}\left(\bm{W}^{(t)} \right)-\hat{R}_{\mathrm{eff}}-\Psi$.
By Eqn.~\eqref{eqn:iteration}, we have
\[
\begin{split}
\bar{R}_{\mathrm{I}}\left(\bm{W}^{(t+1)} \right)-\bar{R}_{\mathrm{I}}\left(\bm{W}^{(t)} \right)
\leq - \Omega\left(\eta\right) \left\|\bigtriangledown\hat{R}_{\mathrm{I}}\left( \bm{W}^{(t)}\right) \right\|^2.
\end{split}
\]
Then, we have
\[
\begin{split}
&\sqrt{\bar{R}_{\mathrm{I}}\left(\bm{W}^{(t+1)} \right)}-\sqrt{\bar{R}_{\mathrm{I}}\left(\bm{W}^{(t)} \right)}=\frac{\bar{R}_{\mathrm{I}}\left(\bm{W}^{(t+1)} \right)-\bar{R}_{\mathrm{I}}\left(\bm{W}^{(t)} \right) }{\sqrt{\bar{R}_{\mathrm{I}}\left(\bm{W}^{(t+1)} \right)}+\sqrt{\bar{R}_{\mathrm{I}}\left(\bm{W}^{(t)} \right)}}\\
\leq&
\frac{ - \Omega\left(\eta\right) \left\|\bigtriangledown\hat{R}_{\mathrm{I}}\left( \bm{W}^{(t)}\right) \right\|^2 }{2\sqrt{\bar{R}_{\mathrm{I}}\left(\bm{W}^{(t)} \right)}}\leq -O\left(\frac{m^{1/2}\phi^{1/2}\rho^{1/2}\bar{\lambda}^{1/2}\alpha^{1/2}}{d^{1/2}N} \right) \left\|\eta\bigtriangledown\hat{R}_{\mathrm{I}}\left( \bm{W}^{(t)}\right) \right\|
\end{split}
\]
where the last inequality follows from Lemma \ref{lma:gradientboud}.

By the triangle inequality, for any $t\in[T]$, we have
\begin{equation}\label{eqn:proofweightrange}
\begin{split}
&\left\| \bm{W}^{(t)}-\bm{W}^{(0)}\right\|\leq \sum_{s=0}^{t}\left\|\eta \bigtriangledown \hat{R}_{\mathrm{I}}\left(\bm{W}^{(s)} \right)\right\|\\
\leq& O\left(\frac{d^{1/2}N}{m^{1/2}\phi^{1/2}\rho^{1/2}\bar{\lambda}^{1/2}\alpha^{1/2}} \right)\sqrt{\bar{R}_{\mathrm{I}}\left(\bm{W}^{(0)} \right)}\leq O\left(\frac{\Gamma}{\sqrt{m}}\right),
\end{split}
\end{equation}
where $\Gamma=Nd^{1/2}\phi^{-1/2}\rho^{-1/2}\bar{\lambda}^{-1/2}\alpha^{-1/2} $. Hence,
with the choice of $m$,  we have $\tau=O(N^{-9/2}\phi^{3/2}\rho^{3/2}\bar{\lambda}^{3/2}\alpha^{3/2}L^{-15/2}\log^{-3/2}(m))$.

\textbf{Convergence of network output.}
By Lemma~\ref{lma:cellriskgradientbound},we have
\[
	\sum_{k=1}^N\sum_{i\in\mathcal{I}_{\phi,k}}(\mu_i+\lambda_i)\left\| h_{\bm{W}^{(T)}}\left(x_i \right)-y_{\mathrm{eff},k} \right\|^2
\leq \frac{1}{\rho^2}O\left( \left(\hat{R}_{\mathrm{I}}\left(\bm{W}^{(T)} \right) -\hat{R}_{\mathrm{eff}}\right)+\widetilde{O}(L^{5/2}\phi\log^{1/2}(m))\right)\leq O(\epsilon).
\]
Denoting $k(x_i)$ as the index of the cell containing $x_i$ and rearranging the above summation, we have
\[
\sum_{x_i\in S_z}\mu_i \left\| h_{\bm{W}^{(T)}}\left(x_i \right)-y_{\mathrm{eff},k(x_i)} \right\|^2+\sum_{x_j\in S_g}\lambda_j \left\| h_{\bm{W}^{(T)}}\left(x_j\right)-y_{\mathrm{eff},k(x_j)} \right\|^2\leq O(\epsilon).
\]

\end{proof}

%% file: ProofGeneralization.tex
\subsection{Proof of Generalization}
In this section, we prove the generalization bound based on Rademacher complexity. We first present the bound of Rademacher complexity for neural networks.
\begin{lemma}\label{lma:boundRad}[Theorem 3.3 in \cite{Rademacher_neural_bartlett2017spectrally},Lemma A.3 in \cite{Overparameterization_needed_chen2019much}]
If risk functions are 1-Lipschitz continuous, with probability at least $1-L\exp(-\Omega(m))$, the Rademacher complexity $\mathfrak{R}_S(\mathcal{F})$ for the risk set $\mathcal{F}=\left\lbrace r\left(h_{\bm{W}}\left(x \right),y \right): (x,y)\in \mathcal{X}\times\mathcal{Y}, \left\|\bm{W}-\bm{W}^{(0)}\right\|\leq \tau\right\rbrace$, $\tau=\frac{\Gamma}{\sqrt{m}}$ with $\Gamma=Nd^{1/2}\phi^{-1/2}\rho^{-1/2}\bar{\lambda}^{-1/2}\alpha^{-1/2}$	given a dataset $S$ of $n$ samples is bounded as
\begin{equation}
\mathfrak{R}_S(\mathcal{F})\leq  \Phi/\sqrt{n},
\end{equation}
where $\Phi=O\left( 4^L L^{3/2} m^{1/2}\phi^{-b-1/2}d \rho^{-1/2}\bar{\lambda}^{-1/2}\alpha^{-1/2} \right)$.
	\end{lemma}

Then, we need to bound the error between effective labels in Definition~\ref{def:efflabel} and the output of optimal hypothesis in Definitions~\ref{def:knwoledgeimperfect} and \ref{def:regularizedimperfectness}.
\begin{lemma}\label{lma:boundefflabel}
Consistent with Definition \ref{def:efflabel}, assume that for any smooth set $k\in \mathcal{U}_{\phi}(S_z)$ (containing at least one labeled sample), $y_{\mathrm{eff},k}$ equivalently minimizes $\sum_{i\in I_{\phi,k}} \frac{1-\beta}{n_z}\mathds{1}(x_i\in S_z)r\left(h,\!z_i \right)+\frac{\beta}{n_g'}\mathds{1}(x_i\in S_g')r_{\mathrm{K}}\left( h,g_i\right)$, and for any smooth set $k\in [N]\setminus \mathcal{U}_{\phi}(S_z)$ (not containing labeled sample), $y_{\mathrm{eff},k}$ equivalently minimizes $\sum_{i\in I_{\phi,k}}\frac{1}{n_g''}r_{\mathrm{K}}\left( h,g_i\right)$.\\
\textbf{(a)}	Letting $h^*_{\mathrm{K}}$ and $h^*_{\mathrm{R},\beta}$ be the optimal hypothesis for empirical risks in Definitions~\ref{def:knwoledgeimperfect} and \ref{def:regularizedimperfectness}, respectively,
we have with probability at least $1-O(\phi)$ over the randomness of $\bm{W}^{(0)}$,
			\[
	\frac{1}{n_{g}''}\sum_{S_{g}''}\left\| h^*_{\mathrm{K},i}-y_{\mathrm{eff},k(x_i)}\right\|^2\leq \widetilde{O}\left(L^{5/4}\phi^{1/2}\log^{1/4}(m)\right),
		\]
and
		\[
	\frac{1-\beta}{n_z}\sum_{S_z}\left\|h^*_{\mathrm{R},\beta,i}-y_{\mathrm{eff},k(x_i)}\right\|^2 +\frac{\beta}{n_{g}'}\sum_{S_{g}'}\left\| h^*_{\mathrm{R},\beta,i}-y_{\mathrm{eff},k(x_i)}\right\|^2\leq \widetilde{O}\left(L^{5/4}\phi^{1/2}\log^{1/4}(m)\right),
	\]
	where $\widetilde{O}\left(L^{5/4}\phi^{1/2}\log^{1/4}(m)\right)=O(L^{5/4}\phi^{1/2}\log^{1/4}(1/\phi)\log^{1/4}m)$.\\
\textbf{(b)}		Letting $\bar{h}^*_{\mathrm{K}}$ and $\bar{h}^*_{\mathrm{R},\beta}$ be the optimal hypothesis for the expected risks in Definitions~\ref{def:knwoledgeimperfect} and \ref{def:regularizedimperfectness}, respectively,
we have with probability at least $1-O(\phi)-\delta$,
	\[
	\frac{1}{n_{g}''}\sum_{S_{g}''}\left\| \bar{h}^*_{\mathrm{K},i}-y_{\mathrm{eff},k(x_i)}\right\|^2\leq \widetilde{O}\left(L^{5/4}\phi^{1/2}\log^{1/4}(m)\right)+O(\sqrt{\frac{\log(1/\delta)}{n_{g}''}}),
	\]
and
	\[
	\begin{split}
	\frac{1-\beta}{n_z}\sum_{S_z}\left\|\bar{h}^*_{\mathrm{R},\beta,i}-y_{\mathrm{eff},k(x_i)}\right\|^2 +\frac{\beta}{n_{g}'}\sum_{S_{g}'}\left\| \bar{h}^*_{\mathrm{R},\beta,i}-y_{\mathrm{eff},k(x_i)}\right\|^2\leq  \widetilde{O}\left(L^{5/4}\phi^{1/2}\log^{1/4}(m)\right)+O(\sqrt{\frac{\log(1/\delta)}{n_{z}}}),
	\end{split}
	\]
	where $\widetilde{O}\left(L^{5/4}\phi^{1/2}\log^{1/4}(m)\right)=O(L^{5/4}\phi^{1/2}\log^{1/4}(1/\phi)\log^{1/4}m)$.
	\end{lemma}

Proof of Lemma \ref{lma:boundefflabel} is given in Section \ref{sec:proofboundefflabel}.

\subsubsection{Proof of Theorem \ref{thm:generalizationbound}}\label{sec:proofgeneralization}
\begin{proof}
	By generalization bound with Rademacher complexity and Lemma \ref{lma:boundRad}, the population risk is bounded with probability at least $1-\delta, \delta\in(0,1)$ as
	\begin{equation}\label{eqn:generalizationproof2}
	\begin{split}
	&R(\bm{W}^{(T)})\\
	=&(1-\lambda)R(\bm{W}^{(T)})+\lambda R(\bm{W}^{(T)})\\
	\leq& \frac{1-\lambda}{n_z}\sum_{S_z}r(h_{\bm{W}^{(T)},i},y_i)+\frac{\lambda}{n_{g}}\sum_{S_{g}}r(h_{\bm{W}^{(T)},i},y_i)+O\left( \Phi+\sqrt{\log(1/\delta)}\right) \left((1-\lambda)\sqrt{\frac{1}{n_z}}+\lambda\sqrt{\frac{1}{n_g}} \right)
	\end{split}
	\end{equation}
	
	For the empirical risk, we have
	\begin{equation}\label{eqn:generalizationproof1}
	\begin{split}
	&\frac{1-\lambda}{n_z}\sum_{S_z}r(h_{\bm{W}^{(T)},i},y_i)+\frac{\lambda}{n_{g}}\sum_{S_{g}}r(h_{\bm{W}^{(T)},i},y_i)\\
	\leq & \frac{1-\lambda}{n_z}\sum_{S_z}\left(r(y_{\mathrm{eff},k(x_i)},y_i)+\left\| h_{\bm{W}^{(T)},i}-y_{\mathrm{eff},k(x_i)}\right\| \right)+\frac{\lambda}{n_{g}}\sum_{S_{g}}\left[r(y_{\mathrm{eff},k(x_i)},y_i)+\left\| h_{\bm{W}^{(T)},i}-y_{\mathrm{eff},k(x_i)}\right\| \right] \\
	\leq &\sqrt{\epsilon}+\frac{1-\lambda}{n_z}\sum_{S_z}r(y_{\mathrm{eff},k(x_i)},y_i)+\frac{\lambda}{n_{g}}\sum_{S_{g}}r(y_{\mathrm{eff},k(x_i)},y_i)
	\end{split}
	\end{equation}
	where the first inequality holds because of the 1-Lipschitz of risk functions such that  $r(h_{\bm{W}^{(t)},i},y_i)-r(y_{\mathrm{eff},k(x_i)},y_i)\leq \left\| h_{\bm{W}^{(t)},i}-y_{\mathrm{eff},k(x_i)}\right\| $, and the second inequality follows from the convergence of network output in Theorem \ref{the:convergence}.
	
	Since with $\beta_{\lambda}=\frac{\lambda n'_{g}}{(1-\lambda) n_g+\lambda n'_{g}}$, the training objective in Eqn.\eqref{eqn:informlossnn} can also be written as
$
	   \hat{R}_{\mathrm{I}}\left(\bm{W}\right)
	   =\left(1-\lambda+\frac{\lambda n_g'}{n_g}\right)\left( \frac{1-\beta_{\lambda}}{n_z}\sum_{S_z}   r\left(h_{\bm{W},i},\!z_i \right)+\frac{\beta_{\lambda}}{n_g'}\sum_{S_g'}r_{\mathrm{K}}\left( h_{\bm{W},i},g_i\right)\right)+\frac{\lambda n_g''}{n_g}\sum_{S_g''}r_{\mathrm{K}}\left( h_{\bm{W},i},g_i\right)
$, for any smooth set $k\in \mathcal{U}_{\phi}(S_z)$ (containing at least one labeled sample), $y_{\mathrm{eff},k}$ equivalently minimizes $\sum_{i\in I_{\phi,k}} \frac{1-\beta_{\lambda}}{n_z}\mathds{1}(x_i\in S_z)r\left(h,\!z_i \right)+\frac{\beta_{\lambda}}{n_g'}\mathds{1}(x_i\in S_g')r_{\mathrm{K}}\left( h,g_i\right)$ by Definition \ref{def:efflabel}. Thus, the bounds of the differences between optimal hypothesis and effective labels in Lemma \ref{lma:boundefflabel} hold for $\beta=\beta_{\lambda}$ and $h_{\mathrm{R},\beta_{\lambda}}^*$.
	Next, we can bound the total effective risk in terms of label and knowledge imperfectness, with probability at least $1-O(\phi)$,
	\begin{equation}\label{eqn:efflabelerror}
	\begin{split}
	&\frac{1-\lambda}{n_z}\sum_{S_z}r(y_{\mathrm{eff},k(x_i)},y_i)+\frac{\lambda}{n_{g}}\sum_{S_{g}}r(y_{\mathrm{eff},k(x_i)},y_i)\\
	=&\frac{1-\lambda}{n_z}\sum_{S_z}r(y_{\mathrm{eff},k(x_i)},y_i)+\frac{\lambda }{n_g}\sum_{S_{g}'}r(y_{\mathrm{eff},k(x_i)},y_i)+\frac{\lambda }{n_g}\sum_{S_{g}''}r(y_{\mathrm{eff},k(x_i)},y_i)\\
	\leq &\frac{1-\lambda}{n_z}\sum_{S_z}r(h^*_{\mathrm{R},\beta_{\lambda},i},y_i)+\frac{\lambda }{n_g}\sum_{S_{g}'}r(h^*_{\mathrm{R},\beta_{\lambda},i},y_i)+\frac{\lambda }{n_g}\sum_{S_{g}''}r(h_{\mathrm{K},i},y_i)\\
	&+\frac{1-\lambda}{n_z}\sum_{S_z}\left\|h^*_{\mathrm{R},\beta_{\lambda},i}-y_{\mathrm{eff},k(x_i)}\right\| +\frac{\lambda}{n_{g}}\sum_{S_{g}'}\left\| h^*_{\mathrm{R},\beta_{\lambda},i}-y_{\mathrm{eff},k(x_i)}\right\|+\frac{\lambda}{n_{g}}\sum_{S_{g}''}\left\| h^*_{\mathrm{K},i}-y_{\mathrm{eff},k(x_i)}\right\|\\
	\leq &(1-\lambda) \widehat{Q}_{\mathrm{R},S_z,S_{g}'}(\beta_{\lambda}) +\lambda \widehat{Q}_{\mathrm{K},S_{g}''}+O(\sqrt{\epsilon})
	\end{split}
	\end{equation}
	where $\beta_{\lambda}=\frac{\lambda n'_{g}}{(1-\lambda) n_g+\lambda n'_{g}}$, the first inequality comes from the Lipschitz continuity of risk functions, 
	and the last inequality holds by Lemma \ref{lma:boundefflabel} and the assumption $\phi\leq \widetilde{O}\left( \epsilon^2 L^{-9/2}\log^{-3}(m)\right) $ such that $\frac{1-\lambda}{n_z}\sum_{S_z}\left\|h^*_{\mathrm{R},\beta_{\lambda},i}-y_{\mathrm{eff},k(x_i)}\right\| +\frac{\lambda}{n_{g}}\sum_{S_{g}'}\left\| h^*_{\mathrm{R},\beta_{\lambda},i}-y_{\mathrm{eff},k(x_i)}\right\|\leq\left(1-\lambda+\frac{\lambda n_g'}{n_g}\right)\sqrt{\frac{1-\beta_{\lambda}}{n_z}\sum_{S_z}\|h^*_{\mathrm{R},\beta_{\lambda},i}-y_{\mathrm{eff},k(x_i)}\|^2+\frac{\beta_{\lambda}}{n_g'}\sum_{S_g'}\|h^*_{\mathrm{R},\beta_{\lambda},i}-y_{\mathrm{eff},k(x_i)}\|^2}\leq O(\sqrt{\epsilon})$ and $\frac{\lambda}{n_{g}}\sum_{S_{g}'}\left\| h^*_{\mathrm{R},\beta_{\lambda},i}-y_{\mathrm{eff},k(x_i)}\right\|\leq \frac{\lambda n_g''}{n_g} \sqrt{\frac{1}{n_g''}\sum_{S_g''}\| h^*_{\mathrm{K},i}-y_{\mathrm{eff},k(x_i)}\|^2}\leq O(\sqrt{\epsilon})$. In the last inequality of \eqref{eqn:efflabelerror}, we absorb $\frac{\lambda }{n_g}\sum_{S_{g}'}r(h^*_{\mathrm{R},\beta_{\lambda},i},y_i)$ into $O(\sqrt{\epsilon})$ because the risk functions are upper bounded and $S_g'$ is the set of samples sharing the same smooth sets with $S_z$, and so $\frac{\lambda }{n_g}\sum_{S_{g}'}r(h^*_{\mathrm{R},\beta_{\lambda},i},y_i)\leq O(\frac{n_g'}{n_g})\leq O(n_z\phi^b)\leq O(\sqrt{\epsilon})$.
	
	Substituting Eqn.\eqref{eqn:efflabelerror} and \eqref{eqn:generalizationproof1} into Eqn.~\eqref{eqn:generalizationproof2}, the population risk is bounded with probability at least $1-O(\phi)-\delta, \delta\in(0,1)$ as
	\[
	\begin{split}
	&R(\bm{W}^{(T)})=(1-\lambda)R(\bm{W}^{(T)})+\lambda R(\bm{W}^{(T)})\\
	\leq& \sqrt{\epsilon}+(1-\lambda) \widehat{Q}_{\mathrm{R},S_z,S_{g}'}(\beta_{\lambda}) +\lambda \widehat{Q}_{\mathrm{K},S_{g}''}\!+\!O\!\left( \Phi+\sqrt{\log(1/\delta)}\right)\!\!\! \left((1-\lambda)\sqrt{\frac{1}{n_z}}+\lambda\sqrt{\frac{1}{n_g}} \right).
	\end{split}
	\]
\end{proof}

\subsubsection{Proof of Theorem \ref{crl:generalizationbound}}\label{sec:proofgeneralization2}
\begin{proof}
Based on the construction of smooth sets in Definition \ref{def:smoothset}, denote $\mathcal{X}'=\mathcal{X}_{\phi}(S_z)=\bigcup_{k\in\mathcal{U}_{\phi}}(S_z)\mathcal{C}_{\phi,k}$ as the region covered by the smooth sets containing at least one sample in $S_z$, and let $\mathcal{X}''=\mathcal{X}/\mathcal{X}'$.  Let $\mathbb{P}_{\mathcal{X}'}=\int_{\mathcal{X}',\mathcal{Y}}p(y\mid x)\mathrm{d}yp(x)\mathrm{d}x$ and $\mathbb{P}_{\mathcal{X}''}=\int_{\mathcal{X}'',\mathcal{Y}}p(y\mid x)\mathrm{d}yp(x)\mathrm{d}x$ where $p(x)$ and $p(y\mid x)$ are probability densities. Then we have
$P_{\mathcal{X}'}\leq O(n_z/N)=O(n_z\phi^b)=O(\sqrt{\epsilon})$ by the assumption that $\phi\leq (\sqrt{\epsilon}/n_z)^{1/b}$, and $
\mathbb{E}\left[ r\left( h_{\bm{W}^{(T)}}(x),y\right) \right]=\mathbb{E}_{\mathbb{P}_{\mathcal{X}'}}\left[r\left( h_{\bm{W}^{(T)}}(x),y\right)  \right]\mathbb{P}_{\mathcal{X}'}+\mathbb{E}_{\mathbb{P}_{\mathcal{X}''}}\left[r\left( h_{\bm{W}^{(T)}}(x),y\right)  \right]\mathbb{P}_{\mathcal{X}''}.
$ By generalization bound with Rademacher complexity and Lemma \ref{lma:boundRad}, the population risk is bounded with probability at least $1-\delta, \delta\in(0,1)$ as
		\begin{equation}\label{eqn:t3.3proof3}
	\begin{split}
	&R(\bm{W}^{(T)})=(1-\lambda)\mathbb{E}\left[ r\left( h_{\bm{W}^{(T)}}(x),y\right) \right]+\lambda \mathbb{E}\left[ r\left( h_{\bm{W}^{(T)}}(x),y\right) \right]\\
	=&(1-\lambda)\mathbb{E}\left[ r\left( h_{\bm{W}^{(T)}}(x),y\right) \right] +\lambda \mathbb{E}_{\mathbb{P}_{\mathcal{X}''}}\left[ r\left( h_{\bm{W}^{(T)}}(x),y\right) \right]+\lambda \left( \mathbb{E}_{\mathbb{P}_{\mathcal{X}'}}\left[ r\left( h_{\bm{W}^{(T)}}(x),y\right) \right]-\mathbb{E}_{\mathbb{P}_{\mathcal{X}''}}\left[ r\left( h_{\bm{W}^{(T)}}(x),y\right) \right]\right) \mathbb{P}_{\mathcal{X}'}\\
	\leq& \frac{(1-\lambda)}{n_z}\sum_{S_z}r\left(h_{\bm{W}^{(T)},i},y_i\right)+\frac{\lambda}{n_{g}''}\sum_{S_{g}''}r\left(h_{\bm{W}^{(T)},i},y_i\right)+\lambda O\left( \sqrt{\epsilon}\right)+O\left( \Phi+\sqrt{\log(1/\delta)}\right) \left((1-\lambda)\sqrt{\frac{1}{n_z}}+\lambda\sqrt{\frac{1}{n_{g}''}} \right).
	\end{split}
	\end{equation}

Then for the empirical risk, we have
\begin{equation}\label{eqn:t3.3generalizationproof1}
\begin{split}
& \frac{1-\lambda}{n_z}\sum_{S_z}r\left(h_{\bm{W}^{(T)},i},y_i\right) +\frac{\lambda}{n_{g}''}\sum_{S_{g}''}r\left(h_{\bm{W}^{(T)},i},y_i\right)\\
\leq &  \frac{1-\lambda}{n_z}\sum_{S_z}r\left(y_{\mathrm{eff},k(x_i)}y_i\right) +\frac{\lambda}{n_{g}''}\sum_{S_{g}''}r\left(y_{\mathrm{eff},k(x_i)},y_i\right)+\frac{1-\lambda}{n_z}\sum_{S_z}\left\| h_{\bm{W}^{(T)},i}-y_{\mathrm{eff},i}\right\|   +\frac{\lambda}{n_{g}''}\sum_{S_{g}''}\left\| h_{\bm{W}^{(T)},i}-y_{\mathrm{eff},i}\right\|\\
\leq &O(\sqrt{\epsilon})+ \frac{1-\lambda}{n_z}\sum_{S_z}r\left(y_{\mathrm{eff},k(x_i)}y_i\right) +\frac{\lambda}{n_{g}''}\sum_{S_{g}''}r\left(y_{\mathrm{eff},k(x_i)},y_i\right),
\end{split}
\end{equation}
where the first inequality holds because of the Lipschitz continuity of risk functions such that  $r(h_{\bm{W}^{(T)},i},y_i)-r(y_{\mathrm{eff},k(x_i)},y_i)\leq \left\| h_{\bm{W}^{(T)},i}-y_{\mathrm{eff},k(x_i)}\right\| $, and the second inequality follows from the convergence of network output in Theorem \ref{the:convergence} (By Theorem \ref{the:convergence}, we have $\frac{(1-\lambda)(1-\beta)}{n_z}\sum_{S_z}\left\|h_{W^{(T)},i}-y_{\mathrm{eff},i}\right\|^2\leq O(\epsilon)$ and $\frac{\lambda}{n_g''}\sum_{S_g''}\left\|h_{W^{(T)},i}-y_{\mathrm{eff},i}\right\|^2\leq O(\epsilon)$, and so $\frac{1-\lambda}{n_z}\sum_{S_z}\left\|h_{W^{(T)},i}-y_{\mathrm{eff},i}\right\|\leq\sqrt{\frac{1-\lambda}{1-\beta}} O(\sqrt{\epsilon})= O(\sqrt{\epsilon})$ and $\frac{\lambda}{n_g''}\sum_{S_g''}\left\|h_{W^{(T)},i}-y_{\mathrm{eff},i}\right\|\leq \sqrt{\lambda}O(\sqrt{\epsilon})= O(\sqrt{\epsilon})$).

Next, we bound the empirical risk in terms of label and knowledge imperfectness as follows:
\begin{equation}\label{eqn:t3.3efflabelerror}
\begin{split}
&\frac{1-\lambda}{n_z}\sum_{S_z}r\left(y_{\mathrm{eff},k(x_i)}y_i\right) +\frac{\lambda}{n_{g}''}\sum_{S_{g}''}r\left(y_{\mathrm{eff},k(x_i)},y_i\right)\\
\leq &\frac{1-\lambda}{n_z}\sum_{S_z}r(h^*_{\mathrm{R},\beta,i},y_i) +\frac{\lambda }{n_{g}''}\sum_{S_{g}''}r(h_{\mathrm{K},i},y_i)+\frac{1-\lambda}{n_z}\sum_{S_z}\left\|h_{\mathrm{R},\beta,i}^*-y_{\mathrm{eff},k(x_i)}\right\|+\frac{\lambda}{n_{g}''}\sum_{S_{g}''}\left\| h_{\mathrm{K},i}^*-y_{\mathrm{eff},k(x_i)}\right\|\\
\leq & (1-\lambda) \widehat{Q}_{\mathrm{R},S_z,S_{g}'}(\beta) +\lambda \widehat{Q}_{\mathrm{K},S_{g}''}+O(\sqrt{\epsilon})
\end{split}
\end{equation}
where the first inequality comes from the Lipschitz continuity of risk functions, and the concavity of squared root, the last inequality holds by Definitions~\ref{def:knwoledgeimperfect} and \ref{def:regularizedimperfectness}, and Lemma~\ref{lma:boundefflabel} and the assumption of $\phi$ such that $\phi\log^{1/2}(1/\phi)\leq O\left( \epsilon^2 L^{-5/2}\log^{-1/2}(m)\right)$. Concretely, by Lemma~\ref{lma:boundefflabel} (a), we have $\frac{1-\beta}{n_z}\sum_{S_z}\left\|h^*_{\mathrm{R},\beta,i}-y_{\mathrm{eff},k(x_i)}\right\|^2\leq \widetilde{O}\left(L^{5/4}\phi^{1/2}\log^{1/4}(m)\right)$ and $	\frac{1}{n_{g}''}\sum_{S_{g}''}\left\| h^*_{\mathrm{K},i}-y_{\mathrm{eff},k(x_i)}\right\|^2\leq \widetilde{O}\left(L^{5/4}\phi^{1/2}\log^{1/4}(m)\right)$, and so it holds that  $\frac{1-\lambda}{n_z}\sum_{S_z}\left\|h_{\mathrm{R},\beta,i}^*-y_{\mathrm{eff},k(x_i)}\right\|\leq \frac{1-\lambda}{\sqrt{1-\beta}}\widetilde{O}\left(L^{5/8}\phi^{1/4}\log^{1/8}(m)\right)$ and $\frac{\lambda}{n_{g}''}\sum_{S_{g}''}\left\| h_{\mathrm{K},i}^*-y_{\mathrm{eff},k(x_i)}\right\|\leq \lambda \widetilde{O}\left(L^{5/8}\phi^{1/4}\log^{1/8}(m)\right)$. Thus we obtain the last inequality of \eqref{eqn:t3.3efflabelerror} by the assumption $\phi\log^{1/2}(1/\phi)\leq O\left( \epsilon^2 L^{-5/2}\log^{-1/2}(m)\right)$.

Substituting Eqns.~\eqref{eqn:t3.3efflabelerror} and  \eqref{eqn:t3.3generalizationproof1} into Eqn.~\eqref{eqn:t3.3proof3}, we have
\[
\begin{split}
&R(\bm{W}^{(T)})=(1-\lambda)\mathbb{E}\left[ r\left( h_{\bm{W}^{(T)}}(x),y\right) \right]+\lambda \mathbb{E}\left[ r\left( h_{\bm{W}^{(T)}}(x),y\right) \right]\\
\leq& O(\sqrt{\epsilon})+(1-\lambda) \widehat{Q}_{\mathrm{R},S_z,S_{g}'}(\beta) +\lambda \widehat{Q}_{\mathrm{K},S_{g}''}+O\left( \Phi+\sqrt{\log(1/\delta)}\right) \left((1-\lambda)\sqrt{\frac{1}{n_z}}+\lambda\sqrt{\frac{1}{n_{g}''}} \right).
\end{split}
\]
\end{proof}

\subsubsection{Proof of Corollary \ref{crl:generalizationbound2}}\label{sec:proofgeneralization3}
\begin{proof}
	First, following \eqref{eqn:t3.3proof3},
	with probability at least $1-O(\phi)-\delta, \delta\in(0,1)$, it holds that
	\begin{equation}\label{eqn:cor4.3proof1}
	\begin{split}
	&R(\bm{W}^{(T)})
\leq \frac{(1-\lambda)}{n_z}\sum_{S_z}r\left(h_{\bm{W}^{(T)},i},y_i\right)+\frac{\lambda}{n_{g}''}\sum_{S_{g}''}r\left(h_{\bm{W}^{(T)},i},y_i\right)+\lambda O\left( \sqrt{\epsilon}\right)\\
&+O\left( \Phi+\sqrt{\log(1/\delta)}\right) \left((1-\lambda)\sqrt{\frac{1}{n_z}}+\lambda\sqrt{\frac{1}{n_{g}''}} \right)
	\end{split}
	\end{equation}
	
With the same reason as in Eqn.~\eqref{eqn:t3.3generalizationproof1}, we have
	\begin{equation}\label{eqn:cor4.3proof3}
	\begin{split}
	&\frac{1-\lambda}{n_z}\sum_{S_z}r\left(h_{\bm{W}^{(T)},i},y_i\right) +\frac{\lambda}{n_{g}''}\sum_{S_{g}''}r\left(h_{\bm{W}^{(T)},i},y_i\right)\\
\leq &O(\sqrt{\epsilon})+ \frac{1-\lambda}{n_z}\sum_{S_z}r\left(y_{\mathrm{eff},k(x_i)}, y_i\right) +\frac{\lambda}{n_{g}''}\sum_{S_{g}''}r\left(y_{\mathrm{eff},k(x_i)},y_i\right).
	\end{split}
	\end{equation}
	
Then, unlike in the proof of Theorem~\ref{crl:generalizationbound},  we need to bound the risk in \eqref{eqn:cor4.3proof3} in terms of expected label and knowledge imperfectness. Thus, replacing $h^*_{\mathrm{R},\beta}$ and $h^*_{\mathrm{K}}$ in Eqn.~\eqref{eqn:t3.3efflabelerror} with $\bar{h}^*_{\mathrm{R},\beta}$ and $\bar{h}^*_{\mathrm{K}}$, we have
\begin{equation}\label{eqn:cor4.3proof2}
\begin{split}
& \frac{1-\lambda}{n_z}\sum_{S_z}r\left(y_{\mathrm{eff},k(x_i)}y_i\right) +\frac{\lambda}{n_{g}''}\sum_{S_{g}''}r\left(y_{\mathrm{eff},k(x_i)},y_i\right)\\
\leq & (1-\lambda) \widehat{Q}_{\mathrm{R},S_z,S_{g}'}(\beta) +\lambda \widehat{Q}_{\mathrm{K},S_{g}''}+ 
\frac{1-\lambda}{n_z}\sum_{S_z}\left\|\bar{h}^*_{\mathrm{R},\beta,i}-y_{\mathrm{eff},k(x_i)}\right\| +\frac{\lambda}{n_{g}''}\sum_{S_{g}''}\left\| \bar{h}^*_{\mathrm{K},i}-y_{\mathrm{eff},k(x_i)}\right\|
\end{split}
\end{equation}
By Lemma~\ref{lma:boundefflabel} (b), it holds that $	\frac{1}{n_{g}''}\sum_{S_{g}''}\left\| \bar{h}^*_{\mathrm{K},i}-y_{\mathrm{eff},k(x_i)}\right\|^2\leq \widetilde{O}\left(L^{5/4}\phi^{1/2}\log^{1/4}(m)\right)+O(\sqrt{\frac{\log(1/\delta)}{n_{g}''}})$ and $\frac{1-\beta}{n_z}\sum_{S_z}\left\|\bar{h}^*_{\mathrm{R},\beta,i}-y_{\mathrm{eff},k(x_i)}\right\|^2 \leq  \widetilde{O}\left(L^{5/4}\phi^{1/2}\log^{1/4}(m)\right)+O(\sqrt{\frac{\log(1/\delta)}{n_{z}}})$. Thus we have $	\frac{\lambda}{n_{g}''}\sum_{S_{g}''}\left\| \bar{h}^*_{\mathrm{K},i}-y_{\mathrm{eff},k(x_i)}\right\|\leq \lambda\left(\widetilde{O}\left(L^{5/8}\phi^{1/4}\log^{1/8}(m)\right)+O\left((\frac{\log(1/\delta)}{n_{g}''})^{\frac{1}{4}}\right)\right)\leq \lambda\left(O(\epsilon)+O\left((\frac{\log(1/\delta)}{n_{g}''})^{\frac{1}{4}}\right)\right)$ and $\frac{1-\lambda}{n_z}\sum_{S_z}\left\|\bar{h}^*_{\mathrm{R},\beta,i}-y_{\mathrm{eff},k(x_i)}\right\| \leq  \frac{1-\lambda}{\sqrt{1-\beta}}\left(\widetilde{O}\left(L^{5/8}\phi^{1/4}\log^{1/8}(m)\right)+O\left((\frac{\log(1/\delta)}{n_z})^{\frac{1}{4}}\right) \right)\leq (1-\lambda)\left(O(\epsilon)+O\left((\frac{\log(1/\delta)}{n_z})^{\frac{1}{4}}\right)\right)$.
Therefore, continuing with \eqref{eqn:cor4.3proof2}, it holds that 
\begin{equation}\label{eqn:cor4.3proof11}
\begin{split}
& \frac{1-\lambda}{n_z}\sum_{S_z}r\left(y_{\mathrm{eff},k(x_i)}y_i\right) +\frac{\lambda}{n_{g}''}\sum_{S_{g}''}r\left(y_{\mathrm{eff},k(x_i)},y_i\right)\\
\leq & (1-\lambda) \widehat{Q}_{\mathrm{R},S_z,S_{g}'}(\beta) +\lambda \widehat{Q}_{\mathrm{K},S_{g}''}+O(\sqrt{\epsilon})+O\left( (1-\lambda)\left(\frac{\log(1/\delta)}{n_{z}}\right)^{\frac{1}{4}}+\lambda \left(\frac{\log(1/\delta)}{n_{g}''}\right)^{\frac{1}{4}} \right)\\
		\leq &O(\sqrt{\epsilon})+(1-\lambda) Q_{\mathrm{R}}(\beta) +\lambda Q_{\mathrm{K}}+O\left( (1-\lambda)\left(\frac{\log(1/\delta)}{n_{z}}\right)^{\frac{1}{4}}+\lambda \left(\frac{\log(1/\delta)}{n_{g}''}\right)^{\frac{1}{4}} \right) ,
\end{split}
\end{equation}
	where the second inequality holds by Lemma~\ref{lma:boundefflabel} and the last inequality holds by McDiarmid's inequality.
Finally, substituting Eqns.~\eqref{eqn:cor4.3proof11} and \eqref{eqn:cor4.3proof3} into Eqn.~\eqref{eqn:cor4.3proof1},
	with probability at least $1-O(\phi)-\delta, \delta\in(0,1)$, it holds that
	\[
	\begin{split}
	R(\bm{W}^{(T)})
	\leq O(\sqrt{\epsilon})+(1-\lambda) Q_{\mathrm{R}}(\beta) +\lambda Q_{\mathrm{K}}+O\left( \Phi+\log^{1/4}(1/\delta)\right)\sqrt{\frac{1-\lambda}{\sqrt{n_z}}+\frac{\lambda}{\sqrt{n''_{g}}}}.
	\end{split}
	\]
	This completes the proof.
\end{proof} 

\subsubsection{Proof of Corollary \ref{thm:generalization_choiceweight}}\label{sec:proofsamplingcomplexity}
\begin{proof}
    \textbf{Proof of (a).} If $Q_{\mathrm{K}}\leq \sqrt{\epsilon}$ and $\lambda$ is set as 1, it holds by Corollary \ref{crl:generalizationbound2} that 
    	\[
	\begin{split}
	R(\bm{W}^{(T)})
	&\leq O(\sqrt{\epsilon})+ Q_{\mathrm{K}}+O\left( \Phi+\log^{1/4}(1/\delta)\right)\left(\frac{1}{n''_{g}}\right)^{1/4}\\
	&\leq O(\sqrt{\epsilon})+O\left(\frac{1}{n''_{g}}\right)^{1/4},
	\end{split}
	\]
	where in the last inequality we absorb the scales of the last term by $O$ notation.  Thus, $n_g''\leq O(\frac{1}{\epsilon^2})$ guarantees that $R(\bm{W}^{(T)})\leq \sqrt{\epsilon}$. In the proof of Theorem \ref{crl:generalizationbound}, we prove that the probability that a sample belongs to the region covered by the smooth sets containing at least one labeled sample is $P_{\mathcal{X}'}= O(n_z/N)=O(n_z\phi^b)=O(\sqrt{\epsilon})$. Thus we have $P_{\mathcal{X}''}=1-P_{\mathcal{X}'}= 1-O(\sqrt{\epsilon})$, and so $n_g= \frac{n''_{g}}{P_{\mathcal{X}''}}= n''_{g}/(1-O(\epsilon))\sim O(1/(\epsilon^2-\epsilon^3))$.
	
	\textbf{Proof of (b).} If $Q_\mathrm{K}> \sqrt{\epsilon}$ and $\lambda=  \frac{\sqrt{\epsilon}}{Q_{\mathrm{K}}}$, then by Corollary \ref{crl:generalizationbound2}, we have
	\[
	\begin{split}
	R(\bm{W}^{(T)})
	&\leq O(\sqrt{\epsilon})+Q_{\mathrm{R}}(\beta^*)-\frac{Q_{\mathrm{R}}(\beta^*)}{Q_{\mathrm{K}}}\sqrt{\epsilon} +\sqrt{\epsilon}+O\left(\sqrt{(1-\frac{\sqrt{\epsilon}}{Q_{\mathrm{K}}})\frac{1}{\sqrt{n_z}}+\frac{\sqrt{\epsilon}}{Q_{\mathrm{K}}}\frac{1}{\sqrt{n''_{g}}}}\right)\\
	&\leq O(\sqrt{\epsilon})+O\left(\sqrt{(1-\frac{\sqrt{\epsilon}}{Q_{\mathrm{K}}})\frac{1}{\sqrt{n_z}}+\frac{\sqrt{\epsilon}}{Q_{\mathrm{K}}}\frac{1}{\sqrt{n''_{g}}}}\right),
	\end{split}
	\]
	where the second inequality holds because $\frac{\sqrt{\epsilon}}{Q_{\mathrm{K}}}+\frac{\sqrt{\epsilon}}{Q_{\mathrm{R}}(\beta^*)}\geq 1$ such that $Q_{\mathrm{R}}(\beta^*)-\frac{Q_{\mathrm{R}}(\beta^*)}{Q_{\mathrm{K}}}\sqrt{\epsilon}\leq \sqrt{\epsilon}$. Then to guarantee $R(\bm{W}^{(T)})\leq \sqrt{\epsilon}$, we require that $(1-\frac{\sqrt{\epsilon}}{Q_{\mathrm{K}}})\frac{1}{\sqrt{n_z}}\leq \epsilon$ and $\frac{\sqrt{\epsilon}}{Q_{\mathrm{K}}}\frac{1}{\sqrt{n''_{g}}}\leq \epsilon$. Thus, we have $n_z\sim O\left(\left( 1/\epsilon-1/\left( \sqrt{\epsilon}Q_{\mathrm{K}}\right)  \right)^2   \right) $, $n_g''\sim O(\frac{1}{\epsilon Q^2_{\mathrm{K}}})$ and  $n_g= n''_{g}/(1-O(\epsilon))\sim O(1/\left( (\epsilon -\epsilon^2)Q^2_{\mathrm{K}}\right) )$.
	
	\textbf{Proof of (c).}	We prove (c) by contradiction. If $R(\bm{W}^{(T)})\leq \sqrt{\epsilon}$, we have $(1-\lambda)Q_{\mathrm{R}}(\beta^*)\leq \sqrt{\epsilon}$ and $\lambda Q_{\mathrm{K}}\leq \sqrt{\epsilon}$. Then $\frac{\sqrt{\epsilon}}{Q_{\mathrm{R}}(\beta^*)}+\frac{\sqrt{\epsilon}}{Q_{\mathrm{K}}}\geq 1-\lambda+\lambda=1$. This is contradictory to the condition $\frac{\sqrt{\epsilon}}{Q_{\mathrm{R}}(\beta^*)}+\frac{\sqrt{\epsilon}}{Q_{\mathrm{K}}}\leq 1$. Thus completes the proof. 

\end{proof}

%% file: Proof2nd.tex
\section{Proofs of Lemmas in Appendix~\ref{sec:proofmain}}\label{sec:proof2nd}

We now show the proofs of lemmas in Appendix~\ref{sec:proofmain},
while the proofs of lemmas newly introduced in this section are deferred
to Appendix~\ref{sec:proof3rd}.

\subsection{Proof of Lemma \ref{lma:forwardinputperturbation}}\label{sec:proofforwardinputperturbation}
In this section, we prove the forward perturbation with respect to inputs. We first recall some important notations. For the smooth set $k\in[N]$, layer $l\in[L]$, let $h_{l,k}=h_l(x_k')$, $h_{l,i}=h_l(x_i)$ be the activated output of $l$th layer, and $f_{l,k}=\bm{W}_lh_{l-1}(x_k')$,  $f_{l,i}=\bm{W}_lh_{l-1}(x_i)$ be the pre-activated output of $l$th layer for some weight $\bm{W}\in\mathcal{B}\left(\bm{W}^{(0)}, \tau \right) $. At initialization, denote $h^{(0)}_{l,k}=h^{(0)}_l(x_k')$, $h^{(0)}_{l,i}=h^{(0)}_l(x_i)$, $f^{(0)}_{l,k}=\bm{W}^{(0)}_lh^{(0)}_{l-1}(x_k')$,  $f^{(0)}_{l,i}=\bm{W}^{(0)}_lh^{(0)}_{l-1}(x_i)$, the diagonal matrices $\bm{D}^{(0)}_{l,k}\in\mathbb{R}^{m\times m}$ and $\bm{D}^{(0)}_{l,i}\in\mathbb{R}^{m\times m}$ with  $\left[\bm{D}^{(0)}_{l,k}\right]_{j,j}=\mathds{1}([f^{(0)}_{l,k}]_j\geq 0)$ and $\left[\bm{D}^{(0)}_{l,i}\right]_{j,j}=\mathds{1}([f^{(0)}_{l,i}]_j\geq 0)$ for $i\in\mathcal{I}_{\phi,k}, j\in[m]$. Then we denote for initialization $f'_{l,i}=f^{(0)}_{l,i}-f^{(0)}_{l,k}$ and
the diagonal matrix $\bm{D}_l'\in\mathbb{R}^{m\times m}$ with $\left[\bm{D}_l'\right]_{j,j}=\left[\bm{D}^{(0)}_{l,k}\right]_{j,j}-\left[\bm{D}^{(0)}_{l,i}\right]_{j,j}$, omitting the notation $(0)$ and $i,k$.
\begin{lemma}\label{lma:zeronormperturbation}
If $f'_{l,i}$ can be written as $f'_{l,i}=f'_{l,i,1}+f'_{l,i,2}$ with $\|f'_{l,i,1}\|\leq O(L^{3/2}\phi\log^{1/2}(1/\phi))$ and $\|f'_{l,i,2}\|_{\infty}\leq O(L\phi^{2/3}\log^{1/2}(1/\phi)m^{-1/2})$,  then with probability at least $1-\exp(-\Omega(m\phi^{2/3}L))$ over the randomness of $\bm{W}^{(0)}$, we have
\[
\|\bm{D}_l'f^{(0)}_{l,i} \|_0\leq \|\bm{D}_l'\|_0\leq O(m\phi^{2/3}L\log^{1/2}(1/\phi)),
\]
\[
\|\bm{D}_l'f^{(0)}_{l,i} \|\leq O(\phi L^{3/2}\log^{1/2}(1/\phi)).
\]
\end{lemma}
Proof Lemma \ref{lma:zeronormperturbation} is given in Section \ref{sec:proofzeronormperturbation}.

\textbf{Proof of Lemma \ref{lma:forwardinputperturbation}}
\begin{proof}
	We first prove the following three conclusions by induction under the assumptions in Lemma \ref{lma:forwardinputperturbation}: for $i\in\mathcal{I}_{\phi,k},k\in[N]$, with probability at least $1-O(\phi)$,\\
	(a)$f'_{l,i}$ at initialization can be written as $f'_{l,i,1}+f'_{l,i,2}$ with $\|f'_{l,i,1}\|\leq O(L^{3/2}\phi\log^{1/2}(1/\phi))$ and $\|f'_{l,i,2}\|_{\infty}\leq O(L\phi^{2/3}m^{-1/2}\log^{1/2}(1/\phi))$.\\
	(b) At initialization, $\|\bm{D}_l'f^{(0)}_{l,i} \|_0\leq \|\bm{D}_l'\|_0\leq O(m\phi^{2/3}L\log^{1/2}(1/\phi))$, $\|\bm{D}_l'f^{(0)}_{l,i} \|\leq O(\phi L^{3/2}\log^{1/2}(1/\phi))$.\\
	(c)$\|h^{(0)}_{l,i}-h^{(0)}_{l,k}\|\leq O(L^{5/2}\phi\sqrt{\log(m)\log(1/\phi)})$ and $\|f^{(0)}_{l,i}-f^{(0)}_{l,k}\|\leq O(L^{5/2}\phi\sqrt{\log(m)\log(1/\phi)})$.
	
	When $l=0$, we have $\|h^{(0)}_{0,i}-h^{(0)}_{0,k}\|=\|x_i-x_k\|\leq O(\phi)$. Since $\left[ \bm{W}_1^{(0)}\right]_{j,j} \sim\mathcal{N}\left(0,\frac{2}{m} \right),j\in[m] $, we have $\|f^{(0)}_{1,i}-f^{(0)}_{1,k}\|\leq O(\phi\log^{1/2}(1/\phi))$ with probability at least $1-O(\phi)$ over the randomness of $\bm{W}_1^{(0)}$. By Lemma \ref{lma:zeronormperturbation}, the above three conclusions hold.
	Then we assume the conclusions (a) holds for  layer $a, a\leq l-1$ and prove (a)(b)(c) hold for $l$.
	
	First, we re-write $f'_{l,i}$ as
	\[
	\begin{split}
	 &f'_{l,i}=f^{(0)}_{l,i}-f^{(0)}_{l,k}=\bm{W}^{(0)}_l\left(\bm{D}^{(0)}_{l-1,k}+\bm{D}_{l-1}' \right)\left(f^{(0)}_{l-1,k}+f'_{l-1,i} \right) -\bm{W}^{(0)}_l\bm{D}^{(0)}_{l-1,k}f^{(0)}_{l-1,k}\\
	 =&\bm{W}^{(0)}_l\bm{D}_{l-1}'\left( f^{(0)}_{l-1,k}+f'_{l-1,i} \right) +\bm{W}^{(0)}_l\bm{D}^{(0)}_{l-1,k}f'_{l-1,i}\\
	 =&\cdots\\
	 =&\sum_{a=2}^l\left( \prod_{b=a+1}^{l}\bm{W}_b^{(0)}\bm{D}_{b-1,k}^{(0)}\right) \bm{W}^{(0)}_a\bm{D}_{a-1}'\left( f^{(0)}_{a-1,k}+f'_{a-1,i} \right)+\bm{W}_1^{(0)}(x_i-x_k')
	 \end{split}
	\]
	
	By Lemma \ref{lma:zeronormperturbation}, and the inductive assumption (a) for layer $a, a\leq l-1$, we have with probability at least $1-\exp(-\Omega\left(m\phi^{2/3}L \right) )$,
	\begin{equation}\label{eqn:fwdpurterbproof1}
	\|\bm{D}_{a}'\left( f^{(0)}_{a,k}+f'_{a,i} \right)\|_0\leq O(m\phi^{2/3}L\log^{1/2}(1/\phi)),
	\end{equation}
	\begin{equation}\label{eqn:fwdpurterbproof2}
		\|\bm{D}_{a}'\left( f^{(0)}_{a,k}+f'_{a,i} \right)\|\leq O(\phi L^{3/2}\log^{1/2}(1/\phi)),
	\end{equation}
	so (b) holds for layer $l$.
	Then let $q_a=\left( \prod_{b=a+1}^{l}\bm{W}_b^{(0)}\bm{D}_{b-1,k}^{(0)}\right) \bm{W}^{(0)}_a\bm{D}_{a-1}'\left( f^{(0)}_{a-1,k}+f'_{a-1,i} \right)$. By Eqn.\eqref{eqn:fwdpurterbproof1}, \eqref{eqn:fwdpurterbproof2}, and Claim 8.5 ($s=O(m\phi^{2/3}L)$) in \cite{Convergence_zhu_allen2019convergence}, with probability at least $1-\exp(-\Omega\left(m\phi^{2/3}L \log(m)\right) )$, we can write
	$q_a=q_{a,1}+q_{a,2}$ with
	\begin{equation}\label{eqn:fwdpurterbproof3}
	\|q_{a,1}\|\leq O(\phi^{4/3}L^2\log(m)\log^{3/4}(1/\phi)) \quad\mathrm{and}\quad \|q_{a,2}\|_{\infty}\leq O(\phi L^{3/2}\sqrt{\log(m)/m}\log^{1/2}(1/\phi)).
	\end{equation}
	Let $f'_{l,i,1}=\sum_{a=2}^lq_{a,1}+\bm{W}_1^{(0)}(x_i-x_k')$ and 	$f'_{l,i,2}=\sum_{a=2}^lq_{a,2}$. Then we have $f'_{l,i}= f'_{l,i,1}+f'_{l,i,2}$.
	Since $\|\bm{W}_1^{(0)}(x_i-x_k')\|\leq O(\phi\sqrt{\log(1/\phi)})$ with probability at least $1-\phi,\phi\in(0,1)$, by triangle inequality, we can write
	 \[
	 \begin{split}
&\|f^{(0)}_{l,i}-f^{(0)}_{l,k}\|=\|f'_{l,i}\|= \|f'_{l,i,1}+f'_{l,i,2}\| \\
	\leq &O(\phi^{4/3}L^3\log(m)\log^{3/4}(1/\phi)+\phi\sqrt{\log(1/\phi)})+O(\phi L^{5/2}\sqrt{\log(m)}\log^{1/2}(1/\phi))\\
\leq &  O(L^{5/2}\phi\sqrt{\log(m)\log(1/\phi)}),
	\end{split}
	\]
	where the first inequality comes from inequalities \ref{eqn:fwdpurterbproof3}, and the last inequality holds by the assumption $\phi\leq O(L^{-9/2}\log^{-3}(m)\log^{-3/4}(1/\phi))$. Also, by the requirement of $\phi$, we have
with $\|f'_{l,i,1}\|\leq O(\phi^{4/3}L^3\log(m)\log^{3/4}(1/\phi)+\phi\sqrt{\log(1/\phi)})\leq O(L^{3/2}\phi\log^{1/2}(1/\phi))$ and $\|f'_{l,i,2}\|_{\infty}\leq O(\phi L^{5/2}\sqrt{\log(m)/m}\log^{1/2}(1/\phi))\leq  O(L\phi^{2/3}m^{-1/2}\log^{1/2}(1/\phi))$. Thus (a) holds for layer $l$.
And by Lemma \ref{lma:zeronormperturbation}, we have with probability at least $1-\phi$,
\[
\begin{split}
\|h^{(0)}_{l,i}-h^{(0)}_{l,k}\|&=\|(\bm{D}^{(0)}_{l,k}+\bm{D}_l')\left( f^{(0)}_{l,k}+f'_{l,i}\right) -\bm{D}^{(0)}_{l,k}f^{(0)}_{l,k}\|\\
&\leq \|\bm{D}_l'f^{(0)}_{l,k}\|+\|(\bm{D}^{(0)}_{l,k}+\bm{D}_l')f'_{l,i}\|\leq O(L^{5/2}\phi\sqrt{\log(m)\log(1/\phi)}),
\end{split}
\]
where the second inequality comes from Lemma \ref{lma:zeronormperturbation}. Thus, conclusion (c) holds for layer $l$.

Finally, by Lemma 8.2 in \cite{Convergence_zhu_allen2019convergence} which gives forward perturbation regarding weights, we have with probability at least $1-O(\phi)$,
\[
\begin{split}
\|f_{l,i}-f_{l,k}\|&\leq \|f^{(0)}_{l,i}-f^{(0)}_{l,k}\|+\|f_{l,i}-f^{(0)}_{l,i}\|+\|f_{l,k}-f^{(0)}_{l,k}\|\\
&\leq O(L^{3/2}\phi\log^{1/2}(1/\phi))+O(\tau L^{5/2}\sqrt{\log(m)})\leq O(L^{5/2}\phi\sqrt{\log(m)\log(1/\phi)}),
\end{split}
\]
where the last probability holds by the assumption $\tau\leq O(\phi^{3/2})$.
Similarly, we have with probability at least $1-O(\phi)$,
\[
\begin{split}
\|h_{l,i}-h_{l,k}\|&\leq \|h^{(0)}_{l,i}-h^{(0)}_{l,k}\|+\|h_{l,i}-h^{(0)}_{l,i}\|+\|h_{l,k}-h^{(0)}_{l,k}\|\\
&\leq O(L^{3/2}\phi\log^{1/2}(1/\phi))+O(\tau L^{5/2}\sqrt{\log(m)})\leq O(L^{5/2}\phi\sqrt{\log(m)\log(1/\phi)}).
\end{split}
\]
	
\end{proof}
\subsection{Proofs of Lemma \ref{lma:knowledgeriskfuncgradient} and Lemma \ref{lma:cellriskgradientbound}}\label{sec:proofcellriskgradientbound}
\textbf{Proof of Lemma \ref{lma:knowledgeriskfuncgradient}}
\begin{proof}
	By the mean value theorem, $r$ can be represented as
	\begin{equation}
	\begin{split}
	r\left(h' \right)= r\left(h \right)+\bigtriangledown r\left(h \right)^\top (h'-h)+\frac{1}{2}(h'-h)^\top\bigtriangledown^2 r\left(z \right)(h'-h),
	\end{split}
	\end{equation}
	where $z$ lies in the line segment between $h'$ and $h$.
	
	Since the maximum eigenvalue of the Hessian matrix of $r$ is bounded by 1,  for any output of the neural network $h$ and $h'$ , we have
	\begin{equation}
	\begin{split}
	r\left(h' \right)& \leq r\left(h \right)+\bigtriangledown r\left(h \right)^\top (h'-h)+\frac{1}{2}\|h'-h\|_2^2
	\end{split}
	\end{equation}	
	Let $h'=h-\bigtriangledown r(h)$. We have
	\begin{equation}
	r_{\min}\leq  r\left(h' \right) \leq r\left(h \right)-\frac{1}{2}\|\bigtriangledown r(h)\|_2^2.
	\end{equation}
	Thus, we get the first inequality of the lemma $\left\| \bigtriangledown r(h)\right\|^2\leq  2\left( r\left(h \right)-r_{\min}\right)$.
	
	By strong convexity, for any $h$ and $h'$ in the domain of risk function $r$, we have
	\begin{equation}\label{eqn:convexproof1}
	\begin{split}
	r\left(h' \right)& \geq r\left(h \right)+\bigtriangledown r\left(h \right)^\top (h'-h)+\frac{\rho}{2}\|h'-h\|^2\\
	& \geq r\left(h \right)-\frac{1}{2\rho}\|\bigtriangledown r(h)\|^2,
	\end{split}
	\end{equation}
	where the first inequality comes from strong convexity and the second inequality holds by choosing $h'=-\frac{\bigtriangledown r(h)}{\rho}$ that minimizes the right hand side.
	Then letting $h'$ in the left hand side equals to $h^*$ such that $ r\left(h^* \right)=r_{\min}$, we get the second inequality of the lemma $\left\| \bigtriangledown r(h)\right\|^2\geq 2\rho\left( r\left(h \right)-r_{\min}\right) $.
	
	Also, letting $h'$ in Eqn.~\eqref{eqn:convexproof1} be $h^*$, we have
	\begin{equation}
	r_{\min}=r\left(h^* \right) \geq r\left(h \right)-\left\| \bigtriangledown r\left(h \right)\right\|  \left\| (h^*-h)\right\| +\frac{\rho}{2}\|h^*-h\|^2
	\end{equation}
	By the fact that $r_{\min}\leq r\left(h \right)$, we have
	\begin{equation}
	\|h^*-h\|\leq \frac{2}{\rho}\left\| \bigtriangledown r\left(h \right)\right\|.
	\end{equation} We thus get the third inequality.
\end{proof}

\textbf{Proof of Lemma \ref{lma:cellriskgradientbound}}
\begin{proof}
	Denote the risk of the $k$th cell $k\in[N]$  with respect to the input $x_k'\in\mathcal{X}_{\phi}$ for hypothesis $h\in\mathcal{H}$ as
	\[
	\bar{r}^{\circ}_{\mathrm{I},k}=\sum_{i\in\mathcal{I}_{\phi,k}} r_{\mathrm{I},i}(h(x_k')).
	\]
	Recall that $M_k=\sum_{i\in\mathcal{I}_{\phi,k}}(\mu_i+\lambda_i)$. By 1-Lipschitz continuity of risk functions and their gradients, we have with probability at least $1-O(\phi)$,
	\[
	\left| \bar{r}_{\mathrm{I},k}-\bar{r}^{\circ}_{\mathrm{I},k}\right| \leq \sum_{i\in\mathcal{I}_{\phi,k}}(\mu_i+\lambda_i)\left\| h(x_i)-h(x_k')\right\| \leq \widetilde{O}(M_kL^{5/2}\phi\log^{1/2}(m)),
	\]
	\[\text{and }
	\left\| \bigtriangledown_{h}\bar{r}_{\mathrm{I},k}-\bigtriangledown_{h}\bar{r}^{\circ}_{\mathrm{I},k}\right\|\leq \sum_{i\in\mathcal{I}_{\phi,k}}(\mu_i+\lambda_i)\left\| h(x_i)-h(x_k')\right\| \leq \widetilde{O}(M_kL^{5/2}\phi\log^{1/2}(m)).
	\]
	Since the eigenvalues of $\bigtriangledown^2_{h}\bar{r}^{\circ}_{\mathrm{I},k}$ is no less than $M_k \rho$ and $r^*_{\mathrm{eff},k} $ is the minimum value of $\sum_{\mathcal{I}_{\phi,k}} r_{\mathrm{I},i}(h)$, by Lemma \ref{lma:knowledgeriskfuncgradient}, we have
	\[
	\|\bigtriangledown_{h}\bar{r}^{\circ}_{\mathrm{I},k}\|^2\geq 2M_k\rho\left(\bar{r}^{\circ}_{\mathrm{I},k}-r^*_{\mathrm{eff},k} \right)\geq 2M_k\rho\left(\bar{r}_{\mathrm{I},k}-r^*_{\mathrm{eff},k}\right)-\widetilde{O}(M^2_k\rho L^{5/2}\phi\log^{1/2}(m)) .
	\]
	Therefore, we have
	\[
	\begin{split}
	\|\sum_{i\in\mathcal{I}_{\phi,k}}u_i(h(x_i))\|^2&=\|\bigtriangledown_{h}\bar{r}_{\mathrm{I},k}\|^2\geq\|\bigtriangledown_{h}\bar{r}^{\circ}_{\mathrm{I},k}\|^2-\widetilde{O}(M^2_k\ L^{5/2}\phi\log^{1/2}(m))\\
	&\geq 2M_k\rho\left(\bar{r}_{\mathrm{I},k}-r^*_{\mathrm{eff},k}\right)-\widetilde{O}(M^2_k L^{5/2}\phi\log^{1/2}(m))
	\end{split}
	\]
	
	Also, since the eigenvalues of $\bigtriangledown^2_{h}\bar{r}^{\circ}_{\mathrm{I},k}$ is no larger than $M_k$, we have
	\[
	\|\bigtriangledown_{h}\bar{r}^{\circ}_{\mathrm{I},k}\|^2\leq 2M_k\left(\bar{r}^{\circ}_{\mathrm{I},k}-r_{\mathrm{eff},k} \right)\leq 2M_k\left(\bar{r}_{\mathrm{I},k}-r_{\mathrm{eff},k}\right)+\widetilde{O}(M^2_k L^{5/2}\phi\log^{1/2}(m)).
	\]
	Therefore, it holds that
	\[
	\begin{split}
	\|\sum_{i\in\mathcal{I}_{\phi,k}}u_i(h(x_i))\|^2&=\|\bigtriangledown_{h}\bar{r}_{\mathrm{I},k}\|^2\leq\|\bigtriangledown_{h}\bar{r}^{\circ}_{\mathrm{I},k}\|^2+\widetilde{O}(M^2_k\rho L^{5/2}\phi\log^{1/2}(m))\\
	&\leq 2M_k\left(\bar{r}_{\mathrm{I},k}-r_{\mathrm{eff},k}\right)+\widetilde{O}(M^2_k L^{5/2}\phi\log^{1/2}(m)).
	\end{split}
	\]
	
	Applying Lemma \ref{lma:knowledgeriskfuncgradient} for $\bar{r}^{\circ}_{\mathrm{I},k}$, we have
	\[
	\left\| h\left(x_k' \right)-y_{\mathrm{eff},k} \right\|^2\leq \frac{4}{M_k^2\rho^2} \left\| \bigtriangledown_{h}\bar{r}^{\circ}_{\mathrm{I},k}\right\|^2\leq  \frac{1}{M_k\rho^2}O\left( \bar{r}_{\mathrm{I},k}-r_{\mathrm{eff},k} +\widetilde{O}(M_kL^{5/2}\phi\log^{1/2}(m))\right).
	\]
	By applying Lemma \ref{lma:forwardinputperturbation} to $h(x_i)$, we have
	\[
	\begin{split}
	&\left\| h\left(x_i \right)-y_{\mathrm{eff},k} \right\|^2\leq 2\left\| h\left(x_i \right)-h\left(x_k' \right)\right\|^2 + 2\left\| h\left(x_k' \right)-y_{\mathrm{eff},k} \right\|^2 \\
	\leq & \frac{1}{M_k\rho^2}O\left( \bar{r}_{\mathrm{I},k}-r_{\mathrm{eff},k}+\widetilde{O}(M_kL^{5/2}\phi\log^{1/2}(m))\right)+\widetilde{O}(L^{5}\phi^2\log(m)).
	\end{split}
	\]
	Taking weighted summation in the cell $\mathcal{I}_{\phi,k}$, since $\sum_{i\in\mathcal{I}_{\phi,k}}(\mu_i+\lambda_i)=M_k$, we have
	\[
	\begin{split}
	\sum_{i\in\mathcal{I}_{\phi,k}}(\mu_i+\lambda_i)\left\| h\left(x_i \right)-y_{\mathrm{eff},k} \right\|^2
	\leq \frac{1}{\rho^2}O\left( \bar{r}_{\mathrm{I},k}-r_{\mathrm{eff},k}+\widetilde{O}(M_kL^{5/2}\phi\log^{1/2}(m))\right).
	\end{split}
	\]
\end{proof}

\subsection{Proof of Lemma \ref{lma:gradientboud}}\label{sec:proofgradientboud}
\begin{lemma}\label{lma:gradientlowerboundlastlayer}
	Suppose that $m\geq \Omega(N^2d^2\phi^{-1})$. For any $u_i: \|u_i\|\leq \mu_i+\lambda_i, i\in[n']$ and $v_j\sim\mathcal{N}\left( 0,(1/d)\mathbf{I}\right) $, $w_j\sim\mathcal{N}\left( 0,(2/m)\mathbf{I}\right) $, with Assumption \ref{asp:smoothset} satisfied, with probability at least $1-O(\phi)$, we have
	\[
	\sum_{j=1}^m\|\sum_{k=1}^N\sum_{i\in\mathcal{I}_{\phi,k}}\left\langle u_i,v_j\right\rangle \sigma'\left(\left\langle w_j,h^{(0)}_{L-1}(x_i)\right\rangle \right)h^{(0)}_{L-1}(x_i)\|^2\geq  \Omega\left( \frac{\alpha m\phi}{Nd}\right)\left(  \sum_{k=1}^N\|\sum_{i\in\mathcal{I}_{\phi,k}}u_i\|^2-\widetilde{O}(L^{5/2}\phi\log^{1/2}(m))\right).
	\]
\end{lemma}
Lemma \ref{lma:gradientlowerboundlastlayer} is proved in Section \ref{sec:proofgradientlowerboundlastlayer}

\begin{lemma}[Lemma 8.7, Lemma 8.2c in \cite{Convergence_zhu_allen2019convergence}]\label{lma:8.7,8.2Zhu}
For any $\bm{W}\in\mathcal{B}\left( \bm{W}^{(0)}, \tau\right) $, with probability at least $1-\exp\left( -O\left(m\tau^{2/3}L \right) \right) $,
\[
\left\| \bm{V}\bm{D}_{i,L}-\bm{V}\bm{D}^{(0)}_{i,L}\right\|_2\leq O\left( \tau^{1/3}L^2\sqrt{m\log(m)/d}\right)  ,
\]
and $\forall l\in[L]$,
\[
\left\| h_{i,l}-h^{(0)}_{i,l}\right\|\leq O\left( \tau L^{5/2}\sqrt{\log(m)}\right).
\]
\end{lemma}
\textbf{Proof of Lemma  \ref{lma:gradientboud}}
\begin{proof}
	Denote $u_i=u_i(h_{\bm{W}}(x_i))$. The gradient of the empirical informed risk can be expressed as
	\begin{equation}
	\bigtriangledown_{W_l}\hat{R}_{\mathrm{I}}\left(\bm{W} \right)=\sum_{i=1}^{n'}\left( u_i\bm{V}\bm{D}_{L,i}\bm{W}_{L,i}\cdots \bm{W}_{l+1,i}\bm{D}_l\right)^\top h_{l-1,i}.
	\end{equation}
	Let $\bm{G}=\bigtriangledown_{W_L}\hat{R}_{\mathrm{I}}\left(\bm{W}^{(0)} \right)=\sum_{k=1}^{N}\sum_{i\in\mathcal{I}_{\phi,k}}\left( u_i\bm{V}\bm{D}^{(0)}_{L,i}\right)^\top h^{(0),\top}_{L-1,i}$.
	By Lemma \ref{lma:gradientlowerboundlastlayer}, with probability at least $1-O(\phi)$, we have
	\begin{equation}\label{eqn:gradientlowerbound}
	\begin{split}
	\|\bm{G}\|_F^2&\geq \Omega\left( \frac{\alpha m\phi}{Nd}\right)\left(  \sum_{k=1}^N\|\sum_{i\in\mathcal{I}_{\phi,k}}u_i\|^2-\widetilde{O}(L^{5/2}\phi\log^{1/2}(m))\right)\\
	\geq & \Omega\left(\frac{\alpha m\phi\rho}{dN}\right)\left(  \sum_{k=1}^N M_k\left(\bar{r}_{\mathrm{I},k}-r_{\mathrm{eff},k} \right) -\widetilde{O}(L^{5/2}\phi\log^{1/2}(m))\right)\\
	\geq &\Omega\left(\frac{\alpha m\phi\rho\bar{\lambda}}{dN^2}\right) \left( \hat{R}_{\mathrm{I}}-\hat{R}_{\mathrm{eff}}-\widetilde{O}(L^{5/2}\phi\log^{1/2}(m))\right),
	\end{split}
	\end{equation}
	where $\hat{R}_{\mathrm{eff}}=\sum_{k=1}^Nr_{\mathrm{eff},k}$ and the second inequality comes from Lemma \ref{lma:cellriskgradientbound} and the last inequality holds because $\sum_{k=1}^N M_k\left(\bar{r}_{\mathrm{I},k}-r_{\mathrm{eff},k} \right)\geq \bar{M}\sum_{k=1}^N \left(\bar{r}_{\mathrm{I},k}-r_{\mathrm{eff},k} \right)$ with $\bar{M}=\min_k M_k$, $M_k=\sum_{\mathcal{I}_{\phi,k}}(\mu_i+\lambda_i)$, and $N\bar{M}=\bar{\lambda}$.
	
	Here, we need to discuss more about $\bar{\lambda}$ which is different for different objectives. Denote $\bar{p}_z=\min_k|S_z\cap S_{\mathcal{I}_{\phi,k}}|$, $\bar{p}_g=\min_k|S_{g}\cap S_{\mathcal{I}_{\phi,k}}|$ and $\bar{p}_{g/z}=\min_k|(S_{g}\setminus S_z)\cap S_{\mathcal{I}_{\phi,k}}|$. When $\lambda\neq 1$ or $\lambda\neq 0$, $\bar{\lambda}=N\min\left\lbrace \frac{(1-\lambda)\bar{p}_z}{n_z}+\frac{\lambda \bar{p}_g}{n_g}, \frac{\lambda \bar{p}_g}{n_g} \right\rbrace\geq\Omega(\lambda) $ for objective \eqref{eqn:informlossnn},$\bar{\lambda}=N\min\left\lbrace \frac{(1-\lambda)(1-\beta)\bar{p}_z}{n_z}+\frac{(1-\lambda)\beta \bar{p}_g}{n_g}, \frac{\lambda \bar{p}_g}{n'_{g}} \right\rbrace\geq \Omega(\min(1-\lambda,\lambda)) $ for objective \eqref{eqn:informlossnn3}\footnote{Here, $S_{\mathcal{I}_{\phi,k}}$ is the set of samples with their indices in $\mathcal{I}_{\phi,k}$.  Thus, there exists a constant $C$ such that $n_z\leq C N\bar{p}_z$, $n_g\leq C N \bar{p}_g$, $(n_g-n_z)\leq C N\bar{p}_{g/z}$ where $C$ relies on the input distribution. }. Beside, the cases when $\lambda=0$ or $\lambda=1$ mean the corresponding datasets are empty (e.g. when $\lambda=1$ in \eqref{eqn:informlossnn3},  $S_z=\emptyset$ and $S'_{g}=\emptyset$ ), so we have $\bar{\lambda}=1$. In conclusion, we have $\bar{\lambda}=\Omega(\min(1-\lambda, \lambda)\mathds{1}(\lambda\in(0,1))+\mathds{1}(\lambda \in\{0,1\}))$ for two objectives.
	
	Next we bound the difference of $\|\bm{G}\|$ and $\|\bigtriangledown_{W_L}\hat{R}_{\mathrm{I}}(\bm{W})\|$ with $\bm{W}\in\mathcal{B}(\bm{W}^{(0)},\tau)$.
	By definition, we have
	\begin{equation}\label{eqn:gradientboundproof1}
	\begin{split}
	&\left\| \bm{G}-\bigtriangledown_{W_L}\hat{R}_{\mathrm{I}}(\bm{W})\right\|_F=\left\| \sum_{k=1}^{N}\sum_{i\in\mathcal{I}_{\phi,k}}\left( u_i\bm{V}\bm{D}^{(0)}_{L,i}\right)^\top h^{(0),\top}_{L-1,i}-\sum_{k=1}^{N}\sum_{i\in\mathcal{I}_{\phi,k}}\left( u_i\bm{V}\bm{D}_{L,i}\right)^\top h^\top_{L-1,i}\right\|_F \\
	\leq &\underset{(a)}{\left\| \sum_{k=1}^{N}\sum_{i\in\mathcal{I}_{\phi,k}}\left( u_i\bm{V}\bm{D}^{(0)}_{L,i}-u_i\bm{V}\bm{D}_{L,i}\right)^\top h^{(0),\top}_{L-1,i}\right\|_F} +\underset{(b)}{\left\|\sum_{k=1}^{N}\sum_{i\in\mathcal{I}_{\phi,k}}\left( u_i\bm{V}\bm{D}_{L,i}\right)^\top \left( h^{(0)}_{L-1,i}-h_{L-1,i}\right)^\top \right\|_F}
	\end{split}
	\end{equation}
	
	For the term $(a)$ in the above inequality, denoting $h_{L-1,k}=h_{L-1}(x_k')$ and letting $	(a)_k$ be the $k$th item in the summation, we have
	\[
	\begin{split}
	(a)_k\leq &\left\|\sum_{i\in\mathcal{I}_{\phi,k}}\left( \bm{D}^{(0)}_{L,k}-\bm{D}_{L,k}\right) \bm{V}^{\top} u_i^{\top} h^{{(0)},\top}_{L-1,k} \right\|_F\\
	& +\left\| \sum_{i\in\mathcal{I}_{\phi,k}}\left( \bm{D}^{(0)}_{L,i}-\bm{D}_{L,i}\right) \bm{V}^{\top} u_i^{\top} h^{{(0)},\top}_{L-1,i} - \sum_{i\in\mathcal{I}_{\phi,k}}\left( \bm{D}^{(0)}_{L,k}-\bm{D}_{L,k}\right) \bm{V}^{\top} u_i^{\top} h^{{(0)},\top}_{L-1,k} \right\|_F\\
	\leq& O\left(\tau^{1/3}L^2\sqrt{m\log(m)/d} \right)\left\| \sum_{i\in\mathcal{I}_{\phi,k}}u_i\right\|  +\sum_{i\in\mathcal{I}_{\phi,k}}\left\|  \bm{V}\left( \bm{D}^{(0)}_{L,i}-\bm{D}_{L,i}\right)\right\|_2  \left\| u_i\right\| \left\|  h^{{(0)}}_{L-1,i}  \right\|\\&+\sum_{i\in\mathcal{I}_{\phi,k}}\left\|  \bm{V}\left( \bm{D}^{(0)}_{L,k}-\bm{D}_{L,k}\right)\right\|_2  \left\| u_i\right\| \left\|  h^{{(0)}}_{L-1,k}  \right\|\\
	\leq& O\left(\tau^{1/3}L^2\sqrt{m\log(m)/d} \right)\left( \left\| \sum_{i\in\mathcal{I}_{\phi,k}}u_i\right\|+M_k\right),
	\end{split}
	\]
	where the second inequality comes from Lemma \ref{lma:8.7,8.2Zhu} and Cauchy-Schwartz inequality, and the last inequality comes from Lemma \ref{lma:8.7,8.2Zhu} and $\sum_{i\in\mathcal{I}_{\phi,k}}\|u_i\|\leq M_k$ and Lemma \ref{initialized risk} such that $\left\|  h^{{(0)}}_{L-1,i}  \right\|\leq \tilde{O}(1)$ with probability at least $1-O(\phi)$.
	
	For the term $(b)$, it holds that
	\[
	\begin{split}
	(b)_k\leq& \left\|\sum_{i\in\mathcal{I}_{\phi,k}}\left(\bm{V}\bm{D}_{L,k}\right)^\top u_i^\top \left( h^{(0)}_{L-1,k}-h_{L-1,k}\right)^\top \right\|_F\\
	+& \left\|\sum_{i\in\mathcal{I}_{\phi,k}}\left(\bm{V}\bm{D}_{L,i}\right)^\top u_i^\top \left( h^{(0)}_{L-1,i}-h_{L-1,i}\right)^\top-\sum_{i\in\mathcal{I}_{\phi,k}}\left(\bm{V}\bm{D}_{L,k}\right)^\top u_i^\top \left( h^{(0)}_{L-1,k}-h_{L-1,k}\right)^\top \right\|_F\\
	\leq & O\left( \tau L^{5/2}\sqrt{m\log(m)/d}\right) \left\| \sum_{i\in\mathcal{I}_{\phi,k}}u_i\right\|+\sum_{i\in\mathcal{I}_{\phi,k}}\left\| \bm{V}\bm{D}_{L,i}\right\|_2  \left\| u_i\right\|   \left\| h^{(0)}_{L-1,i}-h_{L-1,i}\right\| \\
	&+\sum_{i\in\mathcal{I}_{\phi,k}}\left\| \bm{V}\bm{D}_{L,k}\right\|_2  \left\| u_i\right\|   \left\| h^{(0)}_{L-1,k}-h_{L-1,k}\right\|\\
	\leq &O\left( \tau L^{5/2}\sqrt{m\log(m)/d}\right) \left( \left\| \sum_{i\in\mathcal{I}_{\phi,k}}u_i\right\|+M_k\right),
	\end{split}
	\]
where the second inequality comes from Lemma \ref{lma:8.7,8.2Zhu} and Cauchy-Schwartz inequality, and the last inequality comes from Lemma \ref{lma:8.7,8.2Zhu} and $\sum_{i\in\mathcal{I}_{\phi,k}}\|u_i\|\leq M_k$ and Lemma \ref{initialized risk} such that $\left\|  h^{{(0)}}_{L-1,i}  \right\|\leq \tilde{O}(1)$ with probability at least $1-O(\phi)$.
	
	Therefore, we can bound Eqn.~\eqref{eqn:gradientboundproof1} as
	\begin{equation}\label{eqn:gradientlowerboundproof2}
	\begin{split}
	&\left\| \bm{G}-\bigtriangledown_{W_L}\hat{R}_{\mathrm{I}}(\bm{W})\right\|_F \leq O\left( \tau^{1/3} L ^{5/2}\sqrt{m\log(m)/d}\right) \left( \sum_{k=1}^N\left\| \sum_{i\in\mathcal{I}_{\phi,k}}u_i\right\|+1\right)\\
	\leq &O\left( N^{1/2}\tau^{1/3}L^{5/2}\sqrt{m\log(m)/d}\right) \left( \sqrt{\left( \hat{R}_{\mathrm{I}}-\hat{R}_{\mathrm{eff}}+\widetilde{O}(L^{5/2}\phi\log^{1/2}(m))\right) }+1/\sqrt{N}\right)\\
	\leq &O\left( N^{1/2}\tau^{1/3} L^{5/2}\sqrt{m\log(m)/d}\right) \left( \sqrt{\left( \hat{R}_{\mathrm{I}}-\hat{R}_{\mathrm{eff}}-\widetilde{O}(L^{5/2}\phi\log^{1/2}(m))\right) }\right),
	\end{split}
	\end{equation}
	where the second inequality holds by Lemma \ref{lma:cellriskgradientbound} and the last inequality holds because $\phi$ is small enough such that $\hat{R}_{\mathrm{I}}\left(\bm{W}\right)-\hat{R}_{\mathrm{eff}}+\widetilde{O}(L^{5/2}\phi\log^{1/2}(m))\leq  2(\hat{R}_{\mathrm{I}}\left(\bm{W}\right)-\hat{R}_{\mathrm{eff}}-\widetilde{O}(L^{5/2}\phi\log^{1/2}(m)))$.
	
	Combining Eqn.~\eqref{eqn:gradientlowerboundproof2} with Eqn.~\eqref{eqn:gradientlowerbound}, we have
	\[
	\begin{split}
	&\left\|\bigtriangledown_{W_L}\hat{R}_{\mathrm{I}}\left(\bm{W} \right) \right\|_F\geq \left\| \bm{G}\right\|_F- \left\| \bm{G}-\bigtriangledown_{W_L}\hat{R}_{\mathrm{I}}(\bm{W})\right\|_F\\
	\geq &\Omega\left(\sqrt{\frac{\alpha m\phi\rho\bar{\lambda}}{dN^2}}-O\left( \tau^{1/3}N^{1/2} L^{5/2}\sqrt{m\log(m)/d}\right)\right) \left( \hat{R}_{\mathrm{I}}-\hat{R}_{\mathrm{eff}}-\widetilde{O}(L^{5/2}\phi\log^{1/2}(m))\right)\\
	\geq &\Omega\left(\sqrt{\frac{\alpha m\phi\rho\bar{\lambda}}{dN^2}}\left( \hat{R}_{\mathrm{I}}-\hat{R}_{\mathrm{eff}}-\widetilde{O}(L^{5/2}\phi\log^{1/2}(m))\right)\right),
	\end{split}
	\]
	where the last inequality holds by the choice of $m\geq\Omega\left(N^{11}L^{15}d\phi^{-4}\rho^{-4}\bar{\lambda}^{-4}\alpha^{-4}\log^3(m) \right) $ and the weight update range  in the proof of Theorem \ref{the:convergence}  such that $\tau^{1/3}=O(N^{-3/2}\phi^{1/2}\rho^{1/2}\bar{\lambda}^{1/2}\alpha^{1/2}L^{-5/2}\log^{-1/2}(m))$.
	\[
	\left\|\bigtriangledown_{\bm{W}}\hat{R}_{\mathrm{I}}\left(\bm{W} \right) \right\|_F^2\geq \left\|\bigtriangledown_{W_L}\hat{R}_{\mathrm{I}}\left(\bm{W} \right) \right\|_F^2\geq \Omega\left(\frac{\alpha m\phi\rho\bar{\lambda}}{dN^2}\right) \left( \hat{R}_{\mathrm{I}}-\hat{R}_{\mathrm{eff}}-\widetilde{O}(L^{5/2}\phi\log^{1/2}(m))\right).
	\]
\end{proof}

\subsection{Proof of Lemma \ref{lma:risksmoothness}}\label{sec:proofrisksmoothness}

\textbf{Proof of Lemma \ref{lma:risksmoothness}}
\begin{proof}
	Since the maximum eigenvalue of the second order derivation of the informed risk function $	r_{\mathrm{I},i}(h_{\bm{W},i})$ with respect to $h$ is less than $\mu_i+\lambda_i$, we have
	\begin{equation}
	\begin{split}
	r_{\mathrm{I},i}(h_{\bm{W}',i})- r_{\mathrm{I},i}(h_{\bm{W},i})
	\leq u_i(h_{\bm{W},i})^\top\left(h_{\bm{W}',i}- h_{\bm{W},i}\right) +O\left( \frac{\mu_i+\lambda_i}{2}\left\| h_{\bm{W}',i}- h_{\bm{W},i}\right\|^2\right) .
	\end{split}
	\end{equation}
	Then denote $\widehat{\bm{W}}=\bm{W}'-\bm{W}$. We have
		\begin{equation}
	\begin{split}
	&\sum_{i\in\mathcal{I}_{\phi,k}}r_{\mathrm{I},i}(h_{\bm{W}',i})- r_{\mathrm{I},i}(h_{\bm{W},i})- \left\langle \bigtriangledown_{\bm{W}}r_{\mathrm{I},i}(h_{\bm{W},i}) ,\widehat{\bm{W}}\right\rangle\\
	\leq &\sum_{i\in\mathcal{I}_{\phi,k}} u_i(h_{\bm{W},i})^\top\left(h_{\bm{W}',i}- h_{\bm{W},i}-\left\langle \bigtriangledown_{\bm{W}}h_{\bm{W},i},\widehat{\bm{W}}\right\rangle \right)+O\left( \sum_{i\in\mathcal{I}_{\phi,k}}\frac{\mu_i+\lambda_i}{2}\left\| h_{\bm{W}',i}- h_{\bm{W},i}\right\|^2\right) \\
	\leq &\left( \sum_{i\in\mathcal{I}_{\phi,k}} u_i(h_{\bm{W},i})\right) ^\top\left(h_{\bm{W}',k}- h_{\bm{W},k}-\left\langle \bigtriangledown_{\bm{W}}h_{\bm{W},k},\widehat{\bm{W}}\right\rangle \right)
	+O\left( \sum_{i\in\mathcal{I}_{\phi,k}}\frac{\mu_i+\lambda_i}{2}\left\| h_{\bm{W}',i}- h_{\bm{W},i}\right\|^2\right)\\
	&+\sum_{i\in\mathcal{I}_{\phi,k}} \left\|u_i(h_{\bm{W},i})\right\|\left\|\left[ \left(h_{\bm{W}',i}- h_{\bm{W},i}-\left\langle \bigtriangledown_{\bm{W}}h_{\bm{W},i},\widehat{\bm{W}}\right\rangle \right)-\left(h_{\bm{W}',k}- h_{\bm{W},k}-\left\langle \bigtriangledown_{\bm{W}}h_{\bm{W},k},\widehat{\bm{W}}\right\rangle \right)\right]\right\|,
	\end{split}
	\end{equation}
	where Cauchy-Schwartz inequality is used in the last inequality.
	By Theorem 4 in \cite{Convergence_zhu_allen2019convergence}, we have with probability at least $1-\exp(-\Omega(m\tau^{2/3}L))$,
	\[
	\begin{split}
	&\left\|h_{\bm{W}',i}- h_{\bm{W},i}-\left\langle \bigtriangledown_{\bm{W}}h_{\bm{W},i},\widehat{\bm{W}}\right\rangle \right\|\\
	\leq & O\left( \tau^{1/3}L^{5/2}\sqrt{m\log(m)}d^{-1/2}\right)\left\| \widehat{\bm{W}}\right\|+O\left( L^2\sqrt{m/d}\left\| \widehat{\bm{W}}\right\|^2\right)
	\end{split}
\]
By Claim 11.2 in \cite{Convergence_zhu_allen2019convergence}, we have
\[
\left\| h_{\bm{W}',i}- h_{\bm{W},i}\right\|\leq O(L\sqrt{m/d})\left\| \widehat{\bm{W}}\right\|.
\]
Thus since $\left\|u_i(h_{\bm{W},i})\right\|\leq O(\mu_i+\lambda_i)$, we have
\begin{equation}\label{eqn:riskerror}
\begin{split}
	&\sum_{i\in\mathcal{I}_{\phi,k}}r_{\mathrm{I},i}(h_{\bm{W}',i})- r_{\mathrm{I},i}(h_{\bm{W},i})- \left\langle \bigtriangledown_{\bm{W}}r_{\mathrm{I},i}(h_{\bm{W},i}) ,\widehat{\bm{W}}\right\rangle\\
	\leq &O\left( \sqrt{M_k\left(\bar{r}_{\mathrm{I},k}-r_{\mathrm{eff},k}+\widetilde{O}(L^{5/2}\phi\log^{1/2}(m))\right) }+M_k\right)  O\left( \tau^{1/3}L^{5/2}\sqrt{m\log(m)}d^{-1/2}\right)\left\| \widehat{\bm{W}}\right\|\\
&+O(M_kL^2m/d)\left\|\widehat{\bm{W}} \right\|^2.
\end{split}
\end{equation}
where the inequality comes from Lemma \ref{lma:cellriskgradientbound}.
Taking summation over $i\in[N]$, we have
\[
\begin{split}
&\hat{R}_{\mathrm{I}}\left(\bm{W}' \right)-\hat{R}_{\mathrm{I}}\left(\bm{W}\right)\leq \left\langle \bigtriangledown_{\bm{W}}\hat{R}_{\mathrm{I}}\left( \bm{W}\right) ,\bm{W}'-\bm{W}\right\rangle+O(L^2m/d)\left\|\widehat{\bm{W}} \right\|^2\\
&\quad +\left( \sqrt{\left( \hat{R}_{\mathrm{I}}\left(\bm{W}\right)-\hat{R}_{\mathrm{eff}}+\widetilde{O}(L^{5/2}\phi\log^{1/2}(m))\right)}+1/\sqrt{N} \right) O\left( N^{1/2}\tau^{1/3}L^{5/2}\sqrt{m\log(m)}d^{-1/2}\right)\left\| \widehat{\bm{W}}\right\|\\
\leq&\left\langle \bigtriangledown_{\bm{W}}\hat{R}_{\mathrm{I}}\left( \bm{W}\right) ,\bm{W}'-\bm{W}\right\rangle+O(L^2m/d)\left\|\widehat{\bm{W}} \right\|^2\\
&\quad +\left( \sqrt{\left( \hat{R}_{\mathrm{I}}\left(\bm{W}\right)-\hat{R}_{\mathrm{eff}}-\widetilde{O}(L^{5/2}\phi\log^{1/2}(m))\right)}\right) O\left( N^{1/2}\tau^{1/3}L^{5/2}\sqrt{m\log(m)}d^{-1/2}\right)\left\| \widehat{\bm{W}}\right\|,
\end{split}
\]
where the second inequality comes from  the choice of $\phi$ such that $\hat{R}_{\mathrm{I}}\left(\bm{W}\right)-\hat{R}_{\mathrm{eff}}+\widetilde{O}(L^{5/2}\phi\log^{1/2}(m))\leq  2(\hat{R}_{\mathrm{I}}\left(\bm{W}\right)-\hat{R}_{\mathrm{eff}}-\widetilde{O}(L^{5/2}\phi\log^{1/2}(m)))$ and  $1/\sqrt{N}\leq \sqrt{\phi}\leq \sqrt{\left( \hat{R}_{\mathrm{I}}-\hat{R}_{\mathrm{eff}}+\widetilde{O}(L^{5/2}\phi\log^{1/2}(m))\right) }$.

\end{proof}

\subsection{Proof of Lemma \ref{initialized risk}}\label{sec:proofinitialized risk}
\begin{proof}
	By Lemma 7.1 in \cite{Convergence_zhu_allen2019convergence}, with probability at least $1-O(NL)\exp\left(-\Omega\left( m/L\right)  \right) $, we have
	$\forall k\in[N], \left\| h^{(0)}_{k,L}\right\|\leq 2.$. Thus by Lemma \ref{lma:forwardinputperturbation}, we have with probability at least $1-O(\phi)$,
	\[
	\forall k\in[N], \forall i\in\mathcal{I}_{\phi,k}, \left\| h^{(0)}_{i,L}\right\|\leq 2+\widetilde{O}(L^{5/2}\phi\log^{1/2}(m)).
	\]
	Then since each entry of $ \bm{V}$ satisfies $\mathcal{N}\left(0,\frac{1}{d} \right) $ and $O(NL)\exp\left(-\Omega\left( m/L\right)  \right)\leq O(\phi)$,  we have with probability at least $1-O(\phi)$,
	\[
	\left\| h_{\bm{W}^{(0)},i}\right\| =\left\| \bm{V}h_{i,L}^{(0)}\right\| \leq 2\left\| h^{(0)}_{i,L}\right\|\sqrt{\log(1/\phi)}=O\left(\sqrt{\log(1/\phi)} \right).
	\]
 Let $r_{\mathrm{I},i}(y_{\mathrm{eff},k(x_i)})=\mu_ir(y_{\mathrm{eff},k(x_i)},y_i)+\lambda_ir_K(y_{\mathrm{eff},k(x_i)},g(x_i))$. Thus with probability at least $1-O(\phi)$, by 1-Lipschitz continuity of risk functions, we have
	\[
	\begin{split}
	&r^{(0)}_{\mathrm{I},i}-r_{\mathrm{I},i}(y_{\mathrm{eff},k(x_i)})
	\leq \left( \mu_i+\lambda_i\right) \left\| h_{\bm{W}^{(0)},i}-y_{\mathrm{eff},k(x_i)}\right\|\\
	\leq& \left( \mu_i+\lambda_i\right) \left( \left\| h_{\bm{W}^{(0)},i}\right\| +\left\| y_{\mathrm{eff},k(x_i)}\right\|\right)\\
	\leq& O\left(\left( \mu_i+\lambda_i\right)\log^{1/2}(1/\phi) \right).
	\end{split}
	\]
	Taking summation for $i\in[n']$, we have
	\[
	\hat{R}_{\mathrm{I}}\left(\bm{W}^{(0)} \right)-\hat{R}_{\mathrm{eff}}\leq O\left(\log^{1/2}(1/\phi) \right).
	\]
\end{proof}

\subsection{Proof of Lemma \ref{lma:boundefflabel}}\label{sec:proofboundefflabel}
\begin{proof}
	\textbf{Proof of (a)}:
	Denote $\mathcal{U}_{g}'=\mathcal{U}_{\phi}(S_z)=\left\lbrace k\in[N]\mid \exists x\in S_z, x \in \mathcal{C}_{\phi,k}\right\rbrace$ as the index collection of smooth sets that contain at least one labeled sample, and $\mathcal{U}_{g}''=[N]\setminus \mathcal{U}_{g}'$ as the index collection of smooth sets that only contain knowledge-supervised samples. Denote $h_{\mathrm{K},i}^*=h_{\mathrm{K}}^*(x_i)$ for notation simplicity, and recall that $x_k'\in\mathcal{X}_{\phi}$ is the representative input of the smooth set $k$, so we have
	\begin{equation}\label{eqn:boundefflabelproof2}
	\begin{split}
	&\frac{1}{n_{g}''}\sum_{S_{g}''}r_{\mathrm{K}}\left(h_{\mathrm{K},i}^*,g_i \right) =\frac{1}{n_{g}''}\sum_{k\in \mathcal{U}_{g}''}\sum_{\mathcal{I}_{\phi,k}}r_{\mathrm{K}}\left(h_{\mathrm{K},i}^*,g_i \right) \\
	\geq &\frac{1}{n_{g}''}\sum_{k\in \mathcal{U}_{g}''}\sum_{\mathcal{I}_{\phi,k}}r_{\mathrm{K}}\left(h_{\mathrm{K}}^*(x_k'),g_i \right) -\widetilde{O}\left( L^{5/2}\phi\log^{1/2}(m)\right) \\
	\geq & \frac{1}{n_{g}''}\sum_{k\in \mathcal{U}_{g}''}\sum_{\mathcal{I}_{\phi,k}}r_{\mathrm{K}}\left(y_{\mathrm{eff},k},g_i \right)+\frac{1}{n_{g}''}\sum_{k\in \mathcal{U}_{g}''}\left\langle \bigtriangledown_h\sum_{\mathcal{I}_{\phi,k}}r_{\mathrm{K}}\left(y_{\mathrm{eff},k},g_i \right), h_{\mathrm{K}}^*(x_k')-y_{\mathrm{eff},k} \right\rangle\\
	&+\frac{\rho}{2n_{g}''}\sum_{k\in \mathcal{U}_{g}''}| \mathcal{I}_{\phi,k}|\left\| h_{\mathrm{K}}^*(x_k')-y_{\mathrm{eff},k}\right\|^2   -\widetilde{O}\left( L^{5/2}\phi\log^{1/2}(m)\right) ,
	\end{split}
	\end{equation}
	where the first inequality holds by Lemma \ref{lma:forwardinputperturbation} and Lipschitz continuity of the risk function, and the second inequality holds by the strongly convexity of $\sum_{\mathcal{I}_{\phi,k}}r_{\mathrm{K}}\left(h,g_i \right)$ with respect to $h$.
	 By subtracting $\frac{1}{n_{g}''}\sum_{k\in \mathcal{U}_{g}''}\sum_{\mathcal{I}_{\phi,k}}r_{\mathrm{K}}\left(y_{\mathrm{eff},k},g_i \right)$ from both sides of \eqref{eqn:boundefflabelproof2}, we have
	\begin{equation}\label{eqn:boundefflabelproof1}
	\begin{split}
	&\frac{1}{n_{g}''}\sum_{k\in \mathcal{U}_{g}''}| \mathcal{I}_{\phi,k}|\left\| h_{\mathrm{K}}^*(x_k')-y_{\mathrm{eff},k}\right\|^2\\
	\leq &	\frac{2}{\rho n_{g}''}\sum_{k\in \mathcal{U}_{g}''}\left\|  \bigtriangledown_h\sum_{\mathcal{I}_{\phi,k}}r_{\mathrm{K}}\left(y_{\mathrm{eff},k},g_i \right)\right\| \left\|  h_{\mathrm{K}}^*(x_k')-y_{\mathrm{eff},k} \right\|  +\widetilde{O}\left( L^{5/2}\phi\log^{1/2}(m)\right) \\
	\leq & O\left(\frac{1}{ n_{g}''}\sum_{k\in \mathcal{U}_{g}''}\left\| \bigtriangledown_h\sum_{\mathcal{I}_{\phi,k}}r_{\mathrm{K}}\left(y_{\mathrm{eff},k},g_i \right)\right\|  \right)+\widetilde{O}\left( L^{5/2}\phi\log^{1/2}(m)\right)  \\
	\leq &\sum_{k\in \mathcal{U}_{g}''}\left[  \sqrt{\frac{2|\mathcal{I}_{\phi,k}|}{n_g''}\left(r_{\mathrm{K}}\left(y_{\mathrm{eff},k},g_i \right)-r_{\mathrm{K}}\left(y_{\mathrm{eff},k},g_i \right) \right) }+\frac{|\mathcal{I}_{\phi,k}|}{n_g''}\widetilde{O}\left(L^{5/4}\phi^{1/2}\log^{1/4}(m) \right) \right] +\widetilde{O}\left( L^{5/2}\phi\log^{1/2}(m)\right) \\
	=&\widetilde{O}\left(L^{5/4}\phi^{1/2}\log^{1/4}(m) \right)+\widetilde{O}(L^{5/2}\phi\log^{1/2}(m))\leq \widetilde{O}\left(L^{5/4}\phi^{1/2}\log^{1/4}(m) \right),
	\end{split}
	\end{equation}
	where the first inequality holds since $\left\lbrace h_{\mathrm{K},i}^*,i\in S''_{g}\right\rbrace $ minimizes $\frac{1}{n_{g}''}\sum_{S_{g}''}r_{\mathrm{K}}\left(h(x_i),g_i \right)$, the second inequality holds since $\left\|  h_{\mathrm{K}}^*(x_k')-y_{\mathrm{eff},k} \right\|\leq\left\|  h_{\mathrm{K}}^*(x_k')\right\|+\left\|y_{\mathrm{eff},k} \right\| \leq \widetilde{O}(1)$ by Lemma \ref{initialized risk} and Lemma 8.2(c) in \cite{Convergence_zhu_allen2019convergence}, and the third inequality holds by applying Lemma \ref{lma:cellriskgradientbound} for $\sum_{\mathcal{I}_{\phi,k}}r_{\mathrm{K}}\left(y_{\mathrm{eff},k},g_i \right)$ with $r_{\mathrm{eff},k}=\sum_{\mathcal{I}_{\phi,k}}r_{\mathrm{K}}\left(y_{\mathrm{eff},k},g_i \right)$.
	Therefore, by Lemma \ref{lma:forwardinputperturbation}, we have
	\begin{equation}\label{eqn:boundefflabelproof11}
	\frac{1}{n_{g}''}\sum_{S_{g}''}\left\| h^*_{\mathrm{K},i}-y_{\mathrm{eff},k(x_i)}\right\|^2\leq \frac{1}{n_{g}''}\sum_{k\in \mathcal{U}_{g}''}| \mathcal{I}_{\phi,k}|\left\| h_{\mathrm{K}}^*(x_k')-y_{\mathrm{eff},k}\right\|^2+\widetilde{O}\left(L^{5/2}\phi\log^{1/2}(m)\right) \leq \widetilde{O}\left(L^{5/4}\phi^{1/2}\log^{1/4}(m) \right).
	\end{equation}
	
	Similarly, denote $r_{\mathrm{R},\beta}(h(x_i))=\frac{1-\beta}{n_z}r(h(x_i),y_i)\mathds{1}\left(x_i\in S_z \right) + \frac{\beta}{n_g'}r_K(h(x_i),g(x_i))\mathds{1}\left(x_i\in S_{g}' \right)  $. We have
	\begin{equation}\label{eqn:eqn:boundefflabelproof3}
	\begin{split}
	&\sum_{S_z\bigcup S_{g}'}r_{\mathrm{R},\beta}\left(h_{\mathrm{R},\beta}^*(x_i) \right)
	\geq \sum_{k\in \mathcal{U}_{g}'}\sum_{\mathcal{I}_{\phi,k}}r_{\mathrm{R},\beta}\left(h_{\mathrm{R},\beta}^*(x_k') \right) -\widetilde{O}\left(L^{5/2}\phi\log^{1/2}(m)\right)\\
	\geq &\sum_{k\in \mathcal{U}_{g}'}\sum_{\mathcal{I}_{\phi,k}}r_{\mathrm{R},\beta}\left(y_{\mathrm{eff},k} \right)-\sum_{k\in \mathcal{U}_{g}'} \left\| \bigtriangledown_h\sum_{\mathcal{I}_{\phi,k}}r_{\mathrm{R},\beta}(y_{\mathrm{eff},k})\right\| \left\| h_{\mathrm{R},\beta}^*(x_k')-y_{\mathrm{eff},k} \right\| \\
	&+\frac{\rho}{2}\sum_{k\in \mathcal{U}_{g}'}M_k\left\| h_{\mathrm{R},\beta}^*(x_k')-y_{\mathrm{eff},k}\right\|^2   -\widetilde{O}\left(L^{5/2}\phi\log^{1/2}(m)\right),
	\end{split}
	\end{equation}
	where $M_k=\sum_{i\in\mathcal{I}_{\phi,k}}\left[ \frac{1-\beta}{|S_z|}\mathds{1}\left(x_i\in S_z \right) +\frac{\beta}{|S_{g}'|}\mathds{1}\left(x_i\in S_{g}' \right) \right] $, the first inequality holds by Lemma \ref{lma:forwardinputperturbation} and Lipschitz continuity of the risk function, and the second inequality holds by the strongly convexity of $r_{\mathrm{R},\beta}\left(h_{\mathrm{R},\beta}^* \right)$ with respect to $h_{\mathrm{R},\beta}^*$.
	Then, subtracting $\sum_{k\in \mathcal{U}_{g}'}\sum_{\mathcal{I}_{\phi,k}}r_{\mathrm{R},\beta}\left(y_{\mathrm{eff},k} \right)$ from both sides of \eqref{eqn:eqn:boundefflabelproof3}, similarly as Eqn.~\eqref{eqn:boundefflabelproof1},it holds that
	\begin{equation}\label{eqn:boundefflabelproof4}
	\begin{split}
	&\sum_{k\in \mathcal{U}_{g}'}M_k\left\| h_{\mathrm{R},\beta}^*(x_k')-y_{\mathrm{eff},k}\right\|^2\\
	\leq &	\frac{2}{\rho}\sum_{k\in \mathcal{U}_{g}'}\left\|  \bigtriangledown_h\sum_{\mathcal{I}_{\phi,k}}r_{\mathrm{R},\beta}\left(y_{\mathrm{eff},k} \right)\right\| \left\|  h_{\mathrm{R},\beta}^*(x_k')-y_{\mathrm{eff},k} \right\|  +\widetilde{O}\left(L^{5/2}\phi\log^{1/4}(m)\right)\\
	\leq& \widetilde{O}\left(L^{5/4}\phi^{1/2}\log^{1/4}(m) \right),
	\end{split}
	\end{equation}
	where the first inequality holds because $\sum_{S_z\bigcup S_{g}'}r_{\mathrm{R},\beta}\left(h_{\mathrm{R},\beta}^*(x_i) \right)-\sum_{k\in \mathcal{U}_{g}'}\sum_{\mathcal{I}_{\phi,k}}r_{\mathrm{R},\beta}\left(y_{\mathrm{eff},k} \right)\leq 0$, and the second inequality holds since $\left\|  h_{\mathrm{K}}^*(x_k')-y_{\mathrm{eff},k} \right\|\leq  \widetilde{O}(1)$ by Lemma \ref{initialized risk} and Lemma 8.2(c) in \cite{Convergence_zhu_allen2019convergence} and then applying Lemma \ref{lma:cellriskgradientbound}. 
	Therefore, by Lemma \ref{lma:forwardinputperturbation}, we have
	\[
	\begin{split}
	\frac{1-\beta}{n_z}\sum_{S_z}\left\|h^*_{\mathrm{R},\beta,i}-y_{\mathrm{eff},k(x_i)}\right\|^2 +\frac{\beta}{n_{g}'}\sum_{S_{g}'}\left\| h^*_{\mathrm{R},\beta,i}-y_{\mathrm{eff},k(x_i)}\right\|^2&\leq O\left(\sum_{k\in \mathcal{U}_{g}'}M_k\left\| h_{\mathrm{R},\beta}^*(x_k')-y_{\mathrm{eff},k}\right\|^2\right)\\
	&\leq \widetilde{O}\left(L^{5/4}\phi^{1/2}\log^{1/4}(m) \right).
	\end{split}
	\]
	\textbf{Proof of (b)}:
	Replacing $h_{\mathrm{K}}^*$ in Eqn.~\eqref{eqn:boundefflabelproof2} with $\bar{h}_{\mathrm{K}}^*$ and applying the second and third inequality in Eqn.~\eqref{eqn:boundefflabelproof1}, we have
	\[
	\begin{split}
	&\frac{1}{n_{g}''}\sum_{k\in \mathcal{U}_{g}''}| \mathcal{I}_{\phi,k}|\left\| \bar{h}_{\mathrm{K}}^*(x_k')-y_{\mathrm{eff},k}\right\|^2\\
	\leq &\widetilde{O}\left(L^{5/4}\phi^{1/2}\log^{1/4}(m) \right)+\frac{1}{n_{g}''}\sum_{S_{g}''}r_{\mathrm{K}}\left(\bar{h}_{\mathrm{K}}^*(x_i),g_i \right)-\frac{1}{n_{g}''}\sum_{k\in \mathcal{U}_{g}''}\sum_{\mathcal{I}_{\phi,k}}r_{\mathrm{K}}\left(y_{\mathrm{eff},k},g_i \right)\\
	\leq &\widetilde{O}\left(L^{5/4}\phi^{1/2}\log^{1/4}(m) \right)+\mathbb{E}\left[ r_{\mathrm{K}}\left(\bar{h}_{\mathrm{K}}^*(x),g(x) \right)\right] -\mathbb{E}\left[ r_{\mathrm{K}}\left(y_{\mathrm{eff},k(x)},g(x) \right)\right] +O(\sqrt{\frac{\log(1/\delta)}{n_{g}''}})\\
	\leq &\widetilde{O}\left(L^{5/4}\phi^{1/2}\log^{1/4}(m) \right) +O(\sqrt{\frac{\log(1/\delta)}{n_{g}''}}),
	\end{split}
	\]
	where the second inequality follows from McDiarmid's inequality and the last inequality is because $\bar{h}_{\mathrm{K}}^*(x)$ minimizes $\mathbb{E}\left[ r_{\mathrm{K}}\left(h,,g(x) \right)\right]$.
	Therefore, by Lemma \ref{lma:forwardinputperturbation} and with the same reason as \eqref{eqn:boundefflabelproof11}, we have
	\[
	\frac{1}{n_{g}''}\sum_{S_{g}''}\left\| \bar{h}^*_{\mathrm{K},i}-y_{\mathrm{eff},k(x_i)}\right\|^2\leq \widetilde{O}\left(L^{5/4}\phi^{1/2}\log^{1/4}(m) \right)+O(\sqrt{\frac{\log(1/\delta)}{n_{g}''}}).
	\]
	
	Similarly, replacing $h_{\mathrm{R},\beta}$ in  Eqns.~\eqref{eqn:eqn:boundefflabelproof3} with $h_{\mathrm{R},\beta}^*$ and with the same reason as the second inequality in  \eqref{eqn:boundefflabelproof4}, we have
	\begin{equation}\label{eqn:proof_optknowledgehyp1}
	\begin{split}
	&\sum_{k\in \mathcal{U}_{g}'}M_k\left\| \bar{h}_{\mathrm{R},\beta}^*(x_k')-y_{\mathrm{eff},k}\right\|^2\\
	\leq& \widetilde{O}\left(L^{5/4}\phi^{1/2}\log^{1/4}(m) \right)+\sum_{S_z\bigcup S_{g}'}r_{\mathrm{R},\beta}\left(\bar{h}_{\mathrm{R},\beta}^*(x_i) \right)-\sum_{k\in \mathcal{U}_{g}'}\sum_{\mathcal{I}_{\phi,k}}r_{\mathrm{R},\beta}\left(y_{\mathrm{eff},k} \right).
		\end{split}
\end{equation}
Continuing with \eqref{eqn:proof_optknowledgehyp1} and by McDiarmid's inequality, we have
	\[
\begin{split}
	&\sum_{k\in \mathcal{U}_{g}'}M_k\left\| \bar{h}_{\mathrm{R},\beta}^*(x_k')-y_{\mathrm{eff},k}\right\|^2\\
	\leq & \widetilde{O}\left(L^{5/4}\phi^{1/2}\log^{1/4}(m) \right)+\mathbb{E}\left[\sum_{S_z\bigcup S_{g}'}r_{\mathrm{R},\beta}\left(\bar{h}_{\mathrm{R},\beta}^*(x_i) \right)\right]-\mathbb{E}\left[\sum_{k\in \mathcal{U}_{g}'}\sum_{\mathcal{I}_{\phi,k}}r_{\mathrm{R},\beta}\left(y_{\mathrm{eff},k} \right)\right]+O(\sqrt{\frac{\log(1/\delta)}{n_{z}}})\\
	\leq & \widetilde{O}\left(L^{5/4}\phi^{1/2}\log^{1/4}(m) \right)+O(\sqrt{\frac{\log(1/\delta)}{n_{z}}}),
	\end{split}
	\]
	where the second inequality is because $\bar{h}_{\mathrm{R},\beta}^*(x)$ minimizes $\mathbb{E}\left[\sum_{S_z\bigcup S_{g}'}r_{\mathrm{R},\beta}\left(\bar{h}_{\mathrm{R},\beta}^*(x_i) \right)\right]$.
	And thus by Lemma \ref{lma:forwardinputperturbation}, we have
	\[
	\frac{1-\beta}{n_z}\sum_{S_z}\left\|\bar{h}^*_{\mathrm{R},\beta,i}-y_{\mathrm{eff},k(x_i)}\right\|^2 +\frac{\beta}{n_{g}'}\sum_{S_{g}'}\left\| \bar{h}^*_{\mathrm{R},\beta,i}-y_{\mathrm{eff},k(x_i)}\right\|^2\leq  \widetilde{O}\left(L^{5/4}\phi^{1/2}\log^{1/4}(m) \right)+O(\sqrt{\frac{\log(1/\delta)}{n_{z}}}).
	\]
\end{proof}

%% file: Proof3rd.tex
\section{Proof of Lemmas in Appendix~\ref{sec:proof2nd}}\label{sec:proof3rd}
\subsection{Proof of Lemma \ref{lma:zeronormperturbation}}\label{sec:proofzeronormperturbation}
\begin{proof}
We simply use $\bm{D}'$ to denote $\bm{D}_l'$.	If for some $j\in[m]$, $[\bm{D}']_{j,j}\neq 0$, then it holds that
	\begin{equation}\label{eqn:conditiondiffsign}
	|[f'_{l,i}]_j|=|[f'_{l,i,1}]_j+[f'_{l,i,2}]_j|>|[f^{(0)}_{l,k}]_j|.
	\end{equation}
	Let $\xi: \xi\leq \frac{1}{2\sqrt{m}}$ and $ \xi\leq 2\|f'_{l,i,2}\|_{\infty}$ be a parameter to be chosen later. We then discuss the zero norm of $\bm{D}'$ in the following two cases.
	
	First, we consider the case that $|[f^{(0)}_{l,k}]_j|\leq \xi$. In this case, \eqref{eqn:conditiondiffsign} is easy to be satisfied. Denote $S_1=\left\lbrace j\in[m]\mid |[f^{(0)}_{l,k}]_j|\leq \xi\right\rbrace $. Since  $[f^{(0)}_{l,k}]_j\sim\mathcal{N}(0,\frac{2}{m})$, we have $\mathbb{P}\{[f^{(0)}_{l,k}]_j\leq \xi\}\leq O(\xi \sqrt{m})$. Since $|S_1|=\sum_{j=1}^m\mathds{1}([f^{(0)}_{l,k}]_j\leq \xi)$, we have $\mathbb{E}[\exp(|S_1|)]\leq \exp(\xi m^{3/2}(e-1))$. Thus, by Chernoff bound, $\mathbb{P}(|S_1|\geq 2\xi m^{3/2})\leq \frac{\mathbb{E}[\exp(|S_1|)]}{\exp(\xi m^{3/2})}\leq \exp(\xi m^{3/2}(e-3))$. Hence, with probability at least $1-\exp(-\Omega(m^{3/2}\xi))$, we have
	\[
	|S_1|\leq O(\xi m^{3/2}).
	\]
	Then, for $j\in S_1$ such that $[\bm{D}']_{j,j}\neq 0 $, we have $|[\bm{D}'f^{(0)}_{l,i}]_j |\leq|[f^{(0)}_{l,k}]_j|+|[f'_{l,i,1}]_j|+|[f'_{l,i,2}]_j|\leq |[f'_{l,i,1}]_j|+3\xi/2$. Further, we have
	\[
	\sum_{j\in S_1} [\bm{D}'f^{(0)}_{l,i}]^2_j\leq O(\|f'_{l,i,1}\|^2+\xi^2|S_1|)\leq O(\|f'_{l,i,1}\|^2+\xi^2|S_1|)\leq O(\|f'_{l,i,1}\|^2+\xi^3m^{3/2}).
	\]
	
	Second, we consider the case that $|[f^{(0)}_{l,k}]_j|> \xi$. Denote $S_2=\left\lbrace j\in[m]\mid |[f^{(0)}_{l,k}]_j|> \xi,[\bm{D}']_{j,j}\neq 0  \right\rbrace $. Then, \eqref{eqn:conditiondiffsign} requires that
	\[
	|[f'_{l,i,1}]_j|=|[f'_{l,i}]_j-|[f'_{l,i,2}]_j||\geq |[f'_{l,i}]_j|-|[f'_{l,i,2}]_j|\geq |[f^{(0)}_{l,k}]_j|-|[f'_{l,i,2}]_j|\geq \xi-	\|f'_{l,i,2}\|_{\infty}\geq \xi/2.
	\]
	Thus we have
	\[
	|S_2|\leq \frac{ 4\|f'_{l,i,1}\|^2}{\xi^2}.
	\]
	Then since for $j\in S_2$ such that $[\bm{D}']_{j,j}\neq 0 $, the signs of $[f^{(0)}_{l,k}]_j+[f'_{l,i,1}]_j+[f'_{l,i,2}]_j$ and $[f^{(0)}_{l,k}]_j$ are opposite, we have
	\[
	|[\bm{D}'f^{(0)}_{l,i}]_j |=|[f^{(0)}_{l,k}]_j+[f'_{l,i,1}]_j+[f'_{l,i,2}]_j|\leq |[f'_{l,i,1}]_j+[f'_{l,i,2}]_j|\leq |[f'_{l,i,1}]_j|+\xi/2\leq 2|[f'_{l,i,1}]_j|.
	\]
	Therefore, it holds that
	\[
	\sum_{j\in S_s} [\bm{D}'f^{(0)}_{l,i}]^2_j\leq 4\sum_{j\in S_2}|[f'_{l,i,1}]_j|^2\leq 4\|f'_{l,i,1}\|^2.
	\]
	
	Combining the two cases, we have
	\[
	\|\bm{D}\|_0\leq |S_1|+|S_2|\leq O\left(\xi m^{3/2}+\frac{ 4\|f'_{l,i,1}\|^2}{\xi^2} \right),
	\]
	\[
	\|\bm{D}'f^{(0)}_{l,i}\|^2\leq O(\|f'_{l,i,1}\|^2+\xi^3m^{3/2}).
	\]
	Choosing $\xi=\max\left\lbrace 2\|f'_{l,i,2}\|_{\infty}, \Theta(\frac{\|f'_{l,i,1}\|^{2/3}}{m^{1/2}})\right\rbrace $, and recalling $\|f'_{l,i,1}\|\leq O(L^{3/2}\phi\log^{1/2}(1/\phi))$ and $\|f'_{l,i,2}\|_{\infty}\leq O(L\phi^{2/3}/m^{1/2}\log^{1/2}(1/\phi))$, we get $\|\bm{D}\|_0\leq O(mL\phi^{2/3}\log^{1/2}(1/\phi))$.	Choosing $\xi=2\|f'_{l,i,2}\|_{\infty}$, we get $\|\bm{D}'f^{(0)}_{l,i} \|\leq O(\phi L^{3/2}\log^{1/2}(1/\phi))$.
\end{proof}

\subsection{Proof of Lemma \ref{lma:gradientlowerboundlastlayer}}\label{sec:proofgradientlowerboundlastlayer}

\begin{lemma}[Lemma B.1 in \cite{Convergence_journal_Gu_zou2020gradient}]\label{lma:separateness}
	Assume $m>\Omega\left( L\log(NL)\right) $. For any $x'_i, x'_j\in \mathcal{X}_{\phi}$, $i,j\in [N], l\in[L]$, with probability at least $1-\exp\left( -O(m/L)\right) $ over the randomness of $\bm{W}^{(0)}$, it holds that
	$1/2 \leq \left\| h_{l}(x'_i)\right\| \leq 2$ and $
	\left\| h_{l}(x'_i)/\left\| h_{l}(x'_i)\right\|-h_{l}(x'_j)/\left\| h_{l}(x'_j)\right\| \right\| \geq \phi/2,
	$
	where $h_{l}(x'_i)$ is the output of the $l-$th layer at initialization.
\end{lemma}

Denote $b_i=h_{\bm{W}^{(0)},L-1}(x_i)$ and $\bar{b}_i=b_i/\|b_i\|$ for $x_i\in\mathcal{X}$, and $b'_i=h_{\bm{W}^{(0)},L-1}(x'_i)$ and $\bar{b}'_i=b'_i/\|b'_i\|$ for $x'_i\in\mathcal{X}_{\phi}$.  By Lemma~\ref{lma:separateness}, we have $\forall i\notin \mathcal{I}_{\phi,k}, \left\| \bar{b}_i-\bar{b}'_k\right\| \geq \phi/4$ . Moreover, by Lemma \ref{lma:forwardinputperturbation}, we have  $\forall i\in \mathcal{I}_{\phi,k}, \left\| b_i-b'_j\right\| \leq \widetilde{O}(L^{5/2}\phi\log^{1/2}(m))$.

Then we construct several sets for the vector $w\in\mathcal{R}^m$ subject to $\mathcal{N}(0,(2/m)\mathbf{I})$. Given $\bar{b}'_k$, we construct an orthogonal matrix $Q_k=[\bar{b}'_k, Q'_k]\in\mathcal{R}^{m\times m}$ and let $q_k=Q_k^\top w\sim\mathcal{N}(0,(2/m)\mathbf{I})$. In this way, the vector $w$ is decomposed as two orthogonal vector: $w=Q_kq_k=q_k^{(1)}\bar{b}_k'+Q_k'q_k'$ where $q_k^{(1)}$ is the first element of $q_k$. Letting $\gamma=\sqrt{2\pi}\phi/(32N\sqrt{m})$, we construct the set
\begin{equation}
\mathcal{W}_k=\left\lbrace w\in\mathcal{R}^d\mid |q_k^{(1)}|\leq \gamma, |\left\langle Q_k'q'_k,\bar{b}'_j\right\rangle| \geq 2\gamma, \forall j\neq k\right\rbrace,
\end{equation}
where $[q_k^{(1)},q'_k]=q_k$.

\begin{lemma}[Lemma C.1 in \cite{Convergence_Gu_zou2019improved}]\label{lma:probweightset}
	For any $\mathcal{W}_j$ and $\mathcal{W}_k$, $j\neq k$, we have
	$\mathcal{W}_j\bigcap\mathcal{W}_k=\emptyset$ and
	$
	\mathbb{P}(w\in\mathcal{W}_k)\geq \frac{\phi}{N32\sqrt{2e}}
	$.
\end{lemma}
\begin{lemma}\label{lma:C.2inConvrgenceGu}
	Let $f(w_j)=\sum_{k=1}^N\sum_{i\in\mathcal{I}_{\phi,k}}a_i\sigma'(\left\langle w_j,b_i\right\rangle )b_i$ where $w_j, j\in[m]$ is drawn from $\mathcal{N}(0,(2/m)\mathbf{I})$, $|a_i|\leq O\left( (\mu_i+\lambda_i)/\sqrt{d}\right) $. If for each smooth set $k$, there exists a subset $\mathcal{G}_{k,\alpha}\in[m]$ with size $\alpha m, \alpha\in(0,1)$ such that  $\forall i\in \mathcal{I}_{\phi,k}$, $\forall j\in \mathcal{G}_{k,\alpha}$, $\sigma'(\left\langle w_j,b_i\right\rangle )=\sigma'(\left\langle w_j,b_k'\right\rangle )$ and $\forall j\notin \mathcal{G}_{k,\alpha}$, $\left|\left\langle w_j,b_i\right\rangle\right|\geq \frac{3\sqrt{2\pi}\phi}{16N\sqrt{m}}$, we have for any $j\in \mathcal{G}_{k,\alpha}$,
	$
	\mathbb{P}\left( \left\| f(w_j)\right\|\geq |A_k|/4- M_k/\sqrt{d}\widetilde{O}(L^{5/2}\phi\log^{1/2}(m))\mid w_j\in \mathcal{W}_k\right)  \geq 1/2
	$ where $A_k=\sum_{i\in\mathcal{I}_{\phi,k} }a_i,k\in[N]$.
\end{lemma}
\begin{proof}
	For $j\in  \mathcal{G}_{k,\alpha}$, let $q_k=Q_k^\top w_j\sim\mathcal{N}(0,(2/m)\mathbf{I})$. Then we have $
	w_j=Q_kq_k=q_k^{(1)}\bar{b}_k'+Q_k'q_k'.
	$
	We decompose $f(w_j)$ as
	\begin{equation}\label{eqn:lmaB.4proof1}
	\begin{split}
	f(w_j)&=\sum_{i\in\mathcal{I}_{\phi,k}}a_i\sigma'(\left\langle {w},b_i\right\rangle )b_i+ \sum_{k'\neq k}\sum_{i\in\mathcal{I}_{\phi,k'}}a_i\sigma'(\left\langle {w},b_i\right\rangle )b_i\\
	&=\sum_{i\in\mathcal{I}_{\phi,k}}a_i\sigma'(\left\langle w,b_k'\right\rangle )b_i+ \sum_{k'\neq k}\sum_{i\in\mathcal{I}_{\phi,k'}}a_i\sigma'(\left\langle w,b_i\right\rangle )b_i\\
	&=\sum_{i\in\mathcal{I}_{\phi,k}}a_i\sigma'( q_k^{(1)} )b_i+ \sum_{k'\neq k}\sum_{i\in\mathcal{I}_{\phi,k'}}a_i\sigma'(\left\langle w,b_i\right\rangle )b_i,
	\end{split}
	\end{equation}
	where the second equality holds by the assumption $\forall j\in \mathcal{G}_{k,\alpha}$, $\sigma'(\left\langle w_j,b_i\right\rangle )=\sigma'(\left\langle w_j,b_k'\right\rangle )$.
	
	Then for the second term of \eqref{eqn:lmaB.4proof1}, if $j\in \mathcal{G}_{k',\alpha}$, we have for $i\in\mathcal{I}_{\phi,k'}$,
	$\sigma'(\left\langle w_j,b_i\right\rangle )=\sigma'(\left\langle w_j,b_{k'}'\right\rangle )$ and thus
	\[
	\begin{split}
	&\sum_{k'\neq k}\sum_{i\in\mathcal{I}_{\phi,k'}}a_i\sigma'(\left\langle w,b_i\right\rangle )b_i=\sum_{k'\neq k}\sum_{i\in\mathcal{I}_{\phi,k'}}a_i\sigma'(\left\langle w,b_{k'}'\right\rangle )b_i\\
	=& \sum_{k'\neq k}\sum_{i\in\mathcal{I}_{\phi,k'}}a_i\sigma'(q_k^{(1)}\left\langle \bar{b}_k',b_{k'}'\right\rangle +\left\langle Q_k'q_k',b_{k'}'\right\rangle )b_i\\
	=&\sum_{k'\neq k}\sum_{i\in\mathcal{I}_{\phi,k'}}a_i\sigma'(\left\langle Q_k'q_k',b_{k'}'\right\rangle )b_i
	\end{split}
	\]
	where the last equality holds by the condition $w_j\in\mathcal{W}_k$ such that for $k'\neq k$, $|\left\langle Q_k'q_k',b_{k'}'\right\rangle|\geq 2\gamma\|b_{k'}'\|\geq |q_k^{(1)}|\|b_{k'}'\|\geq |q_k^{(1)}<\bar{b}_k',b_{k'}'>|$ and thus the sign is determined by $\left\langle Q_k'q_k',b_{k'}'\right\rangle$.
	Therefore, if $j\in \mathcal{G}_{k',\alpha}$, we can write \eqref{eqn:lmaB.4proof1}  as
	\begin{equation}\label{eqn:lmaB.4proof2}
	\begin{split}
	f(w_j)
	=\sum_{i\in\mathcal{I}_{\phi,k}}a_i\sigma'( q_k^{(1)} )b_i+ \sum_{k'\neq k}\sum_{i\in\mathcal{I}_{\phi,k'}}a_i\sigma'(\left\langle Q_k'q_k',b_{k'}'\right\rangle )b_i.
	\end{split}
	\end{equation}
	
	In the other case with $j\notin \mathcal{G}_{k',\alpha}$, by assumption $\forall i\in\mathcal{I}_{\phi,k'},\left|\left\langle w_j,b_i\right\rangle\right|\geq \frac{3\sqrt{2\pi}\phi}{16N\sqrt{m}}=6\gamma$, we have with probability at least $1-\exp\left( -O(m/L)\right) $, $\left|\left\langle w_j,\bar{b}_i\right\rangle\right|=\left|\left\langle w_j,b_i\right\rangle\right|\frac{1}{\|b_i\|}\geq 3\gamma$ by Lemma \ref{lma:separateness}. Then $\forall i\in\mathcal{I}_{\phi,k'}$, we have
	\[
	\begin{split}
	&|\left\langle Q_k'q_k',\bar{b}_{i}\right\rangle|=	|\left\langle w_j,\bar{b}_{i}\right\rangle-\left\langle q_k^{(1)}\bar{b}_k',\bar{b}_{i}\right\rangle|\\
	\geq & |\left\langle w_j,\bar{b}_{i}\right\rangle|-|\left\langle q_k^{(1)}\bar{b}_k',\bar{b}_{i}\right\rangle|
	\geq|\left\langle w_j,\bar{b}_{i}\right\rangle|-|q_k^{(1)}|\\
	\geq & |\left\langle w_j,\bar{b}_{i}\right\rangle|-\gamma\geq 2\gamma\\
	\geq &2\gamma\|b_{k'}'\|\geq |q_k^{(1)}|\|b_{k'}'\|\geq |q_k^{(1)}<\bar{b}_k',b_{k'}'>|,
	\end{split}
	\]
	where the first inequality comes from triangle inequality, the second inequality holds by $|\left\langle \bar{b}_k',\bar{b}_{i}\right\rangle|\leq 1$, and the last inequality holds by the condition $w_j\in\mathcal{W}_k$.
	Therefore if $j\notin \mathcal{G}_{k',\alpha}$, we can write the second term in \eqref{eqn:lmaB.4proof1} as
	\[
	\begin{split}
	&\sum_{k'\neq k}\sum_{i\in\mathcal{I}_{\phi,k'}}a_i\sigma'(\left\langle w,b_i\right\rangle )b_i\\
	=& \sum_{k'\neq k}\sum_{i\in\mathcal{I}_{\phi,k'}}a_i\sigma'(q_k^{(1)}\left\langle \bar{b}_k',b_{i}\right\rangle +\left\langle Q_k'q_k',b_{i}\right\rangle )b_i\\
	=&\sum_{k'\neq k}\sum_{i\in\mathcal{I}_{\phi,k'}}a_i\sigma'(\left\langle Q_k'q_k',b_{i}\right\rangle )b_i
	\end{split}
	\]	
	Therefore, if $j\notin \mathcal{G}_{k',\alpha}$, we can write \eqref{eqn:lmaB.4proof1}  as
	\begin{equation}\label{eqn:lmaB.4proof3}
	\begin{split}
	f(w_j)
	=\sum_{i\in\mathcal{I}_{\phi,k}}a_i\sigma'( q_k^{(1)} )b_i+ \sum_{k'\neq k}\sum_{i\in\mathcal{I}_{\phi,k'}}a_i\sigma'(\left\langle Q_k'q_k',b_{i}\right\rangle )b_i.
	\end{split}
	\end{equation}
	
	Note that \eqref{eqn:lmaB.4proof2} and \eqref{eqn:lmaB.4proof3} are different only in terms of whether $b_{k'}'$ or $b_{i}, i\in\mathcal{I}_{\phi,k'}$ determines the second term, but for both of them, the second term does not rely on $q_k^{(1)}$. We thus proceed as follows.
	
	Since $q_k^{(1)}>0$ and $q_k^{(1)}<0$ occurs with equal probability conditioned on the event $w\in\mathcal{W}_k$, we have
	\[
	\mathbb{P}\left[\|f(w_j)\|_2\geq \inf_{q_1>0,q_2<0}\max\left\lbrace \|f\left( q_1\bar{b}_k'+Q_k'q_k'\right)\|, \|f\left( q_2\bar{b}_k'+Q_k'q_k'\right) \|\right\rbrace \mid w\in\mathcal{W}_k \right] \geq 1/2.
	\]
	Thus, with probability at least $1/2$ conditioned on the event $w\in\mathcal{W}_k$, we have
	\[
	\begin{split}
	\|f(w_j)\|&\geq \inf_{q_1>0,q_2<0}\max\left\lbrace \|f\left( q_1\bar{b}_k'+Q_k'q_k'\right)\|, \|f\left( q_2\bar{b}_k'+Q_k'q_k'\right) \|\right\rbrace\\
	&\geq \inf_{q_1>0,q_2<0}\left\| f\left( q_1\bar{b}_k'+Q_k'q_k'\right)-f\left( q_2\bar{b}_k'+Q_k'q_k'\right)\right\|/2 \\
	&=\| \sum_{i\in\mathcal{I}_{\phi,k}}a_ib_i\|
	\end{split}
	\]
	Since $|a_i|\leq O\left( (\mu_i+\lambda_i)/\sqrt{d}\right) $ and $\|b_i-b'_k\|\leq \widetilde{O}(L^{5/2}\phi\log^{1/2}(m))$ for $i\in \mathcal{I}_{\phi,k}$, we have, $\left|\| \sum_{i\in\mathcal{I}_{\phi,k}}a_ib_i\|-\| \sum_{i\in\mathcal{I}_{\phi,k}}a_ib_k'\|\right|\leq \| \sum_{i\in\mathcal{I}_{\phi,k}}a_i(b_i-b_k')\|\leq \sum_{i\in\mathcal{I}_{\phi,k}}|a_i|\|b_i-b_k'\|\leq M_k\sqrt{d}\widetilde{O}(L^{5/2}\phi\log^{1/2}(m))$. Thus,
	\[
	\begin{split}
	\| \sum_{i\in\mathcal{I}_{\phi,k}}a_ib_i\| &\geq\| \sum_{i\in\mathcal{I}_{\phi,k}}a_ib_k'\|-\left|\| \sum_{i\in\mathcal{I}_{\phi,k}}a_ib_i\|-\| \sum_{i\in\mathcal{I}_{\phi,k}}a_ib_k'\|\right|\\
	&\geq \| \sum_{i\in\mathcal{I}_{\phi,k}}a_ib_k'\|-  M_k\sqrt{d}\widetilde{O}(L^{5/2}\phi\log^{1/2}(m))\\
	&\geq |A_k|/4- M_k/\sqrt{d}\widetilde{O}(L^{5/2}\phi\log^{1/2}(m)),
	\end{split}
	\]
	where the last inequality follows from Lemma \ref{lma:separateness}. The proof is completed.
\end{proof}

\begin{lemma}[Bernstein inequality]\label{lma:Bernstein inequalities}
	Let $X_1,\cdots, X_n$ be independent zero-mean random variables. If $|X_i|\leq 1$ almost surely for all $i$, then $\forall t>0$,
	\[
	\mathbb{P}\left( \sum_{i=1}^nX_i\geq t\right) \leq \exp\left( -\frac{1}{2}t^2/\left(\sum_{i=1}^n\mathbb{E}[X_i^2]+\frac{1}{3}t \right) \right).
	\]
\end{lemma}

\textbf{Proof of Lemma \ref{lma:gradientlowerboundlastlayer}}
\begin{proof}
	Denote $b_i=h^{(0)}_{L-1}(x_i)$, so $[h^{(0)}_{L}(x_i)]_j=\left\langle w_j,b_i\right\rangle$. For any fixed $[u_1,\cdots, u_{n'}]$, denote $a_i(v_j)=\left\langle u_i,v_j\right\rangle , i\in[n'], j\in[m]$  and $A_k(v_j)=\sum_{i\in\mathcal{I}_{\phi,k} }a_i(v_j)$.
	Let $f(v_j,w_j)=\sum_{k=1}^N\sum_{i\in\mathcal{I}_{\phi,k} }a_i(v_j)\sigma'(\left\langle w_j,b_i\right\rangle )b_i$. Define the event for $k\in[N]$
	\[
	\mathcal{E}_k=\left\lbrace j\in\mathcal{G}_{k,\alpha}:w_j\in\mathcal{W}_k,\left\| f(v_j,w_j)\right\|\geq \|\sum_{i\in\mathcal{I}_{\phi,k}}u_i\|/(4\sqrt{d})-   M_{k} d^{-1/2}\widetilde{O}(L^{5/2}\phi\log^{1/2}(m)) \right\rbrace.
	\]
	
	Since $v_j\sim\mathcal{N}(0, (1/d)\mathbf{I})$, we have $A_k(v_j)=\left\langle \sum_{i\in\mathcal{I}_{\phi,k}}u_{i},v_j\right\rangle \sim\mathcal{N}(0,\|\sum_{i\in\mathcal{I}_{\phi,k}}u_i\|^2/d)$. Thus, we have
	\[
	\mathbb{P}\left(\|\sum_{i\in\mathcal{I}_{\phi,k}}u_i\|/\sqrt{d} \leq |A_k(v_j)|\leq 2\|\sum_{i\in\mathcal{I}_{\phi,k}}u_i\|/\sqrt{d} \right) \geq 1/4.
	\]
	Note that when $|A_k(v_j)|\leq 2\|\sum_{i\in\mathcal{I}_{\phi,k}}u_i\|/\sqrt{d}$, we have $\forall i\in\mathcal{I}_{\phi,k}$, $|a_i(v_j)|\leq |A_k(v_j)|/|\mathcal{I}_{\phi,k}|-\widetilde{O}(L^{5/2}\phi\log^{1/2}(m))\leq 2\|\sum_{i\in\mathcal{I}_{\phi,k}}u_i\|/\sqrt{d}/|\mathcal{I}_{\phi,k}|-\widetilde{O}(L^{5/2}\phi\log^{1/2}(m))\leq 2(\mu_i+\lambda_i)/\sqrt{d}-3\widetilde{O}(L^{5/2}\phi\log^{1/2}(m))\leq  3(\mu_i+\lambda_i)\sqrt{d}$ when $\phi$ is small enough, so the condition about $|a_i(v_j)|$ in Lemma \ref{lma:C.2inConvrgenceGu} is met.
	
	Since Assumption  \ref{asp:smoothset}  is satisfied, we have for each smooth set $k$, there exists a subset $\mathcal{G}_{k,\alpha}\in[m]$ with size $\alpha m, \alpha\in(0,1)$ such that  $\forall i\in \mathcal{I}_{\phi,k}$, $\forall j\in \mathcal{G}_{k,\alpha}$, $\sigma'(\left\langle w_j,b_i\right\rangle )=\sigma'(\left\langle w_j,b_k'\right\rangle )$ and $\forall j\notin \mathcal{G}_{k,\alpha}$, $\left|\left\langle w_j,b_i\right\rangle\right|\geq \frac{3\sqrt{2\pi}\phi}{16N\sqrt{m}}$, so the assumption in Lemma  \ref{lma:C.2inConvrgenceGu} is satisfied. Then by Lemma \ref{lma:probweightset}, Lemma \ref{lma:C.2inConvrgenceGu} and the fact that $w_j$ and $v_j$ are independent, we have for $j\in \mathcal{G}_{k,\alpha}$
	\[
	\begin{split}
	&\mathbb{P}(j\in\mathcal{E}_k)=\mathbb{P}\left\lbrace \left\| f(v_j,w_j)\right\|\geq|A_k(v_j)|/4-  M_{k} d^{-1/2}\widetilde{O}(L^{5/2}\phi\log^{1/2}(m))\mid w_j\in\mathcal{W}_k\right\rbrace\\
	&\quad\cdot\mathbb{P}\left\lbrace w_j\in\mathcal{W}_k\right\rbrace	\mathbb{P}\left(\|\sum_{i\in\mathcal{I}_{\phi,k}}u_i\|/\sqrt{d} \leq |A_k(v_j)|\leq 2\|\sum_{i\in\mathcal{I}_{\phi,k}}u_i\|/\sqrt{d} \right) \geq \frac{\phi}{N256\sqrt{2e}}=p_{\phi}.
	\end{split}
	\] and $\mathcal{E}_{k_1}\bigcap \mathcal{E}_{k_2}=\emptyset$ for any $k_1\neq k_2$.
	
	For smooth set $k$, denote Bernoulli random variables $\mathds{1}(j\notin\mathcal{E}_k)$ for $j\in\mathcal{G}_{k,\alpha}$. Then we have $\mathbb{E}\left[ \mathds{1}(j\notin\mathcal{E}_k)\right] =1-p_{\phi}$ and $\mathrm{var}\left[ \mathds{1}(j\notin\mathcal{E}_k)\right] =p_{\phi}(1-p_{\phi})$. By Bernstein inequality in Lemma \ref{lma:Bernstein inequalities} for random variables $\mathds{1}(j\notin\mathcal{E}_k)-(1-p_{\phi}), j\in[\alpha m]$, it holds that
	\[
	\mathbb{P}\left( \sum_{j\in \mathcal{G}_{k,\alpha}}\mathds{1}(j\notin\mathcal{E}_k)-\alpha m(1-p_{\phi})\geq \frac{\alpha mp_{\phi}}{2}\right) \leq \exp\left( -\frac{\frac{1}{4}\alpha mp_{\phi}}{1-p_{\phi}+\frac{1}{6} }\right)\leq \exp\left( -\frac{3}{14}\alpha mp_{\phi}\right).
	\]
	Thus, by union bounds, with probability at least $1-O(N)\exp\left( -O\left(\alpha m\phi/N\right) \right)$, we have for any $k\in[N]$, $\sum_{j\in \mathcal{G}_{k,\alpha}}\mathds{1}(j\notin\mathcal{E}_k)\leq \alpha m-\alpha mp_{\phi}/2$ and
	\begin{equation}\label{eqn:gradientlowerboundproof1}
	\left|\mathcal{G}_{k,\alpha}\bigcap \mathcal{E}_k  \right|=\sum_{j\in \mathcal{G}_{k,\alpha}}\mathds{1}(j\in\mathcal{E}_k)=\alpha m-\sum_{j\in \mathcal{G}_{k,\alpha}}\mathds{1}(j\notin\mathcal{E}_k)\geq \alpha mp_{\phi}/2.
	\end{equation}
	
	Therefore, with probability at least $1-O( \phi)$, it holds that
	\begin{equation}
	\begin{split}
	&	\sum_{j=1}^m\|f(v_j,w_j)\|^2\geq \sum_{j=1}^m \left\|f(v_j,w_j) \right\|^2\sum_{k=1}^N\mathds{1}(j\in\mathcal{E}_k)=\sum_{k=1}^N\sum_{j=1}^m \left\|f(v_j,w_j) \right\|^2\mathds{1}(j\in\mathcal{E}_k)\\
	\geq &\sum_{k=1}^N\sum_{j\in\mathcal{G}_{k,\alpha}} \left\|f(v_j,w_j) \right\|^2\mathds{1}(j\in\mathcal{E}_k)=\sum_{k=1}^N\sum_{j\in\mathcal{G}_{k,\alpha},j\in \mathcal{E}_k} \left\|f(v_j,w_j) \right\|^2\\
	\geq& \sum_{k=1}^N \sum_{j\in\mathcal{G}_{k,\alpha},j\in \mathcal{E}_k}\left( \|\sum_{i\in\mathcal{I}_{\phi,k}}u_i\|/(4\sqrt{d})-   M_{k} d^{-1/2}\widetilde{O}(L^{5/2}\phi\log^{1/2}(m))\right)^2 \\
	\geq& \sum_{k=1}^N\sum_{j\in\mathcal{G}_{k,\alpha},j\in \mathcal{E}_k} \left(\frac{1}{16d}\|\sum_{i\in\mathcal{I}_{\phi,k}}u_i\|^2 -M^2_{k} d^{-1}\widetilde{O}(L^{5/2}\phi\log^{1/2}(m))\right)  \\
	= &\sum_{k=1}^N \left(\frac{1}{16d}\|\sum_{i\in\mathcal{I}_{\phi,k}}u_i\|^2 -M^2_{k} d^{-1}\widetilde{O}(L^{5/2}\phi\log^{1/2}(m))\right) \left|\mathcal{G}_{k,\alpha}\bigcap \mathcal{E}_k  \right|\\
	\geq& \frac{\alpha mp_{\phi}}{2}\sum_{k=1}^N \left(\frac{1}{16d}\|\sum_{i\in\mathcal{I}_{\phi,k}}u_i\|^2 -M^2_{k} d^{-1}\widetilde{O}(L^{5/2}\phi\log^{1/2}(m))\right)\\
	\geq&  \Omega\left( \frac{\alpha m\phi}{Nd}\right)\left(  \sum_{k=1}^N\|\sum_{i\in\mathcal{I}_{\phi,k}}u_i\|^2-\widetilde{O}(L^{5/2}\phi\log^{1/2}(m))\right) ,
	\end{split}	
	\end{equation}
	where the first inequality comes from the fact that $\mathcal{E}_{k_1}\bigcap \mathcal{E}_{k_2}=\emptyset$ such that $\sum_{k=1}^N\mathds{1}(j\in\mathcal{E}_k)\leq 1$ and the second inequality comes from the fact that $\mathcal{G}_{k,\alpha}\in [m]$, and the third inequality holds by the definition of event $\mathcal{E}_k$, and the forth inequality comes from the fact that $(a-b)^2\geq a^2-2ab$ and $\|\sum_{i\in\mathcal{I}_{\phi,k}}u_i\|/(4\sqrt{d})\leq M_kd^{-1/2}/4$, and the fifth inequality comes from \eqref{eqn:gradientlowerboundproof1} and the last inequality holds by the fact that $\sum_{k=1}^NM_k^2\leq (\sum_{k=1}^N M_k)^2=1$.
\end{proof} 

%% file: Application.tex
\section{Preliminaries on Informed Machine Learning}\label{sec:preliminary}

Informed machine learning is rapidly emerging as a broad paradigm
that incorporates
domain knowledge, either directly or indirectly, to augment
the purely data-driven approach and
better accomplish a machine learning task. 
We provide a summary of how domain knowledge is integrated with machine learning \cite{Informed_ML_von19}.

\begin{itemize}
	\item \textit{Training Dataset}. A straightforward
	approach to utilizing domain knowledge is to
	generate (sometimes synthetic) data and enlarge the otherwise limited training dataset.
	For example, based on the simple knowledge of image invariance, cropping\cite{Low_shot_augmentation_gao2018low}, scaling\cite{Fine_grained_meta_zhang2018fine}, flipping\cite{One-shot_domain_translation_benaim2018one} and many other image pre-processing methods have been used to augment the training data for image classification tasks. As another example, in reinforcement learning (e.g.,
	robot control and autonomous driving) where
	initial pre-training is crucial to avoid arbitrarily bad decisions in the real world,
	simulated environments can be built based on domain knowledge,
	providing simulations or demonstrations to generate training data \cite{RL_imperfect_demonstraion_gao2018reinforcement,DQN_demonstration_hester2017deep}.
	Additionally, generative models constructed based on specific knowledge have been shown useful for increasing training data to improve model performance and robustness \cite{Low_shot_augmentation_gao2018low,DNN_Book_Goodfellow-et-al-2016}.

	\item \textit{Hypothesis Set.}
	The goal of a machine learning task is to search for an optimal hypothesis that correctly expresses the relationships between input and output. To reduce the training complexity, the
	target hypothesis set (decided by, e.g., different neural architectures) should contain the optimal hypothesis and preferably be small enough. Thus,
	domain knowledge can be employed for hypothesis set selection. For example, \cite{ModuleNet_chen2020modulenet} makes use of the prior knowledge from the existing neural architectures to design new architectures (and hence, new hypothesis sets) for DNNs. As implicit domain knowledge, long short-term memory recurrent neural networks are commonly used for
	time series prediction \cite{DNN_Book_Goodfellow-et-al-2016}.
	Also, the structure of a knowledge graph helps to determine the hypothesis set of graph learning \cite{knowledge_graph_image_marino2016more,Relation_graph_battaglia2018relational},
	while \cite{Knowledge_based_ANN_towell1994knowledge} maps the domain knowledge represented in propositional logic into neural networks.
	
	\item \textit{Model Training.}
	Domain knowledge can be integrated, either implicitly or explicitly, with the model training procedure in various ways. First, domain knowledge can assist with the initialization of training. For example,
	\cite{Initialization_genetic_ramsey1993case} provides a case-based method to initialize genetic algorithms (i.e., generating the initial population based on
	different cases),
	while  \cite{knowledge-based_network_initialization_husken2000fast,Initialization_label_cooccurence_kurata2016improved,Initialization_decision_tree_humbird2018deep} initialize neural network training with various domain knowledge such as label co-occurrence and decision trees.
	Second, domain knowledge can be used to better tune the hyper-parameters \cite{Collaborative_hyperparameter_tuning_bardenet2013collaborative,hyperparameter_datasets_van2018hyperparameter,smartml_maher2019smartml,tuning_knowledge_graph_bamler2020augmenting}. In \cite{Collaborative_hyperparameter_tuning_bardenet2013collaborative},
	implicit knowledge from previous training is incorporated to improve  hyper-parameter tuning,
	and \cite{hyperparameter_datasets_van2018hyperparameter} extracts knowledge from multiple datasets to determine the most important hyper-parameters. 
	In addition, a more explicit way to integrate domain knowledge is to directly
	modify the training objective function (i.e., risk function) based
	on rigorous characterization of the model output \cite{Informed_ML_von19}.
	For example, in \cite{Incorporating_domain_knowledge_muralidhar2018incorporating}, the knowledge of constraints is incorporated into neural networks expressing the knowledge based loss by the ReLu function. For another example, when learning to optimally
	schedule transmissions for rate maximization in multi-user wireless networks,
	the communication channel capacity can be added as domain knowledge
	to the standard label-based loss to guide scheduling decisions;
	in physics, the analytical expression of a partial differential equation
	can be utilized as domain knowledge on top of labeled data to better learn
	the solution to the equation given different inputs; more examples are shown in Section~\ref{sec:formulation_example}.
	Such integration of explicit and rigorous domain knowledge can significantly
	benefit machine learning tasks (e.g., fewer labels needed than otherwise).
	Thus, it is crucial
	and being actively studied in informed machine learning \cite{Informed_ML_von19,InformedML_PhysicsModeling_Survey_arXiv_2020_DBLP:journals/corr/abs-2003-04919},
	which is also the focus of our work. Note that using domain knowlege
to generate pseudo labeled data to augment the training dataset
is a special case of integrating  domain knowledge
into the training risk function (i.e., the knowledge-based risk is the same
as the data-based risk, except that its labels are generated based on domain knowledge).

	\item \textit{Final Hypothesis.} Domain knowledge
	can also be used for consistency check on the final learnt hypothesis or model \cite{Informed_ML_von19}. For example,  \cite{Physics_pgnn_karpatne2017physics}
	employs physics domain knowledge to construct the final model,
	\cite{pfrommer2018optimisation} builds simulators to validate results of learned model,
	and \cite{fang2017object} leverages semantic consistency is used to refine the predicted probabilities.
\end{itemize}

\section{Application Examples}\label{sec:formulation_example}

We now present a few
 application examples to explain domain knowledge-informed DNNs.

\subsection{Learning for resource management in communications networks}\label{sec:resourcemanagment}
Optimizing resource management is crucial to improve the system performance in communications networks \cite{Power_control_chiang2008power,Goldsmith_WirelessCommunications_2005,InformedML_Wireless_Model_AI_Both_Survey_TCOM_2019_8742579}. Well-known examples include
 power allocation \cite{resource_allocation_hong2014signal,Power_control_chiang2008power,power_control_liang2019towards}, link scheduling \cite{Link_scheduling_gore2010link,Spatial_wireless_scheduling_cui2019spatial}, antenna or beam selection \cite{antenna_selection_sanayei2004antenna,ML_beam_selection_klautau20185g}, among others. While many of the problems were studied
 using theoretical model-based approaches in the past,
machine learning has been increasingly employed, in view of the rapidly growing complexity
of communications technologies that theoretical models are often incapable of capturing accurately \cite{InformedML_Wireless_Model_AI_Both_Survey_TCOM_2019_8742579}.
Let us take power allocation in multi-user wireless interference networks
as an example. The recent work \cite{learning_interference_management_sun2018learning}
uses a pure data-driven approach for power allocation to maximize
the sum rate: a labeled dataset containing channel state information (CSI)
and the corresponding power allocation decisions is collected in advance,
 and a neural network is trained
to learn the optimal power allocation.
On the other hand, Shannon-based transmission rate has been extensively as an analytical objective
function to optimize power allocation, and \cite{power_control_liang2019towards}
exploits this domain knowledge to train an ensemble of neural networks
that directly learn the optimal power allocation for Shannon rate maximization.

The data-driven approach \cite{learning_interference_management_sun2018learning}
can maximize the practically achievable rate (if labels
are collected from real systems),
but is significantly constrained by the limited amount of training samples.
Meanwhile, the knowledge-based approach  \cite{power_control_liang2019towards}
can utilize a large number of input samples (at the expense of higher training complexity), but the resulting power allocation decisions
may not maximize the sum rate in real systems.
The reason is that the Shannon formula for interference channels,
albeit commonly used
for analysis, only represents
an approximation of the
achievable rate which is subject to finite channel code lengths and modulation schemes \cite{Goldsmith_WirelessCommunications_2005}. In other words,
even
an oracle DNN that minimizes this knowledge-based loss may not maximize the achievable rate in practice.

To reap the benefit of both labeled data
and domain knowledge,
informed machine learning can be adopted, resulting in a new informed loss as follows:
\begin{equation}
\min_{h\in\mathcal{H}}(1-\gamma)\left\{\frac{1}{n}\sum_{(x,y)\in S} \left[h(x)-y\right]^2\right\}
+ \gamma\cdot \left\{- \frac{1}{\tilde{n}}\sum_{x\in\tilde{S}_X} \mathrm{Shannon\_rate}[h(x)]+\mathrm{constant}\right\},
\end{equation}
where $x$ is the input (e.g., channel state information), $h(x)$
is the learned power allocation given $x$,
the two loss terms represent label-based loss and knowledge-based loss,
and $n$ and $\tilde{n}$ are the numbers
of labeled data samples and (possibly unlabled) knowledge samples, respectively.
The detailed Shannon formula for wireless networks
can be found in \cite{power_control_liang2019towards,Goldsmith_WirelessCommunications_2005}.

\subsection{Image classification based on semantic knowledge}
Typical image classifiers rely on labeled training data, but labels
can be difficult and expensive to collect in practice \cite{DNN_Book_Goodfellow-et-al-2016}.
As a result, few-shot learning \cite{Generalizing_few_shot_wang2020,Relation_net_sung2018learning,Few-shot_graph_nn_garcia2017few} that only needs a small number of labeled samples has been
proposed. Informed machine learning under our consideration can be viewed as few-shot learning. Concretely, semantic knowledge
formulated as the first-order logic clauses/sentences \cite{Semantic_Loss_smbolic_xu2018semantic,prior_knowledge_diligenti2017integrating}
can be incorporated to improve learning performance given limited labeled samples.
An example logic clause is ``if it is an animal and has wings, then it is a bird''.  By a logic clause $K$, a knowledge-based loss can be defined as $F_K(h(x), g(x))$ for an (possibly unlabled) input
image $x$ and a certain logic clause $g(x)$  that the output class $h(x)$ needs to satisfy.
Then, combining the standard label-based loss with knowledge-based loss,
the model performance can be improved by minimizing
the informed loss
Eqn.~\eqref{eqn:informlossnn} given limited labeled samples.

\subsection{Learning to solve PDEs in scientific and engineering fields}
Partial differential equations (PDEs) are classic problems in many
scientific and engineering fields, such as physics and mechanical engineering,
but are notoriously difficult to solve in most practical settings \cite{Physics-informed_learning,sciML_baker2019workshop}.
In recent years, physics knowledge-informed machine learning has been suggested as a
promising approach to augment or even replace classic PDE solution approaches  \cite{InformedML_HumanKnowledge_Survey_Cell_Umich_2020_DENG2020101656,InformedML_PhysicsModeling_Survey_arXiv_2020_DBLP:journals/corr/abs-2003-04919,PDE_ANN_khoo2017solving,MLapproximation_pde_beck2019machine,PINN_raissi2019physics,DeepXDE_lu2019deepxde}.
For example,
\cite{PINN_raissi2019physics,DeepXDE_lu2019deepxde} proposes a
physics-informed neural network (PINN) to solve  PDEs by minimizing the PDE residual and penalties of boundary/initial
conditions, which correspond to the knowledge-based loss
$F_K(h, g(h))$ in our framework. Additionally, we
can combine the knowledge-based loss with labeled-based loss,
achieving faster convergence and better performances in practice (especially when the PDE-based knowledge
does not perfectly represent the real physical world).
Take magnetic field strength estimation for magnetic materials as an example. If a few measured magnetic field strengths are provided as labels combined with the knowledge of Maxwell equations, the model trained by minimizing the informed loss can perform better in the real world. The measured labels can partly correct the imperfectness of physics knowledge, while the knowledge can improve the generalization in the presence of limited labels.

\subsection{Knowledge distillation and transfer}
Knowledge distillation \cite{Knowledge_Distillation_hinton2015distilling,self_distillation_furlanello2018born,Towards_understanding_distillation_phuong2019towards,ensemble_distillation_allen2020towards} is an important technique to transfer prior knowledge from a pre-trained neural network
(a.k.a. teacher network) to another network (a.k.a. student network), with the same or different architectures. Typically, given an (possibly unlabled) input, knowledge distillation is performed by matching the output of
the student network with the output of the teacher network. In addition,
labeled samples can also be included to introduce a label-based loss.
Thus,
by formulating $g(X)$ as the output of the teacher network, knowledge distillation
can be viewed as a particular instance of informed machine learning, where the knowledge comes from a teacher network and is usually assumed to be perfect. 

%% file: Experiements.tex
\section{Numerical Results}\label{appendix:simulation}

We consider two specific applications ---
{learning a multi-dimensional Bohachevsky function and learning
to manage wireless spectrum}.

\subsection{Settings of Learning with Constraint Knowledge in Section \ref{sec:simulation}}\label{appendix:simulation_setting}

We consider an informed DNN with domain knowledge
in the form of constraints to learn a Bohachevsky function.
The learning task is to learn a relationship $y(x)$. The learner is provided with a dataset with labeled samples $S_z=\left\lbrace (x_i,z_i),i\in[n_z]\right\rbrace$, having possibly noisy labels
\[
z_i=y(x_i)+n_i, n_i\sim \mathcal{N}(0, \sigma_z^2),
\]
and an unlabeled dataset $S_g=\left\lbrace (x_i),i\in[n_g]\right\rbrace$. Additionally, the learner is informed with the constraint knowledge, which includes an upper bound $g_{\mathrm{ub}}(x)$ and an lower bound $g_{\mathrm{lb}}(x)$ on the true label corresponding to input $x$, i.e. $g_{\mathrm{lb}}(x) \leq y(x)\leq g_{\mathrm{ub}}(x)$. A neural network $h_{\bm{W}}(x)$ is used to learn the relationship $y(x)$,
and  the metric of interest is the mean square error (MSE) of the network output $h_{\bm{W}}(x)$ with respect to the true label $y(x)$ on a test dataset $S_t$, which is expressed as
\[
\hat{R}_{S_t}(h_{\bm{W}})=\frac{1}{2|S_t|}\sum_{(x_i,y_i)\in S_t}\mathrm{mse}\left(h_{\bm{W}}(x_i),y_i \right),
\]
where $\mathrm{mse}\left(h_{\bm{W}}(x_i),y_i \right)=\left(h_{\bm{W}}(x_i)-y_i \right)^2 $ with $y_i$ as the true test label with respect to $x_i$.
Assume that the relationship to be learned is governed by a multi-dimensional Bohachevsky function
\[
y(x)=x\bm{A}\bm{A}^\top x^\top-c\cos\left( a^\top x\right)+c,
\]
where $\bm{A}$ is a $b\times b$ matrix, $a$ is a $b$-dimensional vector and $c$ is a constant. The learner has no access to the values of these parameters or the exact form of the relationship, but is empowered with the constraint knowledge in the form of an upper
bound model
\[
g_{\mathrm{ub}}(x)=x\bm{A}\bm{A}^\top x^\top+ub
\] with $ub\geq 2c$, and an lower bound model
\[
g_{\mathrm{lb}}(x)=x\bm{A}\bm{A}^\top x^\top+lb.
\] with $lb\leq 0$.
While it is not strongly convex and hence deviates from the assumptions in our theoretical
analysis, we use ReLU  as
the knowledge-based risk function, i.e., the knowledge-based risk is written as
\[
r_{\mathrm{K}}(h_{\bm{W}}(x))=\mathrm{relu}\left(h_{\bm{W}}(x)-g_{\mathrm{ub}}(x) \right)+\mathrm{relu}\left(g_{\mathrm{lb}}(x)-h_{\bm{W}}(x) \right).
\]
And the label supervised risk given a sample pair $(x,z)$ is $r(h_{\bm{W}}(x))=\mathrm{mse}\left(h_{\bm{W}}(x),z\right)$.

 To show the performance under different levels of imperfectness, we consider labels with different noise variances and different knowledge-informed constraints. For training, the labeled dataset $S_z$ contains $n_z\in\{200,400\}$ labeled samples with label noise variance $\sigma_z^2\in \{0,0.1\}$, and the unlabeled dataset $S_g$ for the knowledge risk contains $n_g=1000$ input samples. The  parameters for knowledge-informed constraint models include $lb=0$ and $ub\in\{0.6,0.8\}$. Naturally, the higher variance $\sigma_z^2$,
  the worse label quality; and the greater $ub$, the worse knowledge quality. The test dataset $S_t$ contains 1000 samples with labels calculated as $y_i=y(x_i),x_i\in S_t$.

For training,
 we use a neural network with two hidden layers, each having 2048 neurons and ReLU activations. Note that for the large network width needed for analysis to gain insights
 is not necessary in practice.
 The network is initialized based on Algorithm~\ref{appendix:training_algorithm}. The training procedure is performed by Adam optimizer for 3000 steps with batch size 100.  The learning rate is set as $10^{-6}$ for the first 2000 steps, $5\times10^{-5}$ for the following 500 steps, and $10^{-5}$ for the remaining 500 steps. We run the network training with 10 random seeds.
 We run the simulations on a HPC cluster with GPUs of type P100.

\subsection{Learning for Resource Management in Wireless Networks}\label{appendix:wireless}

We apply an informed DNN to the
problem of learning for resource management in wireless networks
--- wireless link scheduling in interference channels.
 We first describe problem setup, then present our method by informed
 DNN, and finally show the experiment results.

\subsubsection{Problem Setup}

Link scheduling is a classic and important problem in wireless interference
channels, with the objective of maximizing the sum throughput of wireless links.
Consider a time-slotted wireless network consisting of a transmitter-receiver set $\mathcal{U}=\{1,2,\cdots,N\}$ with $N$ links (i.e., transmitter-receiver pairs) subject to cross-link interference. At the beginning of each time slot, the scheduler needs to decide a subset of links $\mathcal{U}_{\mathrm{S}}\subseteq \mathcal{U}$
 to transmit depending on the channel
state information (CSI).

We assume Rayleigh fading channels with interference across different links.
If a link $u\in\mathcal{U}$ is scheduled, the channel
gain is $g_{u,u}$ subject to Rayleigh fading. For notational convenience, we omit the time slot index.
 Multiple links can be scheduled at the same time slot, creating interference to each other.
 For example, if link $u$ and  link $v$ are scheduled simultaneously, the interference channel gain from the transmitter $u$ to receiver $v$ is $g_{u,v}$, and the interference channel gain from the transmitter $v$ to receiver $u$ is $g_{v,u}$.  Thus, the received signal
  at receiver $u$ can be expressed as $g_{u,u}s_u+\sum_{v\in\mathcal{U}_{\mathrm{S}}/u}g_{v,u}s_v+ \mathrm{noise}_u$, where
   $\mathrm{noise}_u\sim\mathcal{N}\left(0,\sigma^2_{\mathrm{n}} \right) $ is an additive white Gaussian noise and the transmit signals $s_{u}$ and $s_{v}$  are normalized with unit power.
  Considering a centralized setting as in \cite{power_control_liang2019towards},
  the scheduler has access to the direct transmit channel gains as well as interference channel gains
  at the beginning of each time slot, which are contained in a $N\times N$ dimensional CSI vector $x=[g_{1,1},\cdots,g_{1,N},g_{2,1}\cdots,g_{N-1,N},g_{N,1},\cdots,g_{N,N}]$.

The scheduling decision can be represented by a $N$ dimensional scheduling vector $y$. Specifically, if the link $u$ is scheduled, then the $u$-th entry of $y$ is one, and zero otherwise.  By the Shannon rate formula in the communications theory \cite{Goldsmith_WirelessCommunications_2005},
the  achievable rate for link $u$ can be expressed as
\begin{equation}\label{eqn:shannorate}
C^u_{\mathrm{Shannon}}(x,y,\mu)= \log\left(1+\frac{\mu y(u)\|g_{u,u}\|^2}{\sigma^2_{\mathrm{n}}+\sum_{v\in\mathcal{U}/{u}}y(v)\|g_{v,u}\|^2} \right),
\end{equation}
where $\mu(0,1]$ is a parameter subject to real communication systems,
with $C^u_{\mathrm{Shannon}}(x,y,1)$ representing the standard Shannon rate (i.e., when
$\mu=1$).
The sum rate is $C_{\mathrm{Shannon}}(x,y)=\sum_{u\in\mathcal{U}}C^u_{\mathrm{Shannon}}(x,y)$.

In practice, given the CSI vector $x$ and the corresponding decision vector $y$, the real sum rate is denoted as $C_{\mathrm{real}}(x,y)=\sum_{u\in\mathcal{U}}C^u_{\mathrm{real}}(x,y)$. The real rate is difficult to express analytically in view of the complex factors in real environments including various schemes of modulation, finite channel coding and quality of service (QoS) guarantee.
 In fact, except
for a few special cases,
 the \emph{exact} channel capacity for general interference channels (even for
  two links) is still an open problem.
Thus, while
 the Shannon rate is useful and has been utilized
 to design various systems,
 it
 only represents
an approximation of the practically achievable rate

Next, we formulate the link scheduling problem as
\begin{equation}
\max_y \sum_{u\in\mathcal{U}}C^u_{\mathrm{real}}(x,y), \;\;\;\;\;\mathrm{s.t.} \;\;\; y(u)\in\{0,1\},u\in\mathcal{U}.
\end{equation}
 The scheduling objective is the real sum rate in a practical environment. The challenge of this problem is that  the real rate in terms of the CSI $x$ and scheduling decision $y$ is too complex to express precisely, let alone
the longstanding challenges of deriving the exact interference channel capacity \cite{Goldsmith_WirelessCommunications_2005}.

\subsubsection{Informed DNN for Wireless Link Scheduling}

DNNs have strong representation power to learn the optimal scheduling decisions given CSI input \cite{learning_interference_management_sun2018learning},
but they typically require a
large number of labeled samples $(x,y)$ for training.
On the other hand, domain knowledge (i.e., Shannon rate formula ) is also useful,
but it may not capture the real achievable rate in practice
\cite{power_control_liang2019towards}.
Thus, informed DNN, which
exploits domain knowledge to complement labeled samples, has the potential
to reap the benefits of both approaches.

Concretely, we use a DNN to represent the relationship between the scheduling decision $y$ and CSI $x$. Given $N$ links, the input dimension is $N\times N$, which is the dimension of vectorized CSI $x$ and the scheduling decision $y$ is a $N-$dimensional binary vector.
The training is based on a labeled dataset $S_y=\left\lbrace (x_i, y_i), i=1,2,\cdots, n_y\right\rbrace $
collected from real systems or field studies,
where $y_i$ is the true label (i.e., optimal scheduling decision) given $x_i$,  along with
the domain knowledge of Shannon rate. Also, we use $Y_{\mathrm{comb}}\in\{0,1\}^{I_{\max}\times N}, I_{\max}=2^N-1$ to represent all the possible decision combinations. Denote $I(y)$ as the index of a scheduling decision $y$ in $Y_{\mathrm{comb}}$, i.e. $y=\left[Y_{\mathrm{comb}}\right]_{I(y)}$. The output dimension of the DNN is $I_{\max}=2^N-1$ with each entry representing an index for a scheduling decision.

The label-based risk is the cross-entropy loss between the output of the DNN and one-hot encoding labels, which is expressed as
\begin{equation} \hat{R}_{S_y}(\bm{W})=\frac{1}{n_y}\sum_{i=1}^{n_y}\mathrm{cross\_entropy}\left( \mathrm{softmax}(h_{\bm{W}}(x_i)),\mathrm{one\_hot}(I(y_i))\right) ,
\end{equation}
where $\mathrm{one\_hot}(I(y_i))$ is the one-hot encoding of the index of $y_i$.
Given an CSI input $x$ and setting $\mu=\mu_\mathrm{K}$ based on domain experience, we can compute the sum rate of all possible scheduling decisions by the Shannon equation in Eqn.~\eqref{eqn:shannorate} as $C_{\mathrm{Shannon}}(x,\left[ Y_{\mathrm{comb}}\right]_{j} ,\mu_{\mathrm{K}}), j\in [I_{\max}]$ and get the vector of sum rate as $\bm{c}(x)=\left[C_{\mathrm{Shannon}}\left( x,\left[ Y_{\mathrm{comb}}\right]_1,\mu_{\mathrm{K}}\right) ,\cdots, C_{\mathrm{Shannon}}\left( x,\left[ Y_{\mathrm{comb}}\right]_{I_{\max}} ,\mu_{\mathrm{K}}\right)  \right] $. Taking the softmax operation on $T\bm{c}(x)$ with $T$ as a scaling hyper-parameter, we get  $\mathrm{softmax}\left(T \bm{c}(x)\right)$, which is essentially soft encoding of scheduling decisions based on the Shannon rate knowledge.
Therefore, given an input dataset $S_g=\left\lbrace x_i,  i=1,2,\cdots, n_g\right\rbrace $, the knowledge-based risk is designed as
\begin{equation}\label{eqn:knowledgerisk}
\hat{R}_{\mathrm{K}}(\bm{W})=-\frac{1}{n_g}\sum_{i=1}^{n_g}\mathrm{cross\_entropy}\left( \mathrm{softmax}(h_{\bm{W}}(x_i)),\mathrm{softmax}\left(T \bm{c}(x)\right) \right).
\end{equation}
Thus, the DNN can be trained to minimize the informed risk combining both
label-based and knowledge-based risks: $\hat{R}_{\mathrm{I}}(\bm{W})=(1-\lambda)\hat{R}_{S_y}(\bm{W})+\lambda \hat{R}_{\mathrm{K}}(\bm{W})$. That is, the informed DNN uses hard labels for direct supervision,
while exploiting domain knowledge in the form of soft labels for indirect supervision on unlabeled inputs. After training the network, the scheduling decision for CSI $x$ is calculated as $y_{\bm{W}}(x)=\left[ Y_{\mathrm{comb}}\right]_{I_{\bm{W}}(x)}$ with $I_{\bm{W}}(x)=\arg\max_{j\in [I_{\max}]}\left[ h_{\bm{W}(x)}\right]_j $

\subsubsection{Results}

Now, we show the simulation results for the wireless link scheduling problem based on our informed DNN. We first give the simulation settings and then show the results of classification accuracy as well as the sum rate.

\textbf{Simulation Settings.}
For illustration, we consider a simulation scenario with $N=4$ wireless links for scheduling, which is a reasonable setting for many practical ad hoc networks \cite{Goldsmith_WirelessCommunications_2005}.
Given the CSI, the scheduler needs to choose one out of 15 scheduling combinations.
To evaluate the performance of our informed DNN when the domain knowledge
of Shannon rate is not perfect, we construct a synthetic  dataset as the ground truth.
The direct link channel gain of a wireless link $g_{u,u}$ is subject to Rayleigh distribution, with an expected power gain of $100$ dB. The cross-link interference channel gain is also subject to Rayleigh distribution with an expected power gain
  of $10$ dB. The labels in the labeled training dataset and test dataset are generated by a pseudo-real rate expression to reflect some practical constraints:
 \begin{equation}\label{eqn:pseudorealrate}
C_{\mathrm{pseudo}-\mathrm{real}}=C^u_{\mathrm{Shannon}}(x,y,\mu_{\mathrm{R}}),
\end{equation}
which differs from the standard Shannon formula by using a
factor $\mu_{\mathrm{R}}\in(0,1)$ to account for achievable rate degradation.
Note that the pseudo-real rate is only defined to generate synthetic real rate different from the standard Shannon rate for evaluation purposes. In practice, the achievable rate is even more
 complex. In the simulations, we set $\mu_{\mathrm{R}}=0.5$ to generate the training and testing labels as ground truth, while the value of $\mu_{\mathrm{R}}=0.5$ is not
 available to the learner.

 Based on the pseudo-real rate expression, we find the optimal labels (i.e.,
 optimal scheduling decision $y$) via exhaustive search, while
 labels are actually be collected by
field measurement in a practical environment.
We have $n_g=2000$ unlabeled CSI input samples in the training dataset $S_g$ for knowledge-based supervision, and $n_t=10000$ samples in the test dataset.
The test accuracy is defined as the percentage of DNN outputs that are identical
to the optimal scheduling decision label, i.e. for samples in the test dataset, $\mathrm{acc}=\sum_{i=1}^{n_t}{\mathds{1}\left(I_{\bm{W}}(x_i)=I(y_i) \right) }/n_t$.
We compare the results when the labeled training dataset has 100, 500 and 1000 samples, respectively.
Also, we compare the results obtained by setting different parameters $\mu_{\mathrm{K}}\in\{1.0,0.4,0.1\}$ in the knowledge-based Shannon rate in
Eqn.~\eqref{eqn:shannorate}.
The parameter $\mu_{\mathrm{K}}\in\{1.0,0.4,0.1\}$
results in a test accuracy of $\{71.4\%, 91.2\%,52.8\%\}$,
which is the maximum test accuracy obtained by directly solving the scheduling problem based on Eqn.~\eqref{eqn:shannorate} and can be used to informally indicate the knowledge quality. Thus, $\mu_{\mathrm{K}}=0.4$ represents
the best knowledge quality, whereas $\mu_{\mathrm{K}}=0.1$ is the worst.

Now we list the settings for training.
The neural network has three hidden layers with 512, 1024 and  512 neurons, respectively, followed by ReLu activations. The network is initialized based on Algorithm~\ref{alg:neuraltrain}. The training is performed by the Adam optimizer with learning rate $10^{-5}$ for $2000$ steps on a HPC cluster with GPU type P100. We
use 5 random seeds for each setting to evaluate the performance error.

 \begin{figure*}[!t]	
	\centering
	\subfigure[$\mu_{\mathrm{K}}=0.1$]{
		\label{fig:acc0d1}
		\includegraphics[width=0.32\textwidth]{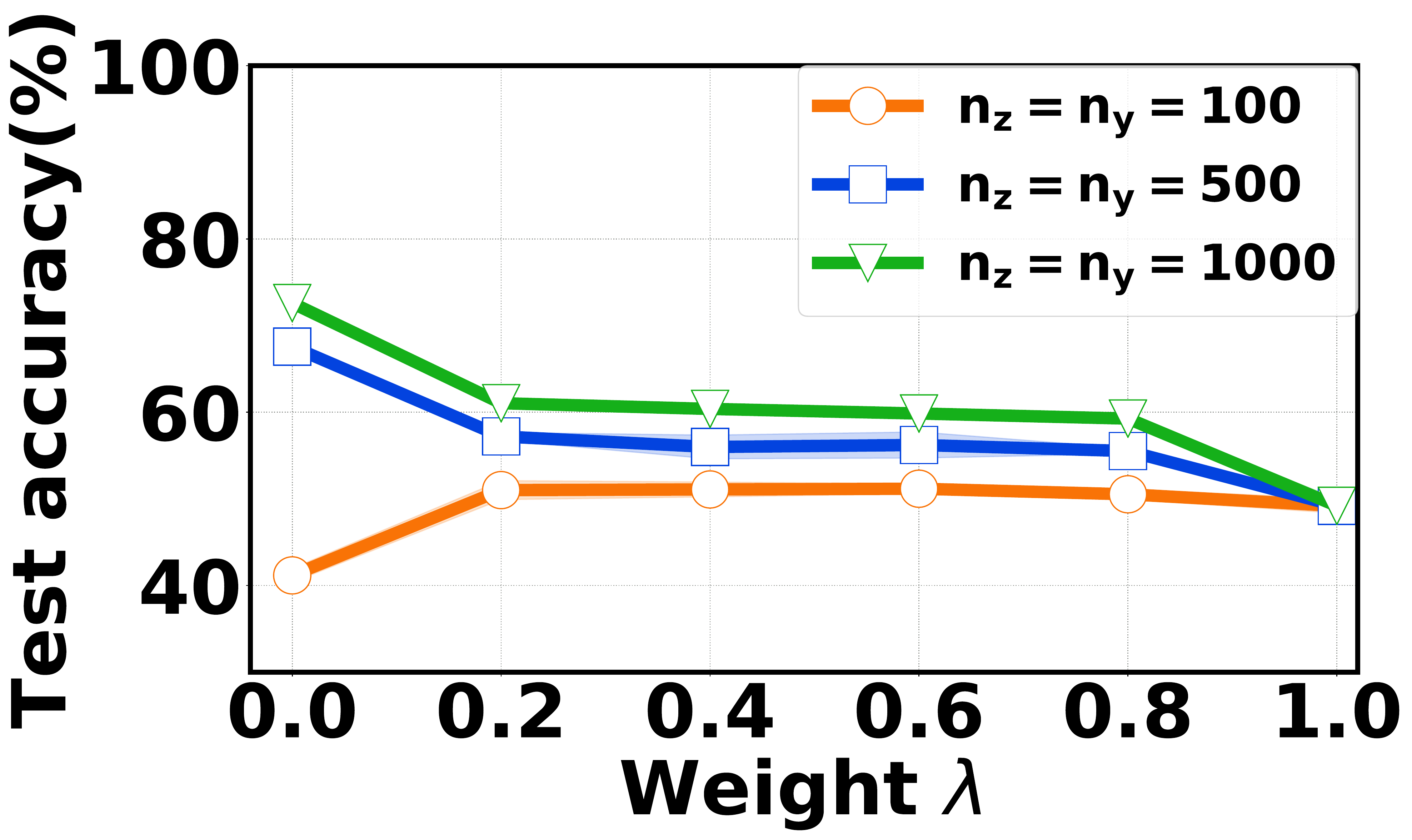}
	}%
	\subfigure[$\mu_{\mathrm{K}}=1.0$]{
		\label{fig:acc1d0} \includegraphics[width={0.32\textwidth}]{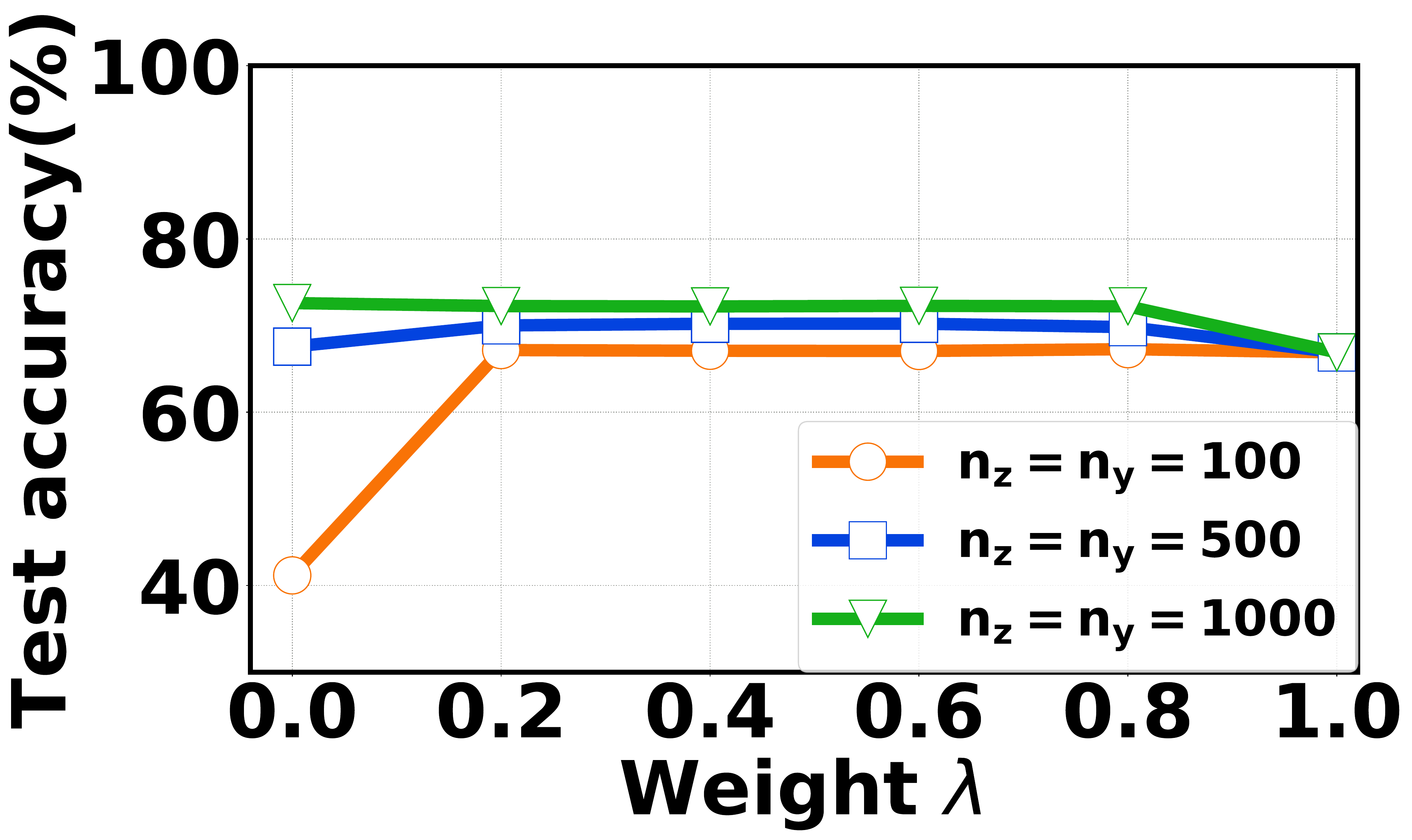}
	}
	\subfigure[$\mu_{\mathrm{K}}=0.4$]{
		\label{fig:acc0d4}
		\includegraphics[width=0.32\textwidth]{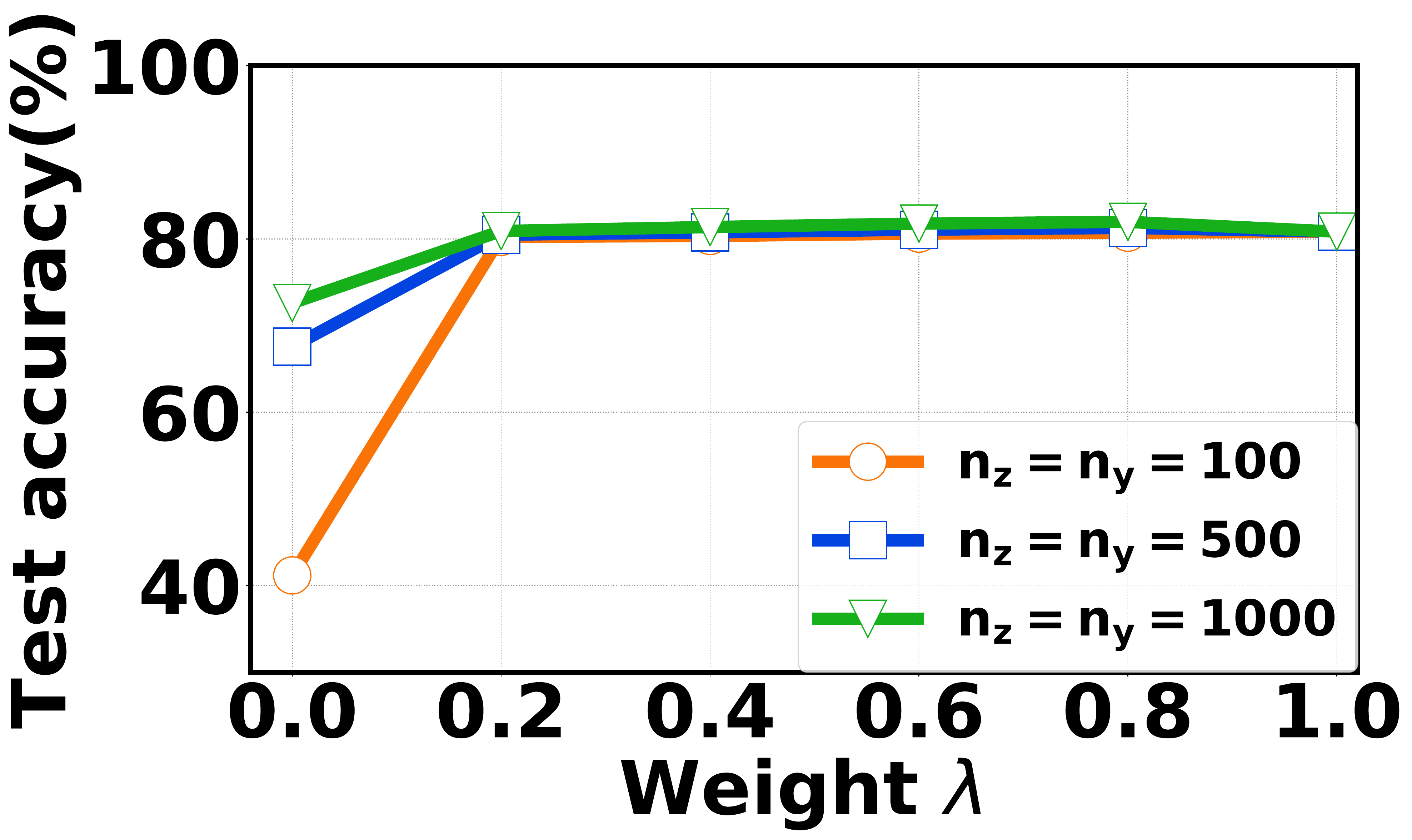}
	}%
	\centering
	\vspace{-0.4cm}	
	\caption{Test accuracy under different knowledge qualities and numbers of labels.}
	\label{fig:simulationacc}
\end{figure*}

\begin{figure*}[!t]	
	\centering
	\subfigure[$\mu_{\mathrm{K}}=0.1$]{
		\label{fig:rate0d1}
		\includegraphics[width=0.32\textwidth]{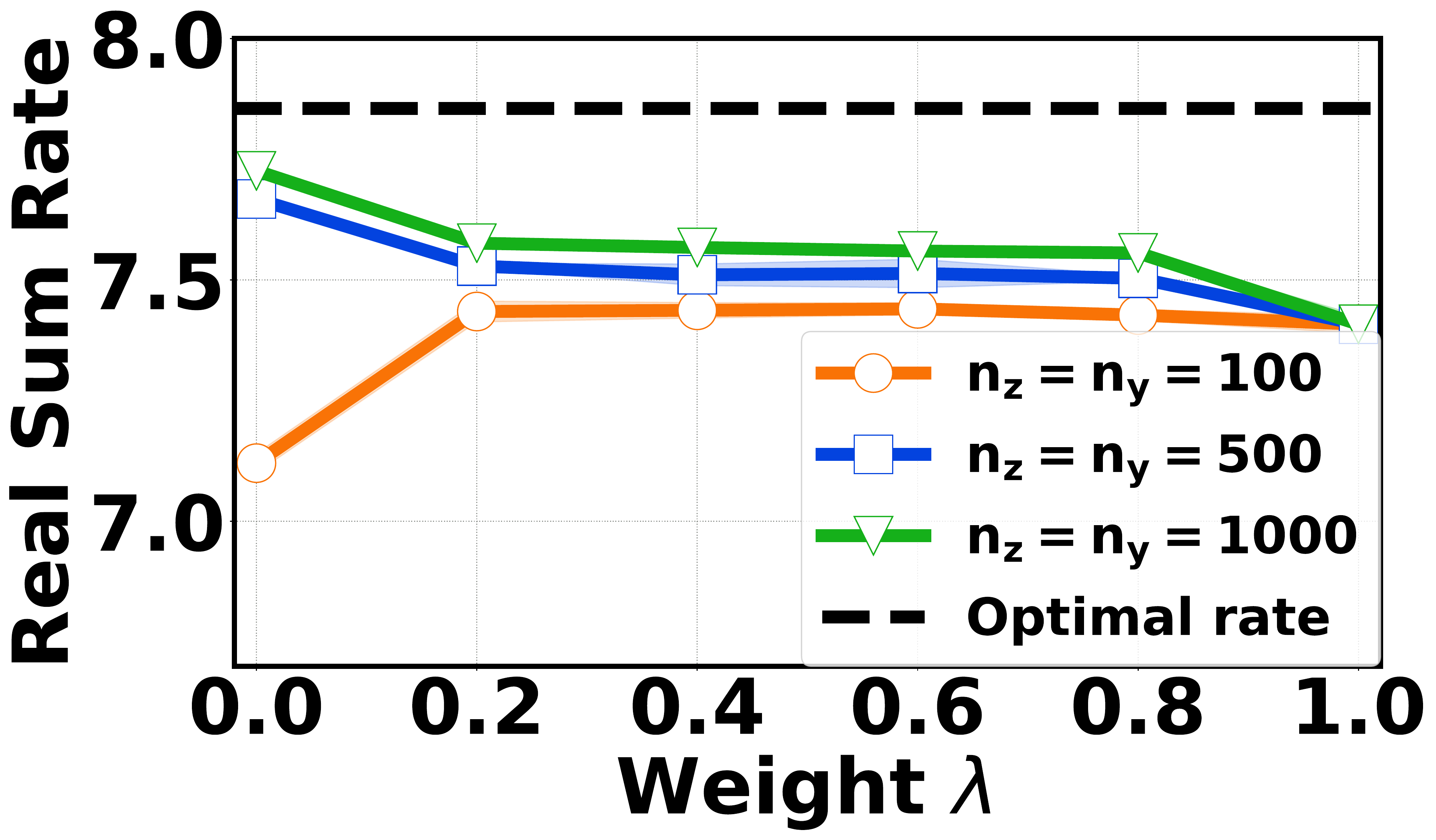}
	}%
	\subfigure[$\mu_{\mathrm{K}}=1.0$]{
		\label{fig:rate1d0} \includegraphics[width={0.32\textwidth}]{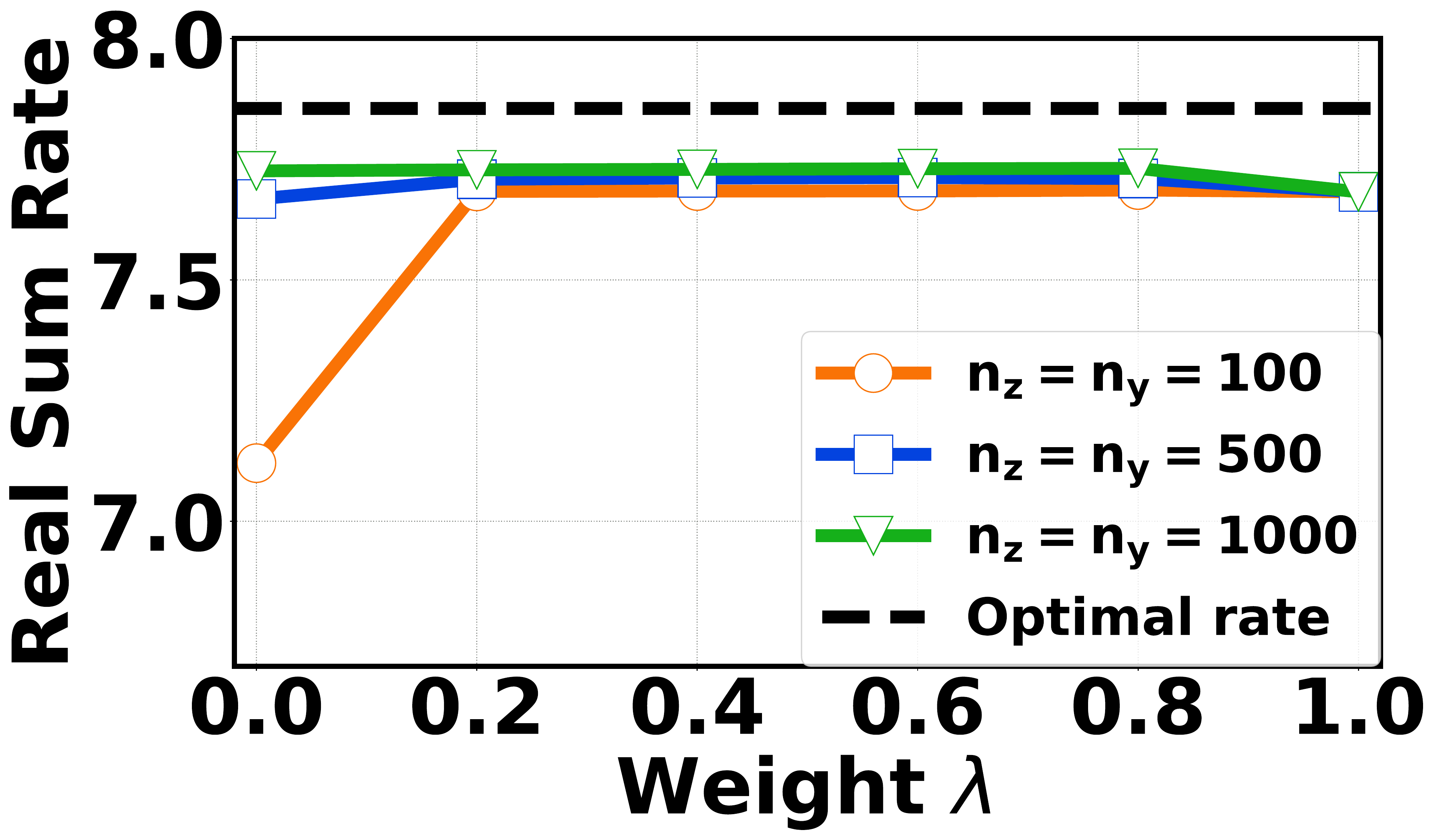}
	}
	\subfigure[$\mu_{\mathrm{K}}=0.4$]{
		\label{fig:rate0d4}
		\includegraphics[width=0.32\textwidth]{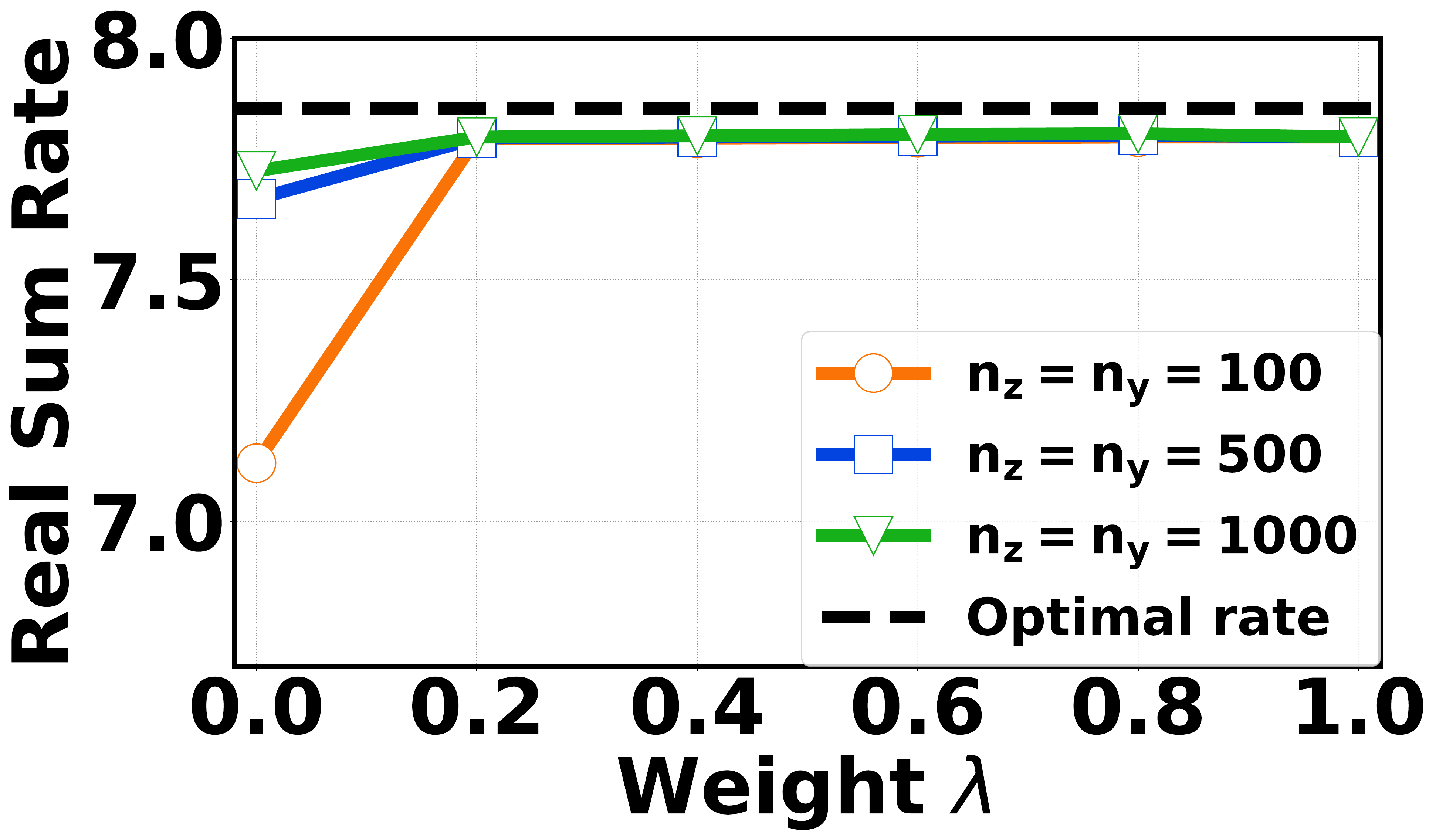}
	}%
	\centering
	\vspace{-0.4cm}	
	\caption{Sum rate under different knowledge qualities and numbers of labels.}
	\label{fig:simulationrate}
\end{figure*}
\textbf{Results.}
The results, including the test accuracy and the test sum rate under different knowledge quality, numbers of labels and weights $\lambda$, are shown in Fig.~\ref{fig:simulationacc} and Fig.~\ref{fig:simulationrate}.
The test sum rate is the (pseudo)  real sum rate defined in Eqn.~\eqref{eqn:pseudorealrate}
with $\mu_{\mathrm{R}}=0.5$. We can find that  the sum rate expectedly increases if the test accuracy increases. From Fig.~\ref{fig:acc0d1} and Fig.~\ref{fig:rate0d1}, we see that if the domain knowledge
quality is only 52.8\% (i.e., $\mu_{\mathrm{K}}=0.1$), it has bad effects on learning when labels are enough. Nevertheless, it still benefits the performance when there are only 100 labels and, if we place a less weight on the knowledge-based risk, the accuracy and sum rate is higher.

If the knowledge quality is 71.4\% (i.e., $\mu_{\mathrm{K}}=1.0$), as shown in Fig.~\ref{fig:acc1d0} and Fig.~\ref{fig:rate1d0}, the domain knowledge has significant benefits when there are only 100 labeled samples. When there are 500 labeled samples, the domain knowledge and labels complement each other and get a better performance than pure label-based and knowledge-based learning. When the number of labeled samples is even higher and reaches 1000,
the integration of domain knowledge cannot benefit the learning further.
In Fig.~\ref{fig:acc0d4} and Fig.~\ref{fig:rate0d4}, when the domain knowledge quality
further improves, we can see that the domain knowledge can still bring benefits even
in the presence of 1000 labeled samples.

From these results, we see that labels and domain knowledge can complement each other. The domain knowledge plays an important role when labels are relatively scarce, while
 labels, even only a few, help improve the learning performance when domain knowledge has a low quality. Additionally, it is important to achieve a balance between label-based supervision and knowledge-based supervision.  In general, we place more weight on the knowledge-based risk if knowledge quality is good enough and the number of labels is small, and vice versa.